%% file: arxiv_inst-opt.tex
\documentclass[11pt]{article}
\usepackage{fullpage}

\date{}

\title{\bfseries From PAC to Instance-Optimal Sample Complexity in the Plackett-Luce Model}

\author{
Aadirupa Saha\thanks{Indian Institute of Science, Bangalore, India. {\tt aadirupa@iisc.ac.in}}, \and Aditya Gopalan \thanks{Indian Institute of Science, Bangalore, India. {\tt aditya@iisc.ac.in} }
}

\usepackage{microtype}
\usepackage{cancel}
\usepackage{graphicx}
\usepackage{booktabs} 
\usepackage{amssymb}
\usepackage{color}
\usepackage{amsmath}
\usepackage{amsthm}
\usepackage{thmtools}
\usepackage{thm-restate}
\usepackage{algorithm}
\usepackage{algorithmic}

\usepackage{natbib}
\usepackage{hyperref}
\usepackage{nameref}

\allowdisplaybreaks

\newtheorem{thm}{Theorem}
\newtheorem{lem}[thm]{Lemma}

\newtheorem{defn}[thm]{Definition}

\theoremstyle{remark}
\newtheorem{rem}{Remark}

\newcommand{\R}{{\mathbb R}}

\newcommand{\N}{{\mathbb N}}

\newcommand{\E}{{\mathbf E}}

\newcommand{\1}{{\mathbf 1}}

\newcommand{\cA}{{\mathcal A}}

\newcommand{\cB}{{\mathcal B}}

\newcommand{\cC}{{\mathcal C}}

\newcommand{\cF}{{\mathcal F}}
\newcommand{\cE}{{\mathcal E}}

\newcommand{\cG}{{\mathcal G}}

\newcommand{\cN}{{\mathcal N}}

\newcommand{\cR}{{\mathcal R}}
\newcommand{\cO}{{\mathcal O}}
\newcommand{\cS}{{\mathcal S}}

\newcommand{\bK}{{\mathbf K}}
\newcommand{\X}{{\mathcal X}}

\newcommand{\tp}{{\tilde p}}
\newcommand{\thetk}{{\Theta_{[k]}}}
\newcommand{\thets}{{\Theta_{S}}}
\newcommand{\thetb}{{\theta_b}}
\newcommand{\thetsh}{{\hat \Theta_{S}}}
\newcommand{\thetg}{{\hat \Theta_{\cG_g}}}

\newcommand{\blam}{{\boldsymbol \lambda}}
\newcommand{\bpi}{{\boldsymbol \pi}}

\renewcommand{\b}{{\mathbf b}}

\newcommand{\hp}{{\hat p}}
\newcommand{\p}{{\mathbf p}}
\newcommand{\q}{{\mathbf q}}

\newcommand{\x}{{\mathbf x}}
\newcommand{\y}{{\mathbf y}}
\newcommand{\z}{{\mathbf z}}
\newcommand{\sm}{\setminus}
\newcommand{\objbest}{\textit{Objective-1 (Winner-Regret)}}
\newcommand{\objk}{\textit{Objective-2 (Top-K-Regret)}}
\def \sc{{Q}}
\def \bi{{Best-Item}}

\def \ordo{{\it Order-Oblivious}}
\def \bcon{{\it Budget-Consistent}}
\def \subrout{{\it Subroutine-1}}
\def \subroutm{{\it Subroutine-m}}

\newcommand{\rb}{\textit{Rank-Breaking}}
\newcommand{\algdiv}{\textit{Partition}}
\newcommand{\algep}{\textit{$(\epsilon,\delta)$-PAC \bi}}
\newcommand{\algfewf}{\textit{PAC-Wrapper}}
\newcommand{\algdnb}{\textit{Divide-and-Battle}}
\newcommand{\algfetf}{\textit{PAC-Wrapper (for \tf)}}
\newcommand{\algfttf}{\textit{Uniform-Allocation}}
\newcommand{\algs}{\textit{Score-Estimate}}

\newcommand{\pl}{Plackett-Luce}
\newcommand{\pll}{{PL$(n,\btheta)$}}
\newcommand{\wf}{ {Winner feedback}}
\newcommand{\tf}{ {Top-$m$ Ranking feedback}}

\newcommand{\fe}{\textit{Probably-Correct-\bi}}
\newcommand{\ft}{\textit{Fixed-Sample-Complexity}}

\newcommand{\btheta}{\boldsymbol \theta}
\newcommand{\bSigma}{\boldsymbol \Sigma}

\newcommand{\bsigma}{\boldsymbol \sigma}

\newcommand{\ceil}[1]{\Big \lceil{#1} \Big\rceil}
\newcommand{\flr}[1]{\Big \lfloor{#1}\Big \rfloor}
\newcommand{\red}[1]{\textcolor{red}{#1}}

\begin{document}

\maketitle

\input{abstract.tex}

\input{introduction.tex}

\input{prob_setup.tex}

\input{res_fxdP.tex}

\input{lb.tex}

\input{res_fxdT.tex}

\input{experiments.tex}

\input{conclusion.tex}
\section*{Acknowledgements}
We thank Praneeth Netrapalli for insightful discussions.



\bibliographystyle{plainnat}
\bibliography{inst-opt-bbpl}  

\newpage

\appendix
\input{appendix.tex}



\end{document}

%% file: abstract.tex
\begin{abstract}
We consider PAC-learning a good item from $k$-subsetwise feedback information sampled from a \pl\, probability model, with {\em instance-dependent} sample complexity performance. In the setting where subsets of a fixed size can be tested and top-ranked feedback is made available to the learner, we give an algorithm with  optimal instance-dependent sample complexity, for PAC best arm identification, of $O\bigg(\frac{\thetk}{k}\sum_{i = 2}^n\max\Big(1,\frac{1}{\Delta_i^2}\Big) \ln\frac{k}{\delta}\Big(\ln \frac{1}{\Delta_i}\Big)\bigg)$,
$\Delta_i$ being the \pl\, parameter gap between the best and the $i^{th}$ best item, and $\thetk$ is the sum of the \pl\, parameters for the top-$k$ items. 
The algorithm is based on a wrapper around a PAC winner-finding algorithm with weaker performance guarantees to adapt to the hardness of the input instance. The sample complexity is also shown to be multiplicatively better depending on the length of rank-ordered feedback available in each subset-wise play. We show optimality of our algorithms with matching sample complexity lower bounds. We next address the winner-finding problem in \pl\, models in the fixed-budget setting with instance dependent upper and lower bounds on the misidentification probability, of $\Omega\left(\exp(-2 \tilde \Delta Q) \right)$ for a given budget $Q$, where $\tilde \Delta$ is an explicit instance-dependent problem complexity parameter. Numerical performance results are also reported.
\vspace*{-10pt}
\end{abstract}


%% file: introduction.tex
\vspace*{-10pt}
\section{Introduction}
\label{sec:intro}
We consider the problem of sequentially learning the best item of a set when  subsets of items can be tested but information about only their relative strengths is observed. This is a basic search problem motivated by applications in recommender systems and information retrieval \cite{hofmann2013fast, radlinski2008does}, crowdsourced ranking \cite{chen2013pairwise}, tournament design \cite{graepel2006ranking}, etc. It has received  recent attention in the online learning community, primarily under the rubric of dueling bandits (e.g., \cite{Yue+12} and online ranking in the Plackett-Luce (PL) discrete choice model  \cite{ChenSoda+18,SGwin18, Ren+18}. 

Our focus in this paper is to study the instance-dependent complexity of learning the (near) best item in a subset-wise PL feedback model by which we mean the following. Each item has an a priori unknown PL weight parameter, and every time a subset of alternatives is selected, an item or items sampled from the PL probability distribution over the subset are observed by the learner. Given a tolerance $\epsilon$ and confidence level $\delta$, the learner faces the task of sequentially playing subsets of items, and stopping and finding an $\epsilon$-optimal arm, i.e., an arm $i$ whose PL parameter satisfies $\theta_i \geq \max_j \theta_j - \epsilon$, with probability of error at most $\delta$.

Existing work on best arm learning in PL models, e.g., \cite{SGwin18}, focuses on attaining the worst-case or {\em instance-independent} sample complexity of learning an approximately best item. By this, we mean that the typical goal is to design algorithms that terminate in a number of rounds bounded by a function of only $\epsilon$, $\delta$ and the number of arms $n$, typically of the form $O\left(\frac{n}{\epsilon^2} \log \left(\frac{1}{\delta}\right) \right)$ rounds. Such worst-case results, though significantly novel, suffer from two weaknesses: (1) The termination time guarantees become vacuous in the setting where an exact best arm is sought ($\epsilon = 0$), and (2) The guarantees do not reflect the fact that some problem instances, in terms of their items' PL parameters, are easier than others to learn, e.g., the instance with parameters $(\theta_1, \ldots, \theta_n) = (1, 0.01, \ldots, 0.01)$ ought to be much easier than $(1, 0.99, \ldots, 0.99)$ since item $1$ is a distinctly clearer winner than in the latter case. In this paper, we set ourselves the more challenging objective of quantifying and attaining sample complexity that depends on the inherent `hardness' of the PL instance. 
In this context, we make the following contributions: 

$\boldsymbol {(1)}$ We give the {first instance-optimal algorithm for the problem of $(\epsilon, \delta)$-PAC learning a best item} in a PL model  when subsets of a fixed size can be tested in each round. This is accomplished by building a novel wrapper algorithm  (Alg. \ref{alg:fewf}) around an $(\epsilon, \delta)$-PAC learning algorithm used as a subroutine that we designed (Alg. \ref{alg:epsdel}). We also provide a {matching instance-dependent lower bound on the sample complexity} of any algorithm, to establish the optimality of our algorithm (Thm. \ref{thm:sc_fewf},\ref{thm:sc_fewfep},\ref{thm:lb_fewf}).
	
$\boldsymbol {(2)}$ When {richer, $m$ length rank-ordered information} is observed per subsetwise query, we show the optimal instance-dependent sample complexity lower bound is much smaller than just with the winner feedback case (Thm. \ref{thm:lb_fetf}). We also propose an optimal algorithm for this setting (Alg. \ref{alg:fttf}) with an $\frac{1}{m}$-factor improved sample complexity guarantee which is shown to be optimal (Thm. \ref{thm:sc_fetf}).

$\boldsymbol {(3)}$ We also study the {fixed-budget version of the best-item learning problem}, where a learning horizon of $Q$ rounds is specified instead of a desired confidence level $\delta$, and the performance measure of interest is the probability of error in identifying a best arm. We give an algorithm for learning the best item of a \pl\, instance under a fixed budget with general $m$-way ranking feedback (Alg. \ref{alg:fttf}, Thm. \ref{thm:pr_fttf}), and also prove an instance-dependent lower bound for it (Thm. \ref{thm:lb_fttf}). 

Our theoretical findings are also {supported with numerical experiments} on different datasets. Related work is discussed in Appendix \ref{app:rel} due to space constraints. 


%% file: prob_setup.tex
\vspace{-10pt}
\section{Problem Setup}
\label{sec:prb_setup}
{\bf Notation.} We denote by $[n]$ the set $\{1,2,...,n\}$. For any subset $S \subseteq [n]$, let $|S|$ denote the cardinality of $S$. 
When there is no confusion about the context, we often represent (an unordered) subset $S$ as a vector, or ordered subset, $S$ of size $|S|$ (according to, say, a fixed global ordering of all the items $[n]$). In this case, $S(i)$ denotes the item (member) at the $i$th position in subset $S$.   
For any ordered set $S$, $S(i:j)$ denotes the set of items from position $i$ to $j$, $i<j$, $\forall i,j \in [|S|]$.
$\bSigma_S = \{\sigma \mid \sigma$ is a permutation over items of $ S\}$, where for any permutation $\sigma \in \Sigma_{S}$, $\sigma(i)$ denotes the element at the $i$-{th} position in $\sigma, i \in [|S|]$. We also denote by $\bSigma_S^m$ the set of permutations of any $m$-subset of $S$, for any $m \in [k]$, i.e. $\Sigma_S^m := \{ \Sigma_{S'} \mid S' \subseteq S, \, |S'| = m \}$. 
$\1(\varphi)$ is generically used to denote an indicator variable that takes the value $1$ if the predicate $\varphi$ is true, and $0$ otherwise. 
$x \vee y$ denotes the maximum of $x$ and $y$, and $Pr(A)$ is used to denote the probability of event $A$, in a probability space.

\begin{defn}[\pl\, probability model]
\label{def:mnl_thet}
A \pl\, probability model, specified by positive parameters $(\theta_1, \ldots, \theta_n)$, is a collection of probability distributions $\{Pr(\cdot | S): S \subset [n], S \neq \emptyset \}$, where for each non-empty subset $S \subseteq [n]$, $Pr(i| S) = \frac{\theta_i \1(i \in S)}{\sum_{j \in S} \theta_j}$ $\forall 1 \leq i \leq n$. The indices $1, \ldots, n$ are referred to as `items'  or `arms' .
%
\end{defn}

Since the \pl\, probability model is invariant to positive scaling of its parameters $\btheta \equiv (\theta_i)_{i=1}^n$, we make the standard assumption that $\max_{i \in [n]} \theta_i = 1$. 


An {online learning algorithm} is assumed to interact with a \pl\, probability model over $n$ items (the `environment') as follows. At each round $t = 1, 2, \ldots$, the algorithm decides to either (a) terminate and return an item $I \in [n]$, or (b) play (test) a subset $S_t \subset [n]$ of $k$ distinct items, upon which it receives stochastic feedback whose distribution is governed by the probability distribution $Pr(\cdot | S_t)$. We specifically consider the following structures for feedback received upon playing a subset $S$:


1. \textbf{\wf:} 
The environment returns a single item $J$ drawn independently from the probability distribution $Pr(\cdot | S)$ where $Pr(J = j|S) = \frac{{\theta_j}}{\sum_{k \in S} \theta_k} \, \forall j \in S$.

2. \textbf{\tf\, ($ 1 \leq m \leq k-1$):} Here, the environment returns an ordered list of $m$ items sampled without replacement from the \pl\, probability model on $S$. More formally, the environment returns a partial ranking $\bsigma \in \bSigma_{S}^m$, drawn from the probability distribution
$
\label{eq:prob_rnk1}
Pr(\bsigma = \sigma|S) = \prod_{i = 1}^{m}\frac{{\theta_{\sigma^{-1}(i)}}}{\sum_{j \in S\sm\sigma^{-1}(1:i-1)}\theta_{j}}, \; \sigma \in \bSigma_S^m.
$ 
This can also be seen as picking an item $\bsigma^{-1}(1) \in S$ according to {\it{\wf}} from $S$, then picking $\bsigma^{-1}(2)$ from $S \setminus \{\bsigma^{-1}(1)\}$, and so on for $m$ times. When $m = 1$, \tf\, is the same as \wf.

\vspace*{-0pt}
\begin{defn}[$(\epsilon,\delta)$-PAC or fixed-confidence algorithm]
\label{def:pac}
An online learning algorithm is said to be $(\epsilon,\delta)$-{PAC} with termination time bound $Q$ if the following holds with probability at least $1-\delta$ when it is run in a \pl\,  model:  (a) it terminates within Q rounds (subset plays), (b) the returned item $I$ is an $\epsilon$-optimal item: $\theta_I \geq \max_{i \in [n]} \theta_i - \epsilon = 1 - \epsilon$. (The probability is over  both the environment and the algorithm.)
\end{defn}
\vspace*{-5pt}
%
By the {\em sample complexity} of an $(\epsilon, \delta)$-{PAC} online learning algorithm $\cA$ for a \pl\, instance $\btheta \equiv (\theta_i)_{i=1}^n$ and playable subset size $k$, we mean the smallest possible termination time bound $Q$ for the algorithm when run on $\btheta$. We use the notation $N_\cA(\epsilon,\delta) \equiv N_\cA(\epsilon,\delta, \btheta, n, k)$ to denote this sample complexity.
We aim to design $(\epsilon,\delta)$-{PAC} algorithms with as small a value of sample complexity as possible, depending on the number of items $n$, the playable subset size $k$, approximation error $\epsilon$, confidence $\delta$, and most importantly, the \pl\, model parameters $(\theta_i)_{i=1}^n$. We also assume item $1$ is optimal: $\theta_1 = \max_{i \in [n]} \theta_i = 1$, and $\Delta_i = \theta_1 - \theta_i$ for any $i \in [n]$.

%
%
%
%
%

%% file: res_fxdP.tex
\section{Instance-dependent regret for \fe\, problem}
\label{sec:fe}
\subsection{Prelude: An algorithm for $\epsilon = 0$}
For clarity of exposition, we first describe the design of a $(0,\delta)$-{PAC} or \fe\, learning algorithm, i.e., an algorithm that attempts to learn the unique best item in a \pl\, model when such an item exists\footnote{When there is more than one best item the problem of finding a best item with confidence is not well-defined.}: $1 = \theta_1 > \max_{i \geq 2} \theta_i$. This is then  generalised in the next section to an online learning algorithm that is $(\epsilon,\delta)$-{PAC}.

{\bf High-level idea behind algorithm design.} The algorithm we propose (\algfewf) is based on using an $(\epsilon,\delta)$-PAC-algorithm known to have (expected) termination time bounded in terms of $\epsilon$ and $\delta$ (a `worst' case termination guarantee not necessarily dependent on instance parameters) as an underlying black-box subroutine. The wrapper algorithm uses the black-box repeatedly, with successively more stringent values of $\epsilon$ and $\delta$, to eliminate suboptimal arms in a phased manner. The termination analysis of the algorithm shows that any suboptimal arm $i \in [n]\sm \{1\}$ survives for about  $O\Big(\frac{1}{\Delta_i^2}\ln \frac{k}{\delta}\Big)$ rounds before being eliminated, which leads to the desired bound of $O\bigg(\sum_{i=2}^{n}\frac{1}{\Delta_i^2}\ln \frac{k}{\delta}\bigg)$ on algorithm's run time performance (with high probability $(1-\delta)$) (Thm. \ref{thm:sc_fewf}).

\vspace{-2pt}
{\bf Algorithm description.} The \algfewf\, algorithm we propose (Alg. \ref{alg:fewf}) runs in phases indexed by $s = 1,2,\ldots$, where each phase $s$ is comprised of the following steps.

{\bf Step 1: Finding a good reference item.} It first calls an $(\epsilon_s,\delta_s)$-PAC subroutine (described in Sec. \ref{sec:alg_epsdel} for completeness) with $\epsilon_s = \frac{1}{2^{s+2}}$ and $\delta_s = \frac{\delta}{120s^3}$ to obtain a `reasonably good item' $b_s$---an item that is likely within an $\epsilon_s$ margin of the \bi\, with probability at least $(1-\delta_s)$) and thus a potential \bi. For this we design a new sequential elimination-based algorithm (Alg. \ref{alg:epsdel} in Appendix \ref{app:wiub_div}),  and argue that it finds such a $(\epsilon_s,\delta_s)$-PAC `good item' with {\em instance-dependent} sample complexity (Thm. \ref{lem:up_algep}), which is crucial in the overall analysis. This is an improvement upon the instance-agnostic Algorithm 6 of \cite{SGwin18} whose sample complexity guarantee is not strong enough to be used along with the wrapper. 

{\bf Step 2: Benchmarking items against the reference item.} After obtaining a candidate good item, the algorithm divides the rest of the current surviving arms into equal-sized groups of size $k-1$, say the groups $\cB_1,\ldots, \cB_{B_s}$, and `stuffs' the good `probe' item $b_s$ into each group, creating $B_s = \ceil{\frac{\cA_{s-1}}{k-1}}$ item groups of size $k$ (the Partition subroutine, Algorithm \ref{alg:div}, Appendix \ref{app:subrout}). It then plays each group $\cB_b,\, b \in [B_s]$ for a total of $t_s =  \frac{2\thetsh}{\epsilon_s^2}\ln \frac{k}{\delta_s} $ rounds, where $\thetsh$ denotes a 'near-accurate' relative score estimate of the \pl\, model for the set $\cB_b$--we use the subroutine \algs\, for estimating $\thetsh$ (see Alg. \ref{alg:algs}, Thm. \ref{thm:algs} in Appendix \ref{app:subrout}). From the winner data obtained in this process, it updates the empirical pairwise win count $w_i$ of each item within any batch $\cB_b$ by applying a rank-breaking idea (see Alg. \ref{alg:updt_win}, Appendix \ref{app:subrout}) .
 
{\bf Step 3: Discarding items weaker than the reference item.} Finally, from each group $\cB_b$, the algorithm eliminates all arms whose empirical pairwise win frequency  over the probe item $b_s$ is less than $\frac{1}{2} - \epsilon_s$ (i.e. $\forall i \in \cB_b$ for which $\hat{p}_{i b_s}< \frac{1}{2} - \epsilon_s$, $\hat{p}_{i j}$ being the empirical pairwise preference of item $i$ over $j$ obtained via \rb). The next phase then begins, unless there is only one surviving item left, which is output as the candidate \bi. Pointers to the $4$ subroutines used in the overall algorithm are as below.
%


\vspace*{-7pt}
\begin{figure}[H]
\begin{center}
\includegraphics[width=0.83\linewidth]{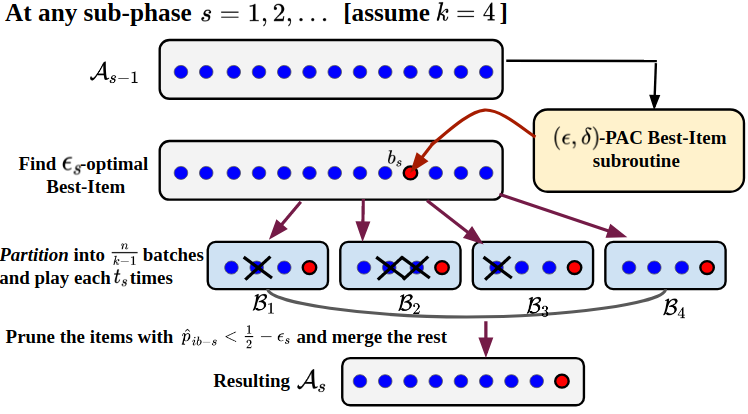}
\vspace*{-5pt}
\caption{\label{fig:alg_pacwrap} A sample run of Alg. \ref{alg:fewf} 
at any sub-phase $s$ with the set of surviving arms $\cA_{s-1}$: \textbf{Step 1.} The algorithm finds a $(\epsilon_s,\delta_s)$-PAC item $b_s$, where $\epsilon_s = \frac{1}{2^{s+2}}$ and $\delta_s = \frac{\delta}{40s^3}$. \textbf{Step 2.} It partitions $\cA_{s-1}$ into $B_s = \ceil{\frac{\cA_{s-1}}{k-1}}$ groups $\cB_1,\ldots \cB_{B_s}$ of size $k$, each containing $b_s$, and plays each for $t_s =  \frac{2k}{\epsilon_s^2}\ln \frac{k}{\delta_s} $ times. \textbf{Step 3.} Based on the received feedback of $t_s$ plays, the algorithm updates the empirical pairwise probability $\hp_{ij}$ of each item pair $(i,j)$ within a group $\cB$ by applying \rb\,, and discards any item $i \in \cB$ with $\hp_{ib_s} < \frac{1}{2}-\epsilon_s$. The rest of the surviving items are then combined to $\cA_s$, and the algorithm recurses to $s+1$.}
\end{center}
\end{figure}
\vspace*{-15pt}

\noindent
\textbf{(1). \algep\, subroutine:} Given $\epsilon, \delta \in (0,1)$, this finds an $(\epsilon, \delta)$-\bi\, in $O\Big(\frac{\thetk}{\epsilon^2}\ln \frac{k}{\delta}\Big)$ samples, where $\thetk=\hspace{-8pt} \underset{S \subseteq [n] \mid |S|=k}{\max} \sum_{i \in S}\theta_i$ (See Alg. \ref{alg:epsdel}, Thm. \ref{lem:up_algep} in Appendix \ref{app:wiub_div}). 

\noindent
\textbf{(2). \rb\, subroutine:} 
This is a procedure of deriving pairwise comparisons from multiwise (subsetwise) preference information \cite{AzariRB+14,KhetanOh16}. 
(See Alg. \ref{alg:updt_win}, Appendix \ref{app:subrout}).

\textbf{(3). \algs\, subroutine:} Given a set $S$ and a reference item $b \in [n]$, this estimates the relative \pl\, scores of the set w.r.t. $b$ (see Alg. \ref{alg:algs}, Appendix \ref{app:subrout}).

\noindent
\textbf{(4). \algdiv:} This partitions a given set of items into equally sized batches (See Alg. \ref{alg:div}, Appendix \ref{app:subrout}).

Fig. \ref{fig:alg_pacwrap} graphically depicts a sample run of a sub-phase $s$ (for $k = 4$). Note that as the playable subset size is $k$, we need to specially treat the final few sub-phases when the number of surviving arms (i.e. $|\cA_s|$) falls below $k$ (Lines $22$-$31$ in Algorithm \ref{alg:fewf}).

\vspace*{-0pt}

\begin{restatable}[\algfewf $(0,\delta)$-{PAC} sample complexity bound with \wf]{thm}{scalgfewf}
	\label{thm:sc_fewf}
	With probability at least $(1-\delta)$, $\cA$ as \algfewf\, (Algorithm \ref{alg:fewf}) returns the \bi\,  with sample complexity $N_\cA(0,\delta) = O\bigg(\frac{\thetk}{k}\sum_{i = 2}^n\max\Big(1,\frac{1}{\Delta_i^2}\Big) \ln\frac{k}{\delta}\Big(\ln \frac{1}{\Delta_i}\Big)\bigg)$, where  $\thetk: = \max_{S \subseteq [n] \mid |S|=k} \sum_{i \in S}\theta_i$.
\end{restatable}

\vspace{-10pt}
\begin{center}
\begin{algorithm}[H]
   \caption{\textbf{\algfewf ~(for \fe\, problem with \wf) }}
   \label{alg:fewf}
\begin{algorithmic}[1]
   \STATE {\bfseries input:} Set of items: $[n]$, Subset size: $n \geq k > 1$, Confidence term $\delta > 0$
   \STATE {\bfseries init:}  $\cA_0 \leftarrow [n]$, $s \leftarrow 1$ 
   \WHILE {$|\cA_{s-1}| \ge k$}
   \STATE Set $\epsilon_s = \frac{1}{2^{s+2}}$, $\delta_s = \frac{\delta}{120s^3}$, $\cR_s \leftarrow \emptyset$
   \STATE $b_s \leftarrow $ \algep$(\cA_{s-1},k,1,\epsilon_s,\delta_s)$
   \STATE $\cB_1,\ldots \cB_{B_s} \leftarrow $ \algdiv$(\cA_{s-1}\sm \{b_s\},k-1)$ 
   \STATE \textbf{if} $|\cB_{B_s}| < k-1$, \textbf{then} $\cR_s \leftarrow \cB_{B_s}$ and $B_s = B_s-1$   
   \FOR{$b = 1,2 \ldots B_s$}	   
   \STATE $\thetsh \leftarrow \algs(b_s,\cB_b,\delta_s)$. Set $\thetsh \leftarrow \max(2\thetsh + 1, 2)$.
   \STATE Set $\cB_b \leftarrow \cB_b \cup \{b_s\}$ 
   \STATE Play $\cB_b$ for $t_s := \frac{2\thetsh}{\epsilon_s^2}\ln \frac{k}{\delta_s}$ rounds
   \STATE Receive the winner feedback: $\sigma_1, \sigma_2,\ldots \sigma_{t_s} \in \bSigma_{\cB_b}^1$ after each respective $t_s$ rounds.
   \STATE Update pairwise empirical win-count $w_{ij}$ using \rb \, on $\sigma_1\ldots \sigma_{t_s}, ~\forall i,j \in \cB_b$  
   \STATE $\hp_{ij} := \frac{w_{ij}}{w_{ij} + w_{ji}}$ for all $i,j \in \cB_b$
   \STATE \textbf{If} $\exists i \in \cB_b$ s.t. $\hp_{ib_s} > \frac{1}{2} - \epsilon_s$, \textbf{then} $\cA_s \leftarrow \cA_{s} \cup \{i\}$
   \ENDFOR
	\STATE $\cA_s \leftarrow \cA_{s} \cup \cR_s$, $s \leftarrow s + 1$
  \ENDWHILE
  \STATE $\cA \leftarrow \cA_{s-1}$, 
  \STATE $\cB \leftarrow$ $\cA_{s-1} \cup \{(k-|\cA_{s-1}|)$ \text{ elements from } $[n]\sm \cA_{s-1}\}$
  \STATE Pairwise empirical win-count $w_{ij} \leftarrow 0$, $~\forall i,j \in \cA$ 
  \WHILE{$|\cA| > 1$}
  \STATE Set $\epsilon_s = \frac{1}{2^{s+2}}$, and $\delta_s = \frac{\delta}{80s^3}$
   \STATE $b_s \leftarrow $ \algep$(\cB,k,m,\epsilon_s,\delta_s)$
   \STATE $\thetsh \leftarrow \algs(b_s,\cA\sm\{b_s\},\delta_s)$. Set $\thetsh \leftarrow \max(2\thetsh + 1, 2)$.
   \STATE Play $\cB$ for $t_s := \frac{2\thetsh}{m\epsilon_s^2}\ln \frac{k}{\delta_s}$ rounds, and receive the corresponding winner feedback: $\sigma_1, \sigma_2,\ldots \sigma_{t_s} \in \bSigma_{\cB}^m$ per round.
   \STATE Update pairwise empirical win-count $w_{ij}$ using \rb \, on $\sigma_1\ldots \sigma_{t_s}, ~\forall i,j \in \cA$  
   \STATE Update $\hp_{ij} := \frac{w_{ij}}{w_{ij} + w_{ji}}$ for all $i,j \in \cA$
   \STATE \textbf{If} $\exists i \in \cA$ with $\hp_{ib_s} < \frac{1}{2} - \epsilon_s$, \textbf{then} $\cA \leftarrow \cA \sm \{i\}$
   \STATE $s \leftarrow s + 1$
  \ENDWHILE  
   \STATE {\bfseries output:} The item remaining in $\cA_s$ 
\end{algorithmic}
\end{algorithm}
\vspace{-5pt}
\end{center}

\begin{rem}
As $\thetk \le k$, \algfewf\, takes $O(\frac{1}{\Delta_i^2})\ln \frac{1}{\delta}$ rounds to eliminate all suboptimal items with confidence $\delta$. However, the dependence of the upper bound on $\thetk$ implies a $1/k$ factor gain in sample complexity when the underlying instance is `easy'. Indeed, when $\thetk = O(1)$, e.g., in an instance where $\theta_1 \approx 1$ and $\theta_i \approx 0$ $\forall i \neq 1$, then the algorithm just takes $O(\frac{1}{k \Delta_i^2})\ln \frac{1}{\delta}$ time to terminate. On the other hand, if $1 = \theta_1 > \theta_i \approx 1$, then $\thetk = \Omega(k)$ which gives the worst case orderwise complexity. 

\end{rem}

\vspace{-5pt}
\textbf{Proof sketch} 
The proof of Thm. \ref{thm:sc_fewf} is based on the following claims:


\textbf{Claim-1:} At any sub-phase $s = 1,2,\ldots$, the \bi\, $a^*$ is likely to beat the ($\epsilon_s,\delta_s$)-PAC item $b_s$ by sufficiently high margin with probability at least $(1-\frac{\delta}{20})$, and hence is never discarded (Lem. \ref{lem:bistays}).

\textbf{Claim-2:} Let $[n]_r := \{ i \in [n] : \frac{1}{2^r} \le \Delta_i < \frac{1}{2^{r-1}} \}$, and we denote the set of surviving arms in $[n]_r$ at $s^{th}$ sub-phase by $\cA_{r,s}$, i.e. $\cA_{r,s} = [n]_r \cap \cA_s$, for any $s = 1,2, \ldots$. Then with probability at least $(1-\frac{19\delta}{20})$, any such set $\cA_{r,s}$ reduces at a constant rate once $s \ge r$, $r = 1,\ldots, \log_2(\Delta_{\min})$ (Lem. \ref{lem:nbigoes})---this ensures that all suboptimal elements get eventually discarded after they are played sufficiently often.

\textbf{Claim-3:} The number of occurrences of any sub-optimal item $i \in [n]\sm\{1\}$ before it gets discarded away is proportional to $O\Big(\frac{1}{\Delta_i^2}\ln \frac{k}{\delta}\Big)$. Combining this over all arms yields the desired sample complexity. Details of the proof is given in Appendix \ref{app:prf_alg_fewf}.
$\hfill \square$.

\subsection{An algorithm for general $\epsilon > 0$} 
It is straightforward to extend the $(0,\delta)$-PAC guarantee for \algfewf\, to get a more general $(\epsilon,\delta)$-PAC algorithm for any given $\epsilon \in [0,1]$. The idea is to simply execute the algorithm as originally specified until (and if) it reaches a phase $s$ such that $\epsilon_s$ falls below the given tolerance $\epsilon$ (i.e. $\epsilon_s \le \epsilon$), at which point the algorithm can stop right after calling the subroutine \algep\, and output the item $b_s$ returned by it. The full algorithm is given in Appendix \ref{app:alg_fewfep} for the sake of brevity.


\begin{restatable}[\algfewf\, $(\epsilon,\delta)$-PAC sample complexity bound with \wf]{thm}{scalgfewfep}
\label{thm:sc_fewfep}
For any $\epsilon \in [0,1]$, with probability at least $(1-\delta)$, $\cA$ as \algfewf\, (Algorithm \ref{alg:fewf}) returns the $\epsilon$-\bi\, (see Defn. \ref{def:pac}) with sample complexity $N_\cA(\epsilon,\delta) = O\bigg(\frac{\thetk}{k}\sum_{i = 2}^n\max\Big(1,\frac{1}{\max(\Delta_i,\epsilon)^2}\Big) \ln\frac{k}{\delta}\Big(\ln \frac{1}{\max(\Delta_i,\epsilon)}\Big)\bigg)$. 
\end{restatable}

\vspace{-8pt}
{\bf Discussion.} To our knowledge, this is the first $(\epsilon,\delta)$-{PAC} learning algorithm for the \pl\, model with general multi-wise comparisons with an {\em item-wise} instance-dependent sample complexity bound.
For $\epsilon > 0$, this is order-wise stronger than the best known worst-case (instance-independent) upper bound of $O\left(\frac{n}{\epsilon^2} \log\left( \frac{k}{\delta} \right) \right)$ \cite{SGwin18}, since $\max(\Delta_i,\epsilon)^2 \geq \epsilon^2$. Thus \algfewf\, is provably able to adapt to the hardness of the \pl\, instance $\btheta$ to stop early in case the instance is `well-separated'. 
Note that for dueling bandits ($k = 2$), our result strictly improves order-wise upon the $\tilde{O}\left(n \cdot \max_{i \geq 2} \frac{1}{\max(\Delta_i, \epsilon)^2} \right)$ sample complexity\footnote{Notation $\tilde{O}(\cdot)$ hides polylogarithmic factors in $\epsilon, \delta, \Delta_i, n, k$.} of the best known $(\epsilon,\delta)$-PAC algorithm (PLPAC) \cite{Busa_pl}---which can be worse by a factor of $n$ for many instances. For example, consider an instance having one `strong' suboptimal item, say $j \in [n]\sm \{1\}$ with $\Delta_j \approx 0$, but $\Omega(n)$ many extremely `weak' items with $\Delta_i \approx 1$; our sample complexity bound is just $\tilde O\Big( \frac{1}{2\Delta_j^2}\ln \frac{1}{\delta} + \frac{n}{2}\ln \frac{1}{\delta}\Big)$, whereas that of PLPAC is  $O\Big( \frac{n}{\Delta_j^2}\ln \frac{n}{\Delta_j \delta} \Big)$.

\input{res_fxdP_m.tex}

\vspace*{-10pt}
\subsection{$(\epsilon,\delta)$-PAC subroutine (used in the main algorithm, \algfewf, i.e. in Alg. \ref{alg:fewf}, \ref{alg:epsdel} or \ref{alg:fetf})}
\label{sec:alg_epsdel}

We briefly describe here the core $(\epsilon,\delta)$-PAC subroutine used in algorithms \ref{alg:fewf} and \ref{alg:fetf} to find an $\epsilon$ \bi\, with high probability $(1-\delta)$ in an instance-dependent way (full details are available in Appendix \ref{app:wiub_div}): 
The algorithm \algep\, first divides the set of $n$ items into batches of size $k$, then plays each group sufficiently long enough until a single item of that group stands out as the empirical winner in terms of its empirical pairwise advantage over the rest (again estimated though \rb). It then just retains this empirical winner  for every group and recurses on the set of surviving winners until only a single item is left, which is declared as the $(\epsilon,\delta)$-PAC item. 

\begin{restatable}[\algep:  Correctness and Sample Complexity with \tf]{thm}{ubalgep}
\label{lem:up_algep}
For any $\epsilon \in \big(0, \frac{1}{8} \big]$ and $\delta \in (0,1)$, with probability at least $(1-\delta)$, \algep\, (Algorithm \ref{alg:epsdel}) returns an item $b_s \in [n]$ satisfying $p_{b_s1} > \frac{1}{2}-\epsilon$ with sample complexity $O\left(\frac{n \thetk}{k}\max\big(1,\frac{1}{m\epsilon^2}\big) \log \frac{k}{\delta}\right)$, where $\thetk: = \max_{S \subseteq [n], |S|=k} \sum_{i \in S}\theta_i$.
\end{restatable}

\begin{rem}
The best item-finding subroutine we develop, along with the corresponding analysis, is an improvement over Alg. 6 of \cite{SGwin18} which had $k$ instead of $\thetk_{[k]} \leq k$ here. The improvement is especially pronounced for instances where $\thetk = O(1)$ (e.g. where $\theta_{a^*} \to 1$ and for all $i \in [n]\sm \{a^*\}$, $\theta_{i} \to 0$ etc.). %
 Note that this is an artefact of the \emph{adaptive nature} of our proposed algorithm (Alg. \ref{alg:epsdel}) which samples each batch adaptively for just sufficiently enough times before discarding out the weakest $(k-1)$ items (see Line $11$), whereas \cite{SGwin18} sample each batch for a fixed $O\Big( \frac{k}{\epsilon^2}\ln \frac{k}{\delta} \Big)$ times irrespective of the empirical outcomes, leading to a \emph{worse, instance independent sample complexity}.
\end{rem}

%% file: res_fxdP_m.tex
\subsection{PAC learning in the \pl\, model with \tf}
\label{sec:alg_fetf}

\textbf{Main Idea.}
Algorithmically, the key modification to make is in the \rb\, subroutine of \algfewf, which now uses a rank-ordered list of $m$ feedback items to output all possible rank-broken comparison pairs. The essence of the $\frac{1}{m}$ factor improvement in the sample complexity over \wf\, lies in the fact that this naturally gives rise to $O(m)$ times additional number of pairwise preferences in comparison to \wf. Hence, it turns out to be sufficient to sample any batch $\cB_b, \forall b \in [B_s]$ for only  $O\big( \frac{1}{m} \big)$ times compared to the earlier case, which finally leads to the improved sample complexity of \algfewf\, for \tf.
The full description of Alg. \ref{alg:fetf} is given in Appendix \ref{app:alg_fetf} for the sake of brevity. 

\begin{restatable}[\algfewf: Sample Complexity for $(0,\delta)$-PAC Guarantee for \tf]{thm}{scalgfetf}
\label{thm:sc_fetf}
With probability at least $(1-\delta)$, \algfewf\, (Algorithm \ref{alg:fewf}) returns the \bi\,  with sample complexity $N_\cA(0,\delta) = O\bigg(\frac{\thetk}{k}\sum_{i = 2}^n\max\Big( 1, \frac{1}{m\Delta_i^2}\Big) \ln\frac{k}{\delta}\Big(\ln \frac{1}{\Delta_i}\Big)\bigg)$. 
\end{restatable}

\begin{rem}
Following the similar procedure as argued in Sec. \ref{sec:alg_epsdel}, one can easily derive an $(\epsilon,\delta)$-PAC version of \algfetf\, as well, and a similar  guarantee as that of Thm. \ref{thm:sc_fewfep} with a reduction factor $1/m$. We omit the details in the interest of space.
\end{rem}

%% file: lb.tex
\section{Instance-dependent lower bounds on sample complexity}
\label{sec:lb}
We here derive information-theoretic lower bounds on sample complexity for \fe\, problem. We first show a lower bound of $\Omega\Big(\sum_{i=2}^{n}\frac{\theta_i\theta_1}{\Delta_i^2}\ln \big( \frac{1}{\delta} + \frac{n}{k}\ln \frac{1}{\delta} \big) \Big)$ with \wf\, implying that the sample complexity of \algfewf\, (Thm. \ref{thm:sc_fewf}) is tight upto logarithmic factors. We then analyze the lower bound  for \tf\, and show an $\frac{1}{m}$-factor improvement in the sample complexity lower bound, establishing the optimality (up to logarithmic factors) of our \algfewf\, algorithm for \tf\, (see Alg. \ref{alg:fetf} and Thm. \ref{thm:sc_fetf}).

\subsection{Lower bound for \wf}
\label{sec:lb_fewf}
\begin{restatable}[Sample complexity lower bound: $(0,\delta)$-{PAC} or \fe\, with \wf]{thm}{lbfewf}
\label{thm:lb_fewf}
Given $\delta \in [0,1]$, suppose $\cA$ is an online learning algorithm for \wf\, which, when run on any \pl\, instance, terminates in finite time almost surely, returning an item $I$ satisfying $Pr(\theta_I = \max_{i} \theta_i) > 1-\delta$. Then, on any \pl\, instance $\theta_1 > \max_{i \geq 2} \theta_i$, the expected number of rounds it takes to terminate is $\Omega\bigg(\max\Big(\sum_{i=2}^{n}\frac{\theta_i\theta_1}{\Delta_i^2}\ln \frac{1}{\delta} , \frac{n}{k}\ln \frac{1}{\delta} \Big)\bigg)$. 


\end{restatable}
\textbf{Proof sketch.} 
We employ the measure-change technique of Kaufmann et al \cite{Kaufmann+16_OnComplexity} (see Lem. \ref{lem:gar16}, Appendix) for lower bounds on the PAC sample complexity for standard multiarmed bandits (MAB). The novelty of our proof lies in mapping their result to our setting: For our case each MAB instance corresponds to an instance of the BB-PL problem with the arm set containing all subsets of $[n]$ of size $k$: $A = \{S = (S(1), \ldots S(k)) \subseteq [n]\}$. 

We now consider any general true \pll\, problem instance $\text{PL}(n,\btheta^1): \theta_1^1 > \theta_2^1 \ge \ldots \ge \theta_n^1$, and corresponding to each suboptimal item $a \in [n]\setminus \{1\}$, we define an alternative problem instance $\text{PL}(n,\btheta^a): \theta_a^a = \theta_1^1 + \epsilon; ~\theta_i^a = \theta_i^1, ~~\forall i \in [n]\sm \{a\}$, for some $\epsilon>0$. %
Then, applying Lemma \ref{lem:gar16} on every pairs of problem instances $(\btheta^1, \btheta^a)$, and suitably upper bounding the KL-divergence terms we arrive at $n-1$ constraints of the form: 
\vspace*{-5pt}
\begin{align*}
& \ln \frac{1}{2.4\delta} \le  \sum_{S \in A \mid a \in S} \E_{\btheta^1}[N_S(\tau_A)]KL(p^1_S,p^a_S) \\
& \le \sum_{S \in A \mid a \in S}\E_{\btheta^1}[N_S(\tau_A)] \frac{ \Delta_a'^2}{\theta_S^1(\theta_1^1 + \epsilon)}, \, \forall a \in [n]\sm\{1\}
\end{align*}
\vspace*{-10pt}

Since the total sample complexity of $\cA$ being $\cN(0,\delta) = \sum_{S \in A}N_S$ (here $N_S$ is the number of plays of subset $S$ by $\cA$), the problem of finding the sample complexity lower bound actually reduces to solving the (primal) linear programming (LP) problem:

\vspace*{-20pt}
\begin{align*}
& \textbf{Primal LP (P):} \min_{S \in A} \sum_{S \in A} \E_{\btheta^1}[N_S]
\text{s.t., } \ln \frac{1}{2.4\delta} \\
& \le \sum_{S \in A \mid a \in S}\E_{\btheta^1}[N_S] \frac{ \Delta_a'^2}{\theta_S^1(\theta_1^1 + \epsilon)}, ~\forall a \in [n]\sm\{1\}
\end{align*}
\vspace*{-15pt}

However above has $O{n \choose k}$ many optimization variables (precisely $\E_{\btheta^1}[N_S]$s), so we instead solve the {\em dual} LP to reach the desired bound. 
Lastly the $\Omega\Big(\frac{n}{k} \ln \frac{1}{\delta}\Big)$ term in the lower bound arises as any learning algorithm must at least test each item a constant number of times via $k$-wise subset plays before judging it optimality which is the bare minimum sample complexity the learner has to incur \cite{ChenSoda+18}.
The complete proof is given in Appendix \ref{app:lb_fewf}. $\hfill \square$.

\subsection{Lower bound for \tf}
\label{sec:lb_fetf}
\begin{restatable}[Sample complexity Lower Bound: $(0,\delta)$-\fe\, with \tf]{thm}{lbfetf}
\label{thm:lb_fetf}
Suppose $\cA$ is an online learning algorithm for \tf\, which, given $\delta \in [0,1]$ and run on any \pl\, instance, terminates in finite time almost surely, returning an item $I$ satisfying $Pr(\theta_I = \max_{i} \theta_i) > 1-\delta$. Then, on any \pl\, instance $\theta_1 > \max_{i \geq 2} \theta_i$, the expected number of rounds it takes to terminate is $\Omega\bigg(\max\Big(\frac{1}{m}\sum_{i=2}^{n}\frac{\theta_i\theta_1}{\Delta_i^2}\ln \big( \frac{1}{\delta} \big) , \frac{n}{k}\ln \frac{1}{\delta} \Big)\bigg)$. 
\end{restatable}
\textbf{Proof sketch.} 
The crucial observation we make here is that due to the chain rule for KL-divergence, the KL divergence for \tf\, is $m$ times than that of just with \wf: 
$
 KL(p^1_S, p^a_S) = KL(p^1_S(\sigma_1), p^a_S(\sigma_1)) +  + \sum_{i=2}^{m}KL(p^1_S(\sigma_i \mid \sigma(1:i-1)), p^a_S(\sigma_i \mid \sigma(1:i-1)))$,
where we abbreviate $\sigma(i)$ as $\sigma_i$ and $KL( P(Y \mid X),Q(Y \mid X)): = \sum_{x}Pr\Big( X = x\Big)\big[ KL( P(Y \mid X = x),Q(Y \mid X = x))\big]$ denotes the conditional KL-divergence. 
Using this and the upper bound on the KL divergences for \wf\, setup as derived for Thm. \ref{thm:lb_fewf}, we get that in this case $KL(p^1_S,p^a_S) \le \frac{ m\Delta_a'^2}{\theta_S^1(\theta_1^1 + \epsilon)}, \, \forall a \in [n]\sm\{1\}$, where lies the crux of the $\frac{1}{m}$-factor improvement in the sample complexity lower bound compared to \wf. The lower bound now can be derived following a similar procedure described for Thm. \ref{thm:lb_fewf}. 
Details are given in \ref{app:lb_fewf}.
$\hfill \square$.

%% file: res_fxdT.tex
\section{The \ft\, learning problem}
\label{sec:ft}
This section studies the problem of finding the \bi\, within a maximum allowed number of queries \sc, with minimum possible probability of misidentification. Note the algorithms for \fe\, setting cannot be used here as they do not take the total sample complexity $\sc$ as input; also, simply terminating such algorithms with a suitable $\delta$ after $\sc$ runs may not necessarily be optimal. 
We present results for the general \tf.

\subsection{Lower Bound: \ft \, setting}
\label{sec:lb_fttf}
We derive an instance-dependent lower bound on error probability in which the problem complexity depends on the complexity term $\Big(\sum_{a = 2}^{n}\frac{\theta_a}{\Delta_a^2}\Big)^{-1}$, unlike the case for our first objective (\fe), which depends on the gap parameter $\frac{1}{\Delta_a^2}, \, \forall  \in [n]\sm\{1\}$. 
We first define a natural consistency or `non-trivial learning' property for any best-arm algorithm given a fixed budget of \sc:

\begin{defn}[\bcon\, Best-Item Identification Algorithm]
\label{def:con}
An online learning algorithm $\cA$, taking as input a sample complexity budget \sc, terminating within \sc\, rounds and outputting an item $I \in [n]$, is said to be \bcon\, if, for every \pl\, instance $\btheta \equiv (\theta_i)_{i=1}^n$ with a unique best item $a^*(\btheta) := \arg \max_{i \in [n]} \theta_i$, it satisfies $Pr_{\btheta}\big(I = a^*(\btheta)\big) \ge 1 - \exp(-f(\btheta) \cdot \sc)$ when run on $\btheta$, where $f: [0,1]^{n} \mapsto \R_+$ is an instance-dependent function mapping every \pl\, instance to a positive real number.
\end{defn}
Informally, a \bcon\, algorithm picks out the best arm in a \pl\, instance with arbitrarily low error probability given enough rounds \sc. 
We next define the notion of a \ordo\ or \emph{label-invariant} algorithm before stating our main lower bound result.
\begin{defn}[Order obliviousness or label invariance]
\label{def:sym_alg}
A \bcon\, algorithm $\cA$ is said to be \ordo\ if its output is insensitive to the specific labelling of items, i.e., if for any PL model $(\theta_1, \ldots, \theta_n)$, bijection $\phi: [n] \to [n]$ and any item $I \in [n]$, it holds that $Pr(\cA \text{ outputs } I \, | \,  (\theta_1, \ldots, \theta_n)) = Pr(\cA \text{ outputs } I \, | \, (\theta_{\phi(1)}, \ldots, \theta_{\phi(n)}))$, where $Pr(\cdot \, | (\alpha_1, \ldots, \alpha_n) )$ denotes the probability distribution on the trajectory of $\cA$ induced by the PL model $(\alpha_1, \ldots, \alpha_n)$. 
\end{defn}

\begin{restatable}[Confidence lower bound in fixed sample complexity $\sc$ for \tf]{thm}{lbfttf}
\label{thm:lb_fttf}
Let $\cA$ be a \bcon\, and \ordo\ algorithm for identifying the \bi\, under \tf. For any \pl\, instance $\btheta$ and sample size (budget) $Q$, its probability of error in identifying the best arm in $\btheta$ satisfies
$ 
Pr_{\btheta}\left(I \neq \arg \max_{i \in [n]} \theta_i \right) = \Omega\left( \exp\left(-2 m \sc \tilde \Delta \right) \right), 
$
where the complexity parameter $\tilde{\Delta} := \Big(\sum_{a = 2}^{n}\frac{(\theta_a)^2}{\Delta_a^2}\Big)^{-1}$.
\end{restatable}

\begin{rem}
As expected, the error probability reduces with increasing feedback size $m$ and budget $\sc$. However a more interesting tradeoff lies in the instant dependent complexity term $\tilde{\Delta}$: for {`easy'} instances where most of the suboptimal item have $\theta_a \to 0$ (i.e. $\Delta_a \to 1$), $\tilde \Delta$ shoots up, in fact attains $\tilde \Delta \to \infty$ in the limiting case where $\theta_a \to 0 \, \forall i \in [n \sm \{1\}$. On the other hand, for `hard' instances, where there exists even one suboptimal item $a \in [n]\sm\{1\}$ with $\theta_a \approx 1$ (i.e. $\Delta_a \approx 0$), $\tilde{\Delta} \to 0$ raising the minimum error probability significantly, which indicates the hardness of the learning problem.
\end{rem}


\vspace{-10pt}
\subsection{Proposed Algorithm for \ft\, setup: \algfttf}
\label{sec:alg_fttf}
\textbf{Main Idea.} Our proposed algorithm \algfttf\, solves the problem with a uniform budget allocation rule:
Since we are allowed to play sets of size $k$ only, we divide the items into $k$-sized batches and eliminate the bottom half of the winning items once each batch is played sufficiently. The important parameter to tune is how long to play the batches. Given a fixed budget $Q$, since one does not have an idea about which batch the \bi\, lies in, a good strategy is to allocate the budget uniformly across all sets formed during the entire run of the algorithm, which can shown to be precisely $O(\frac{n + k\log_2 k}{k})$ sets, so we allocate a budget of $Q' = O\Big( \frac{kQ}{n + k\log_2 k}\Big)$ samples per batch.

\textbf{Algorithm description.} 
The algorithm proceeds in rounds, where in each round it divides the set of surviving items into batches of size $k$ and plays each $Q' = \frac{(n+k)kQ}{2n^2\log_2 k}$ times. Upon this it retains only the top half of the winning arms, eliminating the rest forever. The hope here is that with `enough' observed samples, the \bi\,  always stays in the top half and never gets eliminated. The next round recurses on the remaining items, and the algorithm finally returns the only single element is left as the potential \bi.
The pseudocode is moved to Appendix \ref{app:alg_fttf}.

\begin{restatable}[\algfttf: Confidence bound for \bi\, identification with fixed sample complexity \sc]{thm}{pralgfttf}
\label{thm:pr_fttf}
Given a budget of $\sc$ rounds, \algfttf\, returns the \bi\, of \pll\, with probability at least 
$ 
1-O\bigg( \log_2 n \exp\Big( -\frac{mQ\Delta_{\min}^2}{16(2n + k \log_2k)} \Big) \bigg),
$ 
where $\Delta_{\min} = \min_{i=2}^{n}\Delta_i$.
\end{restatable}
\vspace{-10pt}
\begin{rem}
Thm. \ref{thm:pr_fttf} equivalently shows that with sample complexity at most $O\Bigg( \frac{16(2n + k \log_2k)}{m\Delta_{\min}^2}\ln\bigg(\frac{\log_2 n}{\delta} \bigg)\Bigg)$, \algfttf\, returns the \bi\, with probability at least $(1-\delta)$. The bound is clearly optimal in terms of $m$ and $Q$ (comparing with Thm. \ref{thm:lb_fttf}), however it still remains an open problem to close the gap between the complexity term $\tilde{\Delta} = \Big(\sum_{a = 2}^{n}\frac{(\theta_a)^2}{\Delta_a^2}\Big)^{-1}$ in the lower bound, vs. the $\Big(\frac{n}{\Delta_{\min}^2}  \Big)^{-1}$ term that we obtained.
\end{rem}


%% file: experiments.tex

\section{Experiments}
\label{sec:experiments}
\vspace*{-0pt}
This section reports numerical results of our proposed algorithm \algfewf\, (PW) on different \pl\, environments. 
All reported performances are averaged across $50$ runs.
The default values of the parameters are set to be $k=5$, $\epsilon=0.01$, $\delta= 0.01$, $m=1$ unless explicitly mentioned/tuned in the specific experimental setup. We compared our algorithm with the only existing benchmark algorithm \algdnb\, (DnB) \cite{SGwin18} (even though, as described earlier, it does not apply to instance-optimal analysis, specifically for $\epsilon = 0$; this is  reflected in our experimental results as well). We use $8$ different PL environments (with different $\btheta$ parameters) for the purpose, their descriptions are moved to Appendix \ref{app:expts}. 

Throughout this section, by the term \emph{sample-complexity}, we mean the average (mean) termination time of the algorithms across multiple reruns (i.e. number of subsetwise queries performed by the algorithm before termination).


\subsection{Results: \fe\, setting}

\textbf{Sample-Complexity vs Error-Margin $(\epsilon)$.} Our first set of experiments analyses the sample complexity ($\cN(\epsilon,\Delta)$) of \algfewf\, with varying $\epsilon$ (keeping $\delta$ fixed at $0.1$). As expected, Fig. \ref{fig:sc_vs_eps} shows that the sample complexity increases with decreasing $\epsilon$ for both the algorithms. However, the interesting part is, for PW the sample complexity becomes almost constant beyond a certain threshold of $\epsilon$ (precisely when $\epsilon$ falls below $\Delta_{\min}$) in every case, whereas for DnB it keeps on scaling in $O(\frac{1}{\epsilon^2})$ 
irrespective of the `hardness' of the underlying PL environment due to its non-adaptive nature---this is the region where we excel out. Also, note that the harder the dataset (i.e. the smaller its $\Delta_{\min}$), the smaller this threshold is, as follows from Thm. \ref{thm:sc_fewfep}, which verifies the instance-adaptive nature of our PW algorithm as it terminates as soon as $\epsilon$ falls below $\Delta_{\min}$.

\begin{figure}[h!]
	\begin{center}
		\includegraphics[trim={0cm 0 0cm 0},clip,scale=0.25,width=0.435\textwidth]{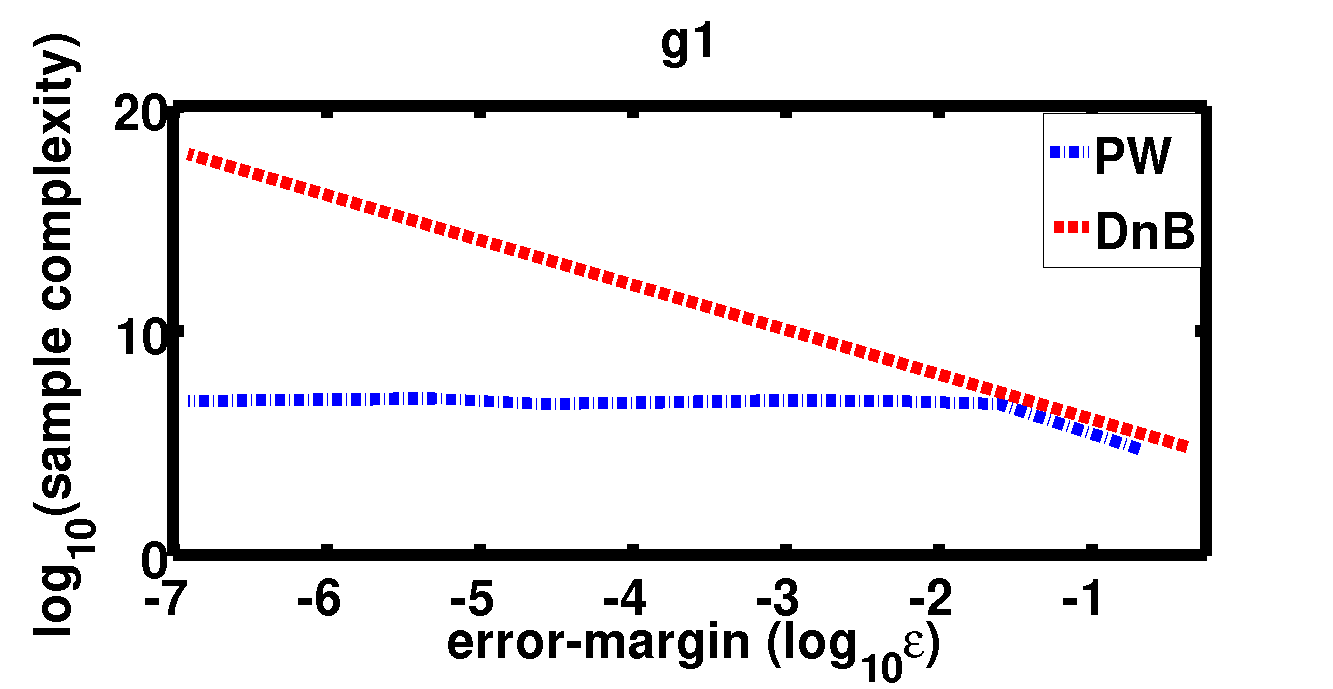}
		\hspace{0pt}
		\includegraphics[trim={0.cm 0 0cm 0},clip,scale=0.25,width=0.435\textwidth]{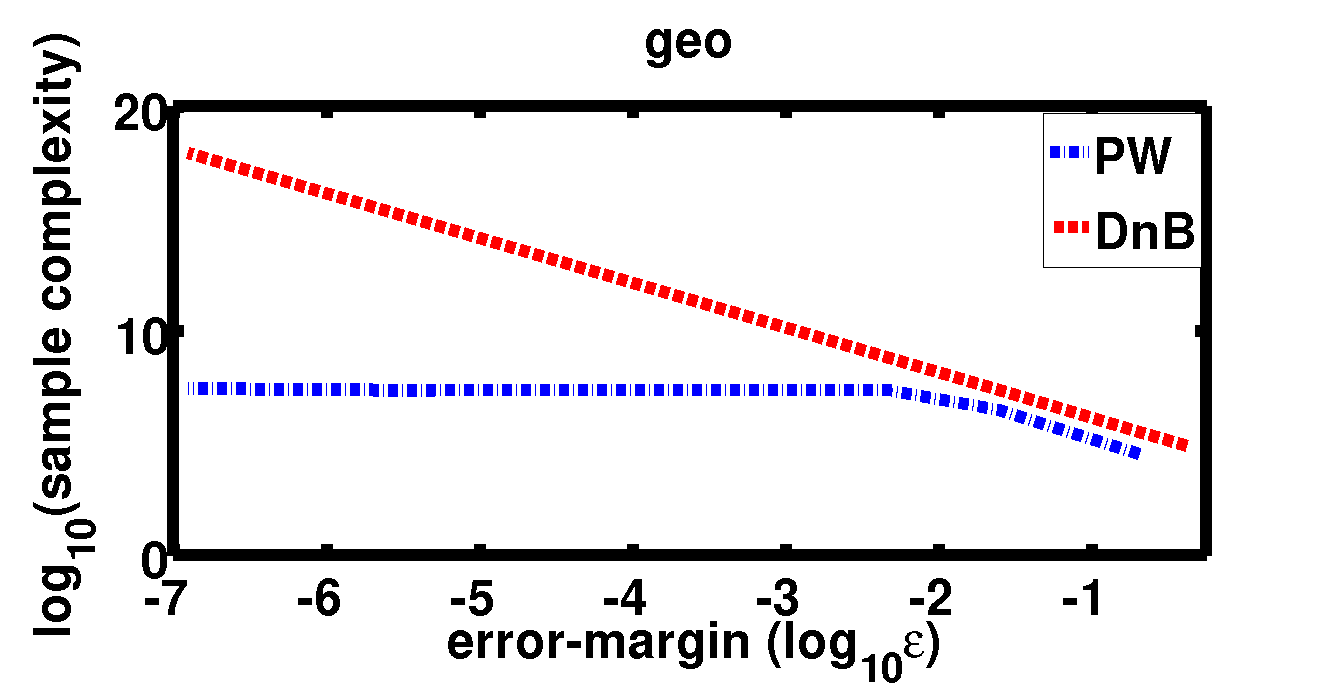}
		\hspace{0pt}
		\includegraphics[trim={0cm 0 0cm 0},clip,scale=0.25,width=0.435\textwidth]{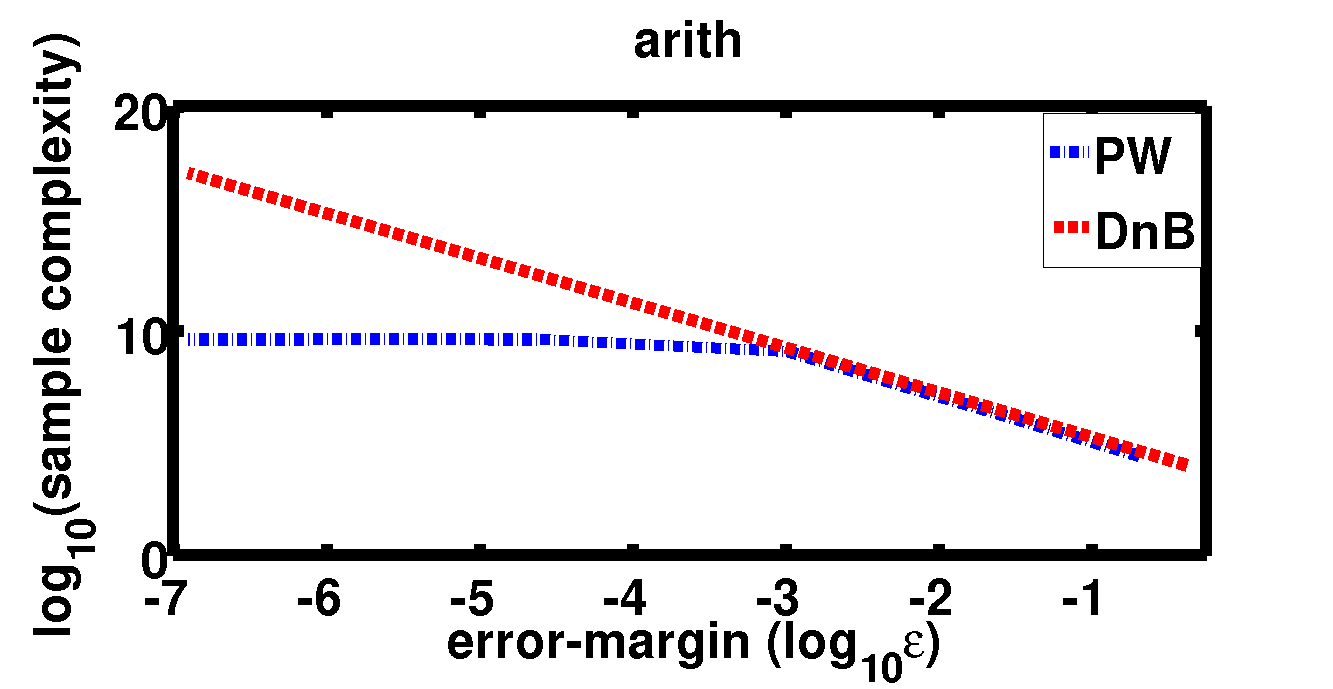}
		\hspace{0pt}
		\includegraphics[trim={0cm 0 0cm 0},clip,scale=0.25,width=0.435\textwidth]{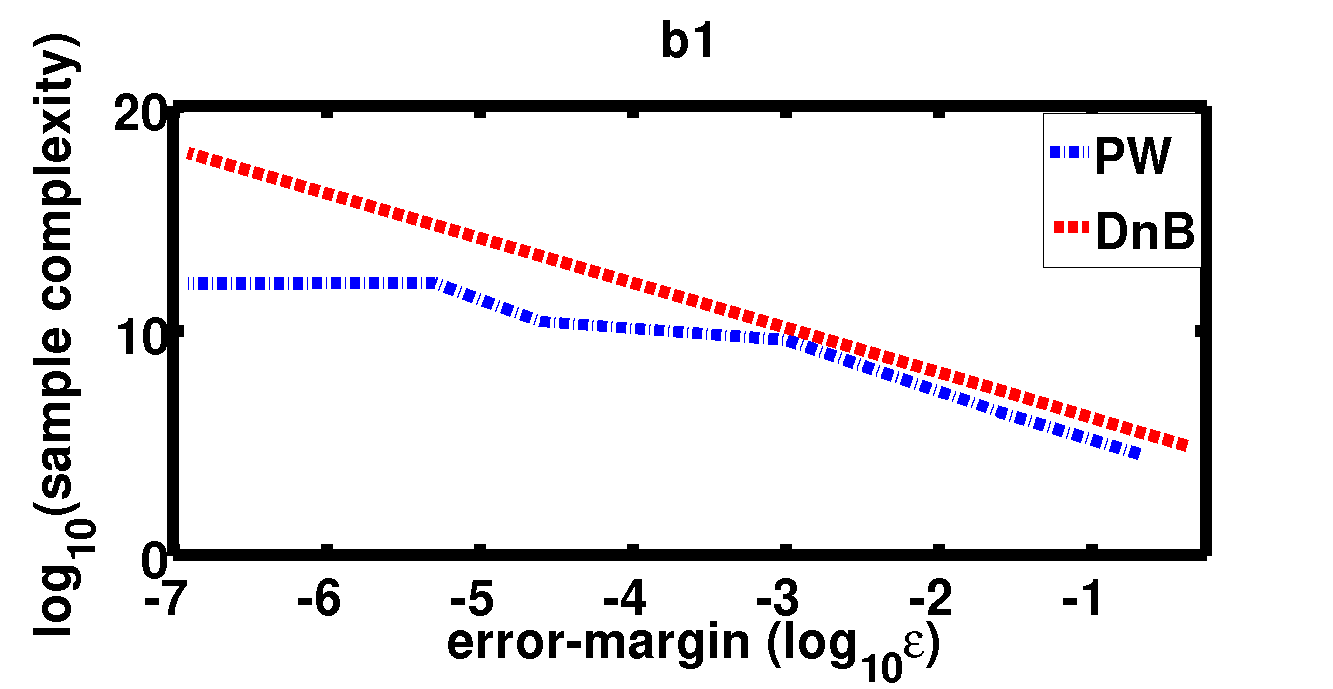}
		\vspace{-10pt}
		\caption{Sample-Complexity vs Error-Margin  $(\epsilon)$ (both in $\log$ scale) of PW and DnB across $4$ different problem instances.}
		\label{fig:sc_vs_eps}
		\vspace{-1pt}
	\end{center}
\end{figure}

\textbf{Itemwise sample complexity.} This experiment reveals the survival time of the items (i.e. total number plays of an item before elimination) in \algfewf\, algorithm. The results in Fig. \ref{fig:itemsc} clearly shows the inverse dependency of the survival time of items w.r.t. their $\theta$ parameter, e.g. for \textbf{g4} dataset, the survival times of the items are categorized into $4$ groups, highest for item $1$, with items $2$-$6$, $7$-$11$, and $12$-$16$ following it---justifying the $O\big(\frac{1}{\Delta_i^2}\big)$ survival times for each item $i$ (in Thm. \ref{thm:sc_fewf} or \ref{thm:sc_fetf}). 

\begin{figure}[h!]
	\begin{center}
		\includegraphics[trim={0cm 0 0cm 0},clip,scale=0.25,width=0.435\textwidth]{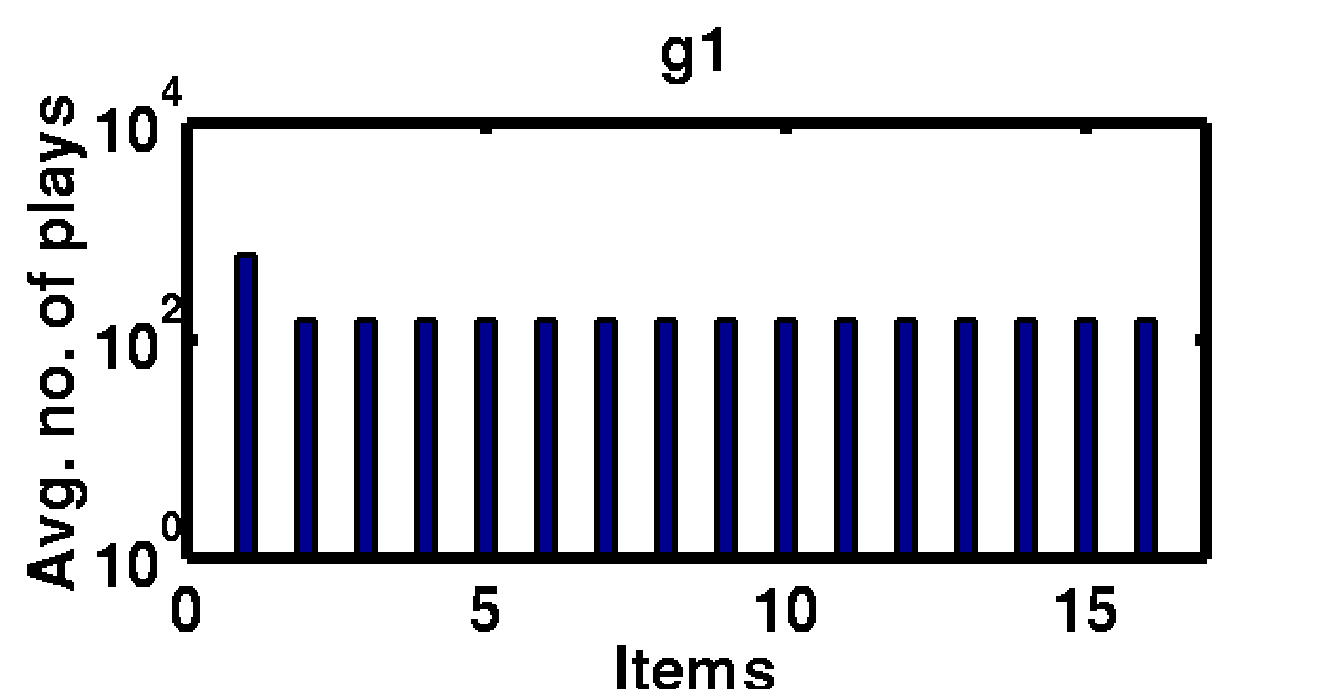}
		\hspace{0pt}
		\includegraphics[trim={0.cm 0 0cm 0},clip,scale=0.25,width=0.435\textwidth]{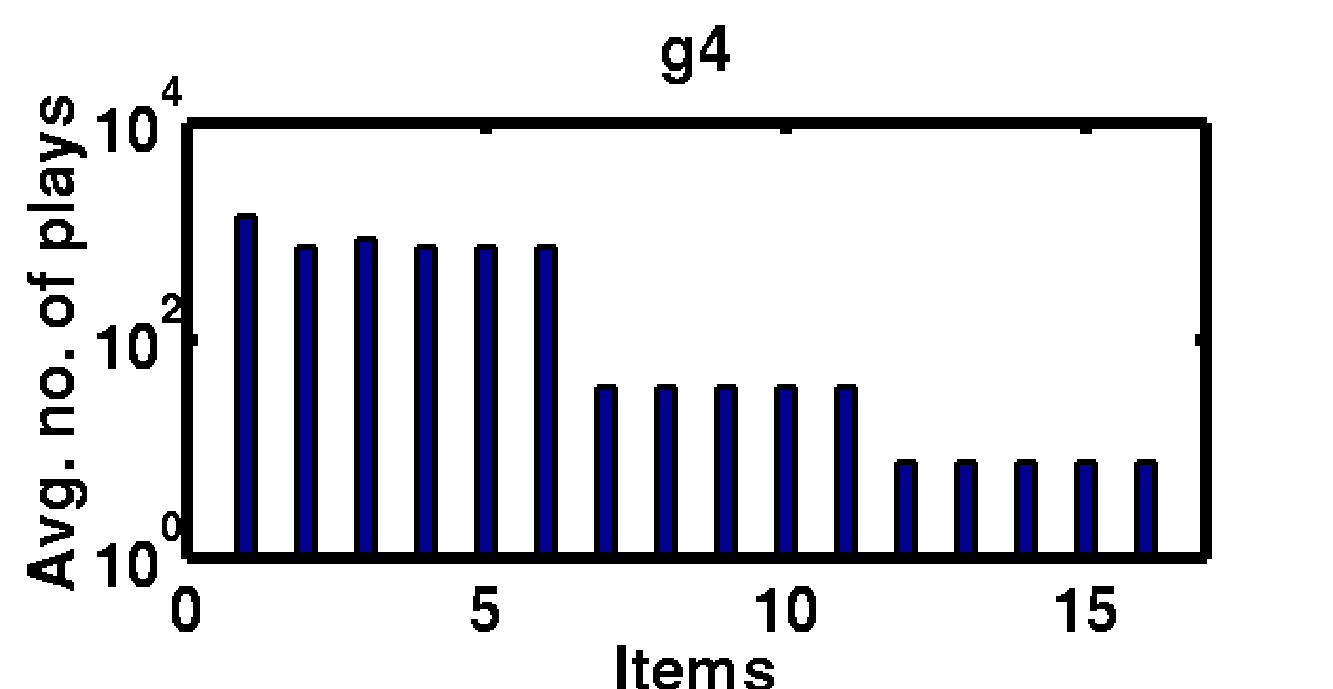}
		\hspace{0pt}
		\includegraphics[trim={0cm 0 0cm 0},clip,scale=0.25,width=0.435\textwidth]{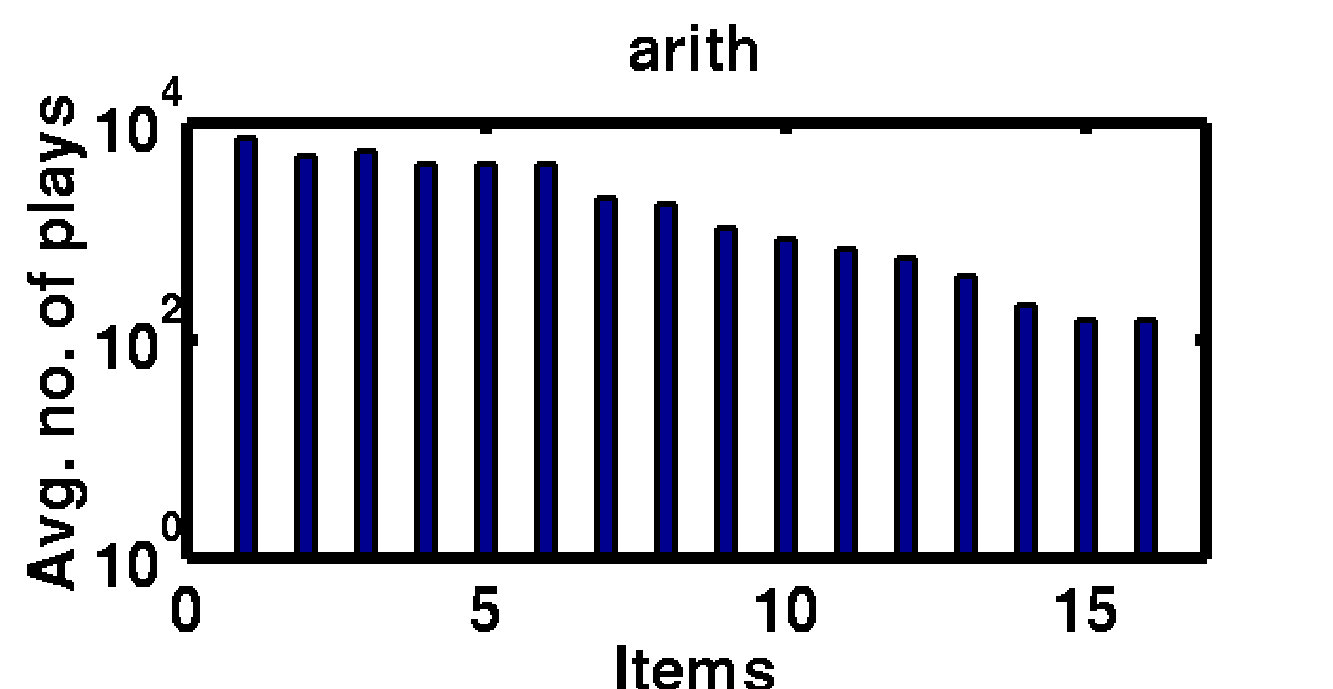}
		\hspace{0pt}
		\includegraphics[trim={0cm 0 0cm 0},clip,scale=0.25,width=0.435\textwidth]{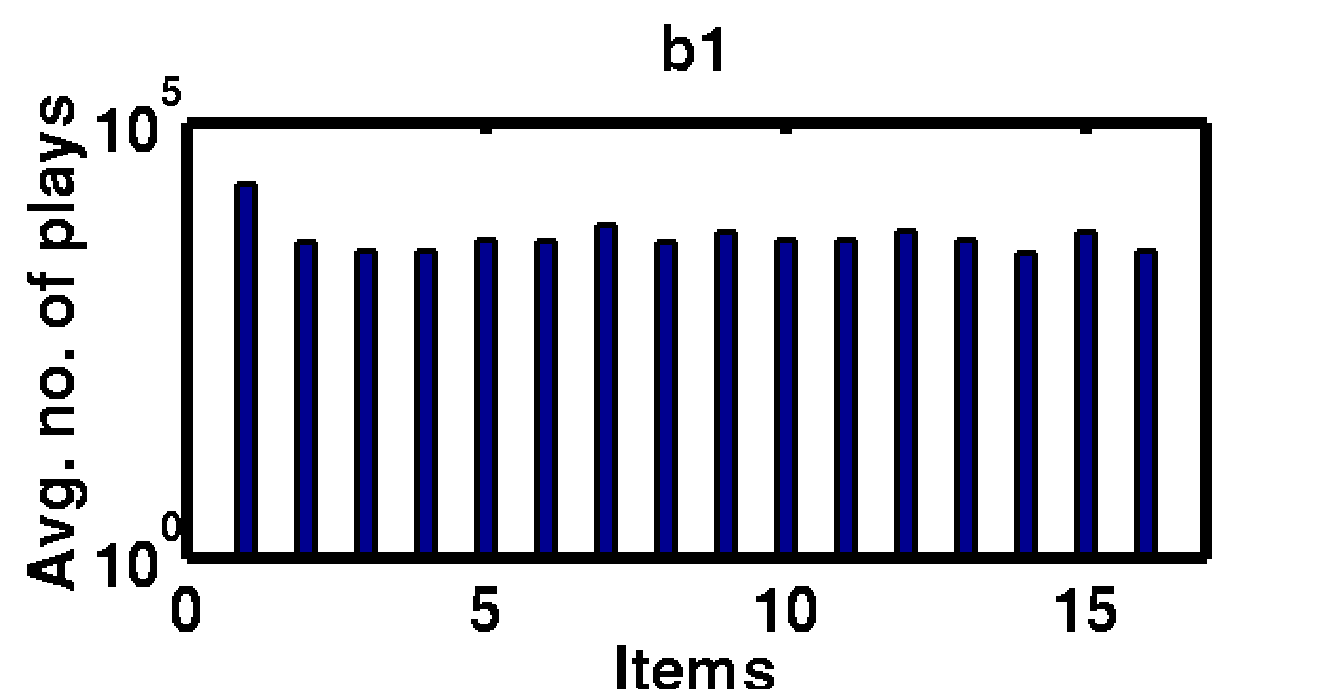}
		\vspace{-10pt}
		\caption{Survival time of different items (Itemwise sample complexity) in PW on $4$ different problem instances.}
		\label{fig:itemsc}
		\vspace{-5pt}
	\end{center}
\end{figure}

\textbf{Tradeoff: Sample-Complexity vs size of Top-ranking Feedback $m$.} In this case we verified the flexibility of \algfewf\, for \tf \,(Alg. \ref{alg:fetf}). We run it on different datasets with increasing size of top-ranking feedback ($m$). Again, justifying the claims of Thm. \ref{thm:sc_fetf}, Fig. \ref{fig:sc_vs_m} shows the sample complexity varies at a rate of $\frac{1}{m}$ (note that as $m$ is doubled, sample complexity gets about halved), while rest of the parameters (i.e. $k, \delta, \epsilon$) are kept unchanged.

\vspace{-10pt}
\begin{figure}[h!]
	\begin{center}
		\includegraphics[trim={0cm 0 0cm 0},clip,scale=0.25,width=0.435\textwidth]{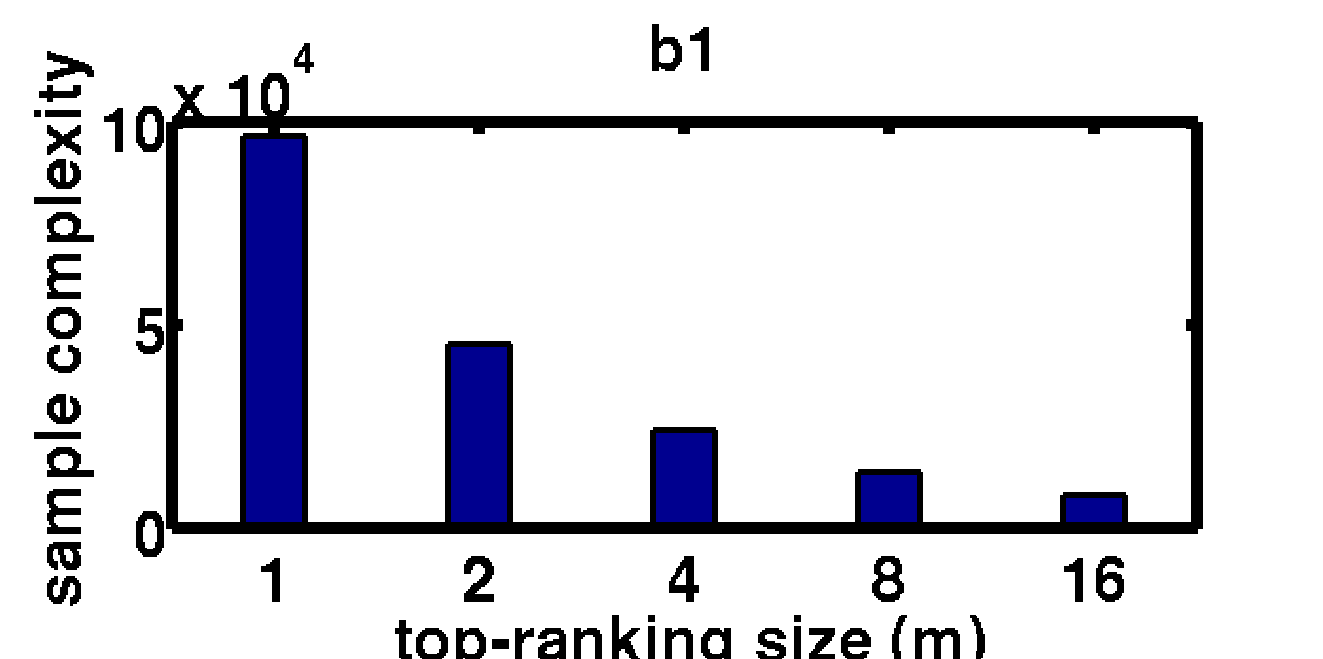}
		\hspace{0pt}
		\includegraphics[trim={0.cm 0 0cm 0},clip,scale=0.25,width=0.435\textwidth]{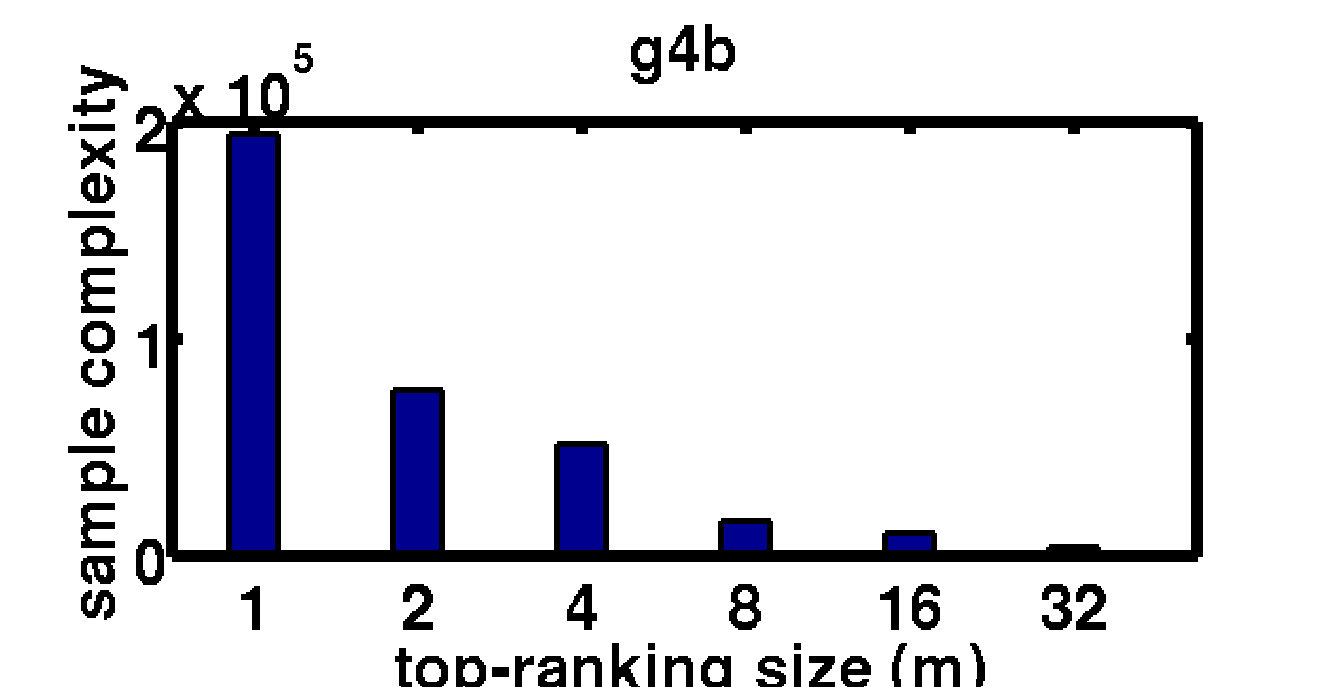}
		\hspace{0pt}
		\includegraphics[trim={0cm 0 0cm 0},clip,scale=0.25,width=0.435\textwidth]{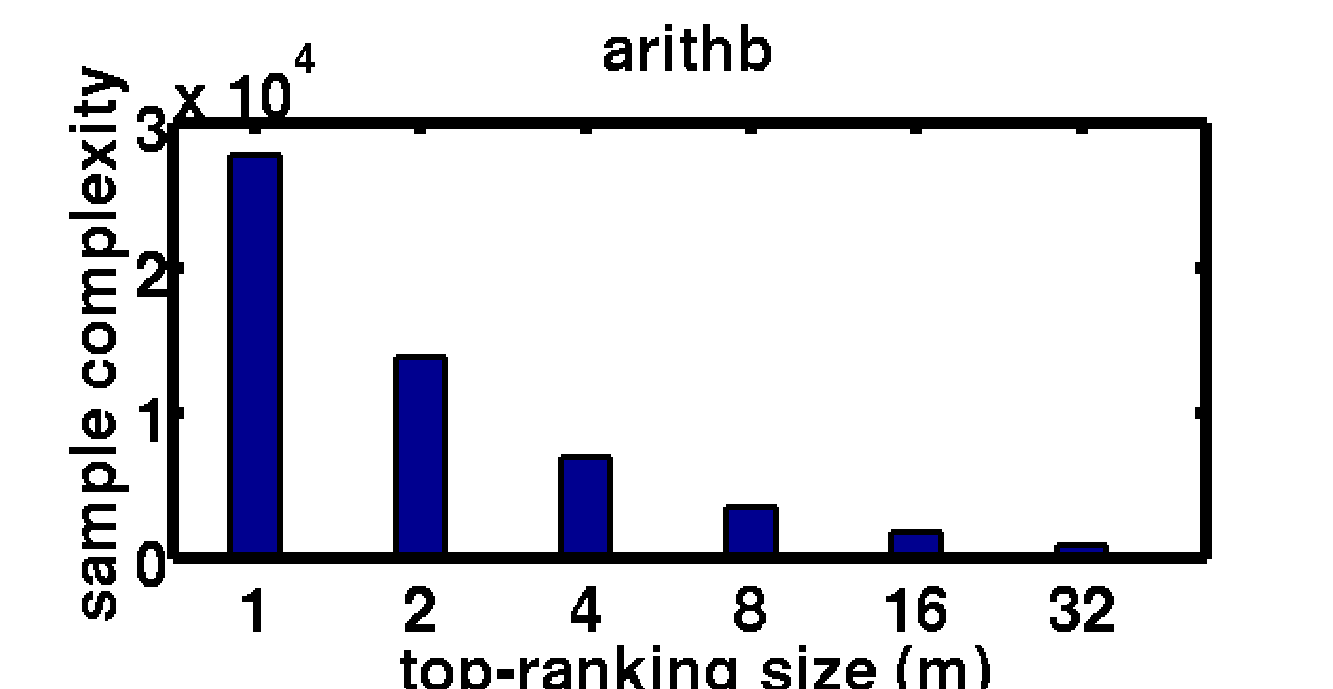}
		\hspace{0pt}
		\includegraphics[trim={0cm 0 0cm 0},clip,scale=0.25,width=0.435\textwidth]{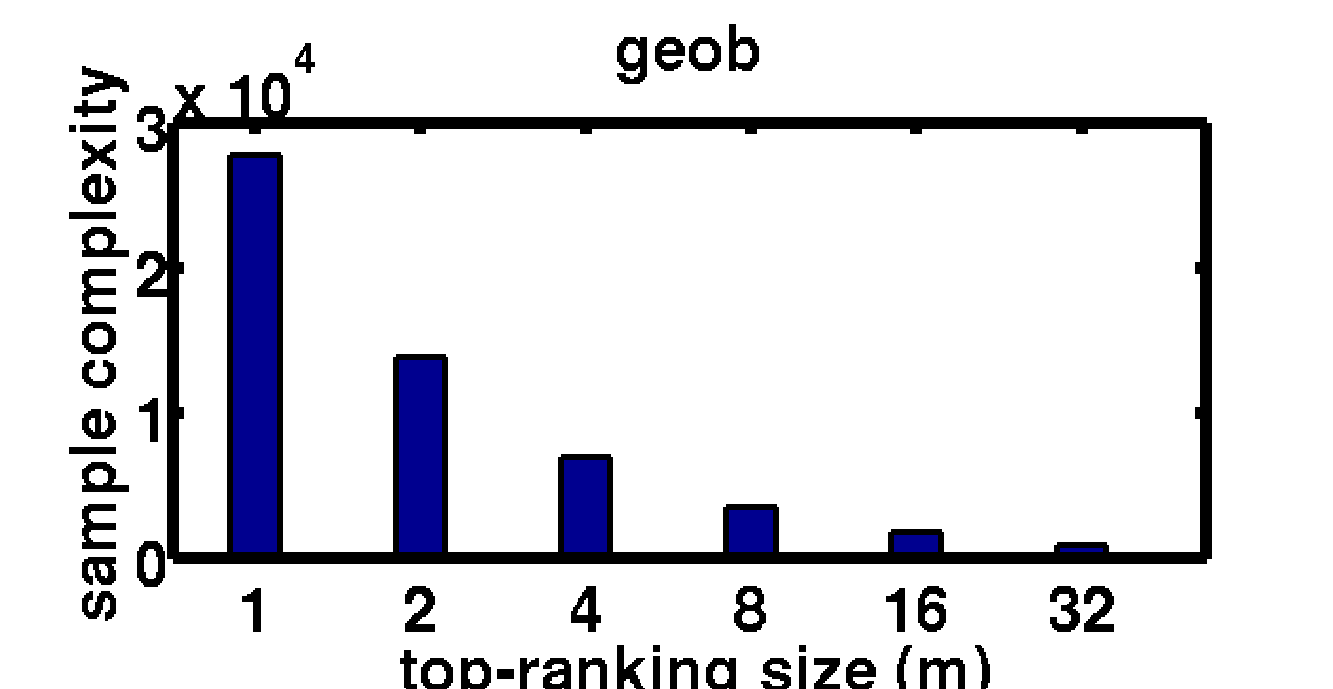}
		\vspace{-10pt}
		\caption{Sample-Complexity vs length of rank ordered feedback $(m)$ of PW for $4$ different problem instances.}
		\label{fig:sc_vs_m}
		\vspace{-15pt}
	\end{center}
\end{figure}

\subsection{Results: \ft\, setting}
\textbf{Success probability ($1-\delta$) vs Sample-Complexity ($\sc$).} Finally we analysed the success probability $(1-\delta)$ of algorithm \algfttf \, (UA) for varying sample complexities $(Q)$, keeping $\epsilon$ fixed at $(\Delta_{\min})/2$. Fig. \ref{fig:del_vs_sc} shows that the algorithm identifies the \bi\, with higher confidence with increasing $Q$---justifying its $O(\exp(-Q)$ error confidence rate as proved in Thm. \ref{thm:pr_fttf}. 
Note that \textbf{g4} being the easiest instance, it reaches the maximum success rate $1$ at a much smaller $Q$, compared to the rest. 
By construction, DnB is not designed to operate in \ft\, setup, but due to lack of any other existing baseline, we still use it for comparison force terminating it if the specified sample complexity is exceeded, and as expected,
here again it performs poorly in the lower sample complexity region.

\vspace{-10pt}
\begin{figure}[h!]
	\begin{center}
		\includegraphics[trim={0cm 0 0cm 0},clip,scale=0.25,width=0.435\textwidth]{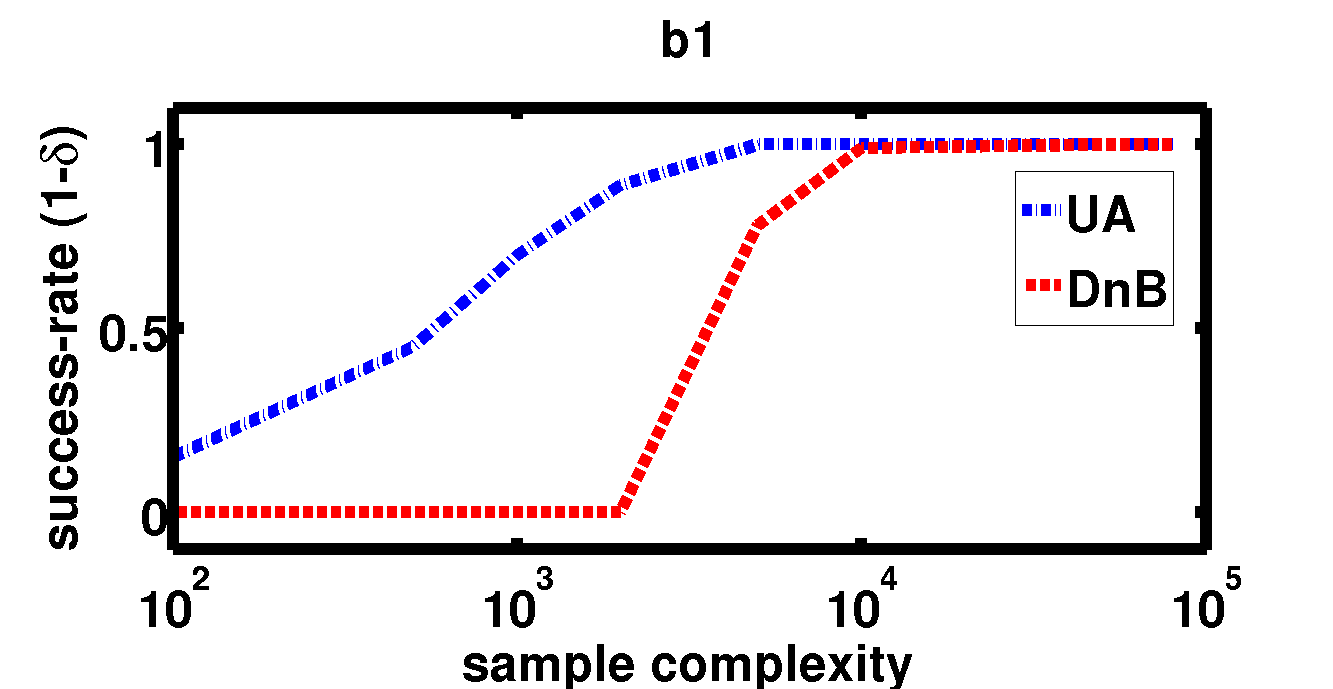}
		\hspace{0pt}
		\includegraphics[trim={0.cm 0 0cm 0},clip,scale=0.25,width=0.435\textwidth]{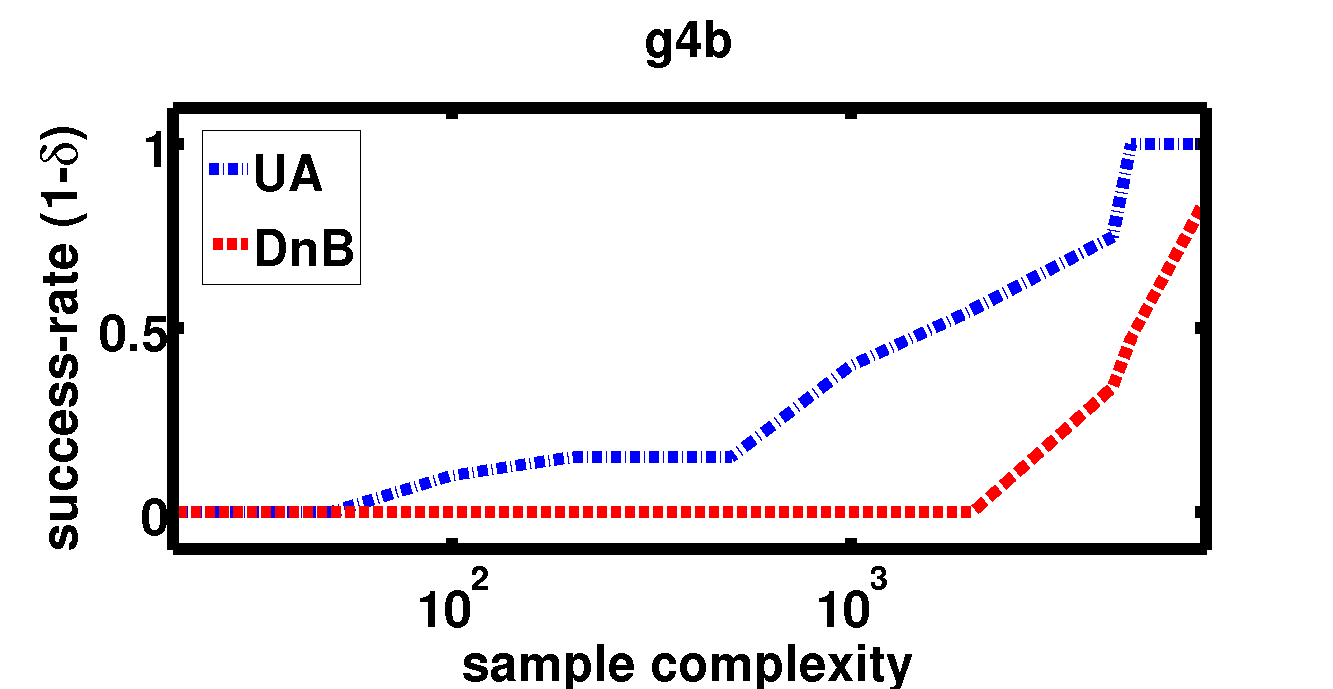}
		\hspace{0pt}
		\includegraphics[trim={0cm 0 0cm 0},clip,scale=0.25,width=0.435\textwidth]{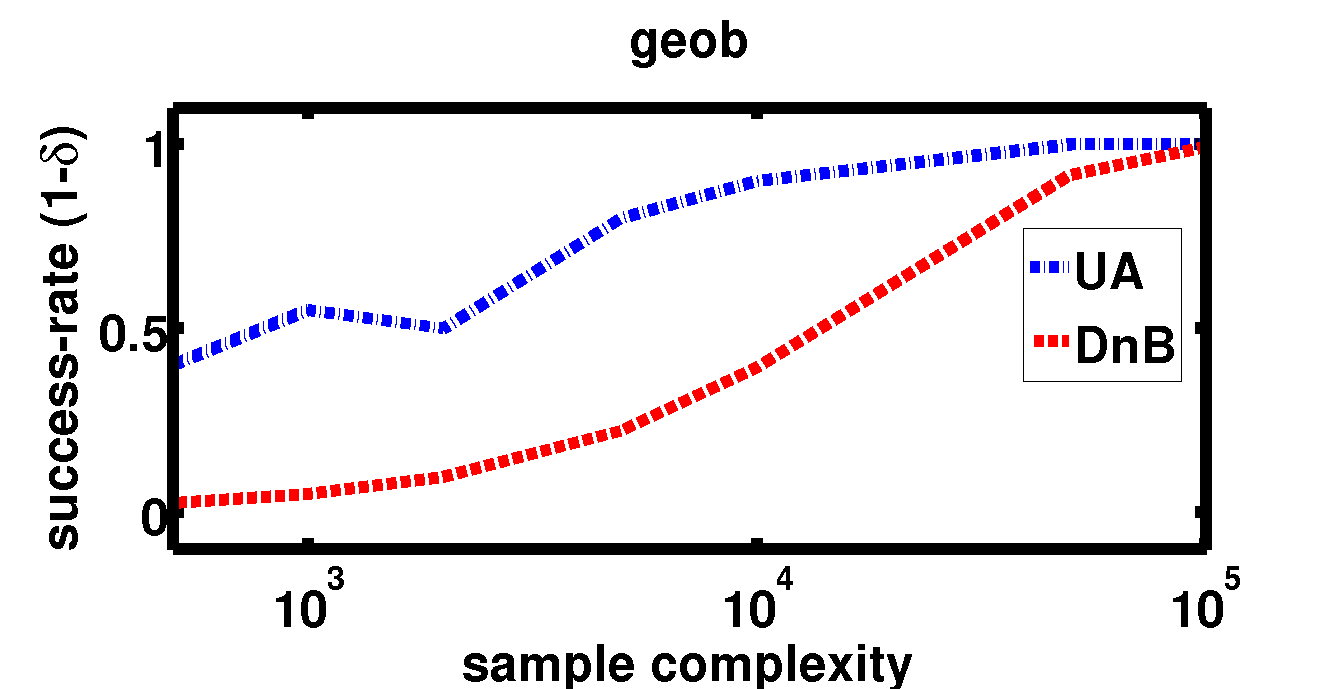}
		\hspace{0pt}
		\includegraphics[trim={0cm 0 0cm 0},clip,scale=0.25,width=0.435\textwidth]{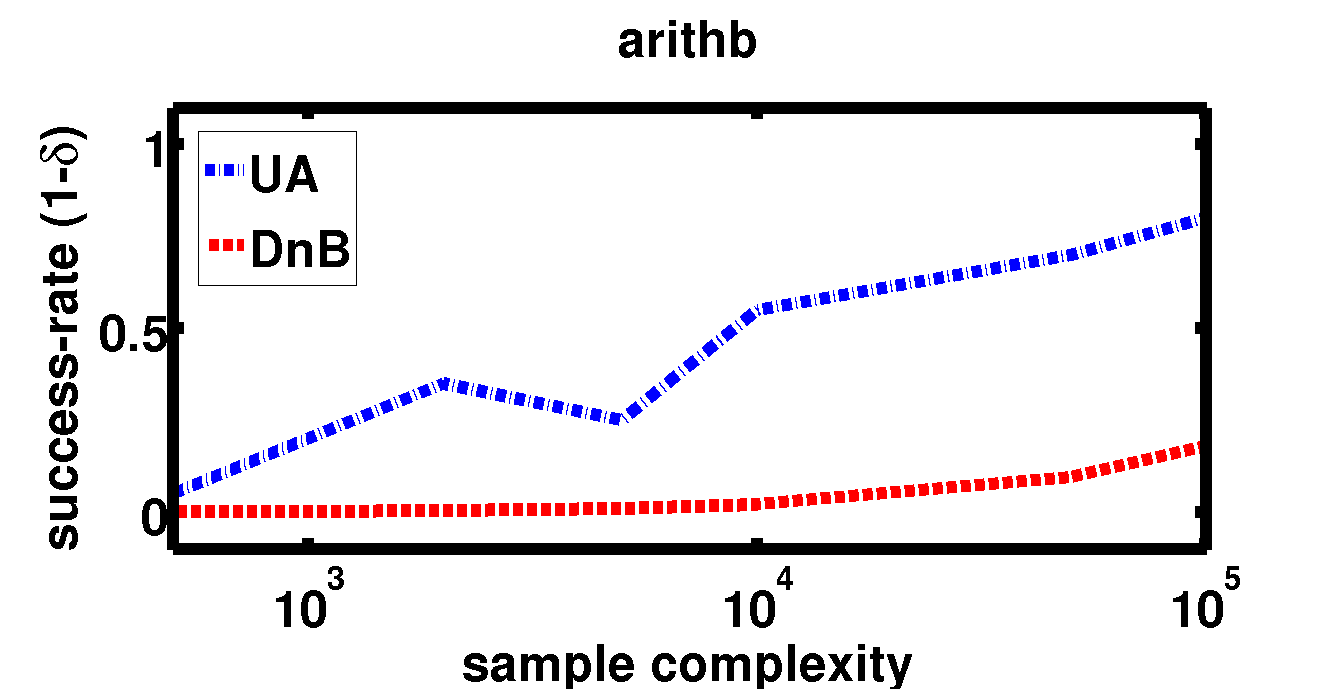}
		\vspace{-10pt}
		\caption{Comparative performances of PW and DnB in terms of Success probability ($1-\delta$) vs Sample-Complexity ($Q$) across $4$ different problem instances.}
		\label{fig:del_vs_sc}
		\vspace{-15pt}
	\end{center}
\end{figure}

%% file: conclusion.tex
\vspace*{-5pt}

\section{Conclusion and Future Work}
\label{sec:conclusion}
\vspace*{-5pt}
 Moving forward, it would be interesting to explore similar algorithmic and statistical questions in the context of other common subset choice models such as the Mallows model, Multinomial Probit, etc. It would also be of great practical interest to develop efficient algorithms for large item sets, especially when there is structure among the parameters to be exploited. One can also aim to develop instant dependent guarantees for other `learning from relative feedback' objectives, e.g. PAC-ranking \cite{Busa_pl}, top-set identification \cite{Busa_top} etc., both in fixed confidence as well as fixed budget setting.

%% file: appendix.tex
\appendix
\onecolumn{

\section*{\centering\Large{Supplementary: From PAC to Instance Optimal Sample Complexity in the Plackett-Luce Model}}
\vspace*{1cm}

\allowdisplaybreaks

\section{Related Works}
\label{app:rel}

{\bf Related work.} 
For classical multiarmed bandits setting, there is a well studied literature on PAC-arm identification problem \cite{Even+06,Audibert+10,Kalyanakrishnan+12,Karnin+13,LilUCB}, where the learner gets to see a noisy draw of absolute reward feedback of an arm upon playing a single arm per round.
Some of the existing results on dueling bandits line of works also focuses on PAC learning from pairwise preference feedback for best arm identification problem
\cite{BTM,SAVAGE,Busa_pl,Busa_mallows}, or even more general problem objectives e.g. PAC top set recovery \cite{Busa_top,MohajerIcml+17,ChenSoda+17}, or PAC-ranking of items \cite{Busa_aaai,falahatgar_nips}, even in the feedback setup of noisy comparisons \cite{braverman+08,nisarg+13}. There are also very few recent developments that focuses on learning for subsetwise feedback in an online setup \cite{Sui+17,Brost+16,SG18,SGwin18,Ren+18,ChenSoda+18}.
Some of the existing work also explicitly consider the Plackett-Luce parameter estimation problem with subset-wise feedback but for offline setup only \cite{SueIcml+17,KhetanOh16}. 
While most of the above work address the ($\epsilon,\delta$)-PAC recovery problem, i.e. finding an `$\epsilon$-approximation' of the desired (set of) item(s) with probability at least $(1-\delta)$, few of them also focuses of instant dependent PAC recovery guarantees where the sample complexity explicitly depends of the parameters of the underlying model, e.g. for classical multiarmed bandits \cite{Audibert+10,Karnin+13,Kalyanakrishnan+12}, or even for preference based bandits \cite{Busa_pl,ChenSoda+18}.

\section{Appendix for Sec. \ref{sec:fe}}


\subsection{Subroutines used in \algfewf\,(Alg. \ref{alg:fewf})}
\label{app:subrout}

\textbf{\algdiv\, subroutine:} Partition a given set of items $\cA$ into $B = \ceil{\frac{|\cA|}{k}}$ equally sized batches $\cB_1, \cB_2, \ldots \cB_B$, each of size at most $k$.

\begin{center}
\vspace*{-5pt}
\begin{algorithm}[h]
   \caption{\textbf{\algdiv} subroutine}
   \label{alg:div}
\begin{algorithmic}[1]
   \STATE {\bfseries Input:} Set of items: $\cA \subseteq [n]$, Batch size: $k \in [n]$
   \STATE $B \leftarrow \ceil{\frac{|\cA|}{k}}$
   \STATE Divide $\cA$ into $B$ subsets $\cB_1, \cB_2, \ldots \cB_B$ such that $\cB_i \cap \cB_j = \emptyset$, $\cup_{i=1}^B\cB_i = \cA$ and $|\cB_i| = k\,, \forall i \in [B-1]$
   \STATE {\bfseries Output:} $B$ batches $\cB_1, \cB_2, \ldots \cB_B$
\end{algorithmic}
\end{algorithm}
\end{center}

\noindent 
\textbf{\algs\, subroutine:} Our proposed algorithm  relies on a black-box subroutine for efficient estimation of  sum of the \pl\, model score parameters ($\theta$) of any given subset $S \subseteq [n]$, which we denote by $\Theta_S := \sum_{i \in S}\theta_i$. 
We achieve this with the subroutine \algs\, (Alg. \ref{alg:algs}) which requires a pivot element $b \in [n]$ to estimate the sum of the score parameters of the given set $S$ (i.e. $\Theta_S$): The algorithm simply plays the subset $S \cup \{b\}$ sufficiently many times and estimate $\Theta_S$ based on the relative win counts of items in $\Theta_S$ w.r.t. pivot $b$.
Under the assumption that $b \in [n]$ is a sufficiently good item such that $\theta_b > \frac{1}{2}$, Thm. \ref{thm:algs} shows Alg. \ref{alg:algs} successfully estimates the relative scores of any subset $S$ (upto multiplicative constants) with high confidence $(1-\delta_s)$.

\begin{center}
\vspace*{-5pt}
\begin{algorithm}[H]
   \caption{\textbf{\algs}$(S,b,\delta)$ subroutine}
   \label{alg:algs}
\begin{algorithmic}[1]
   \STATE {\bfseries Input:} Set of items: $S \subseteq [n]$, pivot $b$, and confidence parameter $\delta$
   \REPEAT
   \STATE Play $S\cup\{b\}$ and observe the winner.
   \UNTIL{$b$ wins for $d = 10\ln \frac{4}{\delta} $ times} 
   \STATE Let $T$ be the total number of plays of $S\cup\{b\}$, and let $Z = T - d$.
   \STATE {\bfseries Return: $\frac{Z}{d}$} 
\end{algorithmic}
\end{algorithm}
\end{center}

\begin{restatable}[{\algs\, high probability estimation guarantee}]{thm}{estalgs}
\label{thm:algs}
Let $\thets := \sum_{i \in S}\theta_i$. Given $\thets > \theta_b$, with probability at least $(1-\delta)$: 

i. the algorithm terminates in at most $\frac{10(\thetb+\thets)}{\thetb}\ln \frac{2}{\delta}$ rounds and,

ii. the output returned by \algs\, (Alg. \ref{alg:algs}) satisfies: 
\[
 \Big | \frac{Z}{d} - \frac{\thets}{\theta_b}\Big| \le \frac{1}{2}\max\bigg( \frac{\thets}{\theta_b},1\bigg)
\]
\end{restatable}

\begin{proof}
Let $X_i$ denotes the time iteration when $b$ wins for the $i^{th}$ time, $\forall i \in [d]$. Note that this implies $X_i \sim \text{Geometric}\Big(\frac{\theta_b}{\thets + \theta_b}\Big)$, $\forall i \in [d]$. Then
from Lem. $7$ of \cite{SGrank18}, we have for any $\eta > 0$,
\begin{align*}
Pr\Big( \Big| \frac{Z}{d} - \frac{\thets}{\theta_b} \Big| \ge \eta \Big) \le 2\exp\Bigg( - \frac{2 d \eta^2}{\Big( 1 + \frac{\thets}{\theta_b}\Big)\Big( \eta + 1 + \frac{\thets}{\theta_b} \Big)} \Bigg).
\end{align*}

We first want to get the right hand side $2\exp\Bigg( - \frac{2 d \eta^2}{\Big( 1 + \frac{\thets}{\theta_b}\Big)\Big( \eta + 1 + \frac{\thets}{\theta_b} \Big)} \Bigg) \le \frac{\delta}{2}$, which further implies to have $d \ge \frac{\Big( 1 + \frac{\thets}{\theta_b}\Big)\Big( \eta + 1 + \frac{\thets}{\theta_b} \Big)}{2\eta^2} \ln \frac{4}{\delta}$. Towards this we now would consider two cases:

\textbf{Case 1:} Suppose $\frac{\thets}{\theta_b} \ge 1$:
Then we can set $\eta = \frac{\thets}{2\thetb}$ and thus one must have: 
\begin{align*}
\frac{\Big( 1 + \frac{\thets}{\theta_b}\Big)\Big( \eta + 1 + \frac{\thets}{\theta_b} \Big)}{2\eta^2} \ln \frac{4}{\delta}
& = \frac{\Big( \frac{\theta_b}{\thets} + 1 \Big)\Big( \frac{\theta_b}{\thets} + \frac{3}{2} \Big)}{2(1/4)} \ln \frac{4}{\delta} \le \frac{\Big( 2 \Big)\Big( \frac{5}{2} \Big)}{2(1/4)} \ln \frac{4}{\delta} = 10 \ln \frac{4}{\delta} \le d
\end{align*}

\textbf{Case 2:} Suppose $\frac{\thets}{\theta_b} < 1$: In this case we may set $\eta = \frac{1}{2}$ so then it suffices to have
\begin{align*}
\frac{\Big( 1 + \frac{\thets}{\theta_b}\Big)\Big( \eta + 1 + \frac{\thets}{\theta_b} \Big)}{2\eta^2} \ln \frac{4}{\delta}
& \le \frac{\Big( 2 \Big)\Big( \frac{5}{2} \Big)}{2/4} \ln \frac{4}{\delta} = 10 \ln \frac{4}{\delta}
\end{align*}

Thus combining both cases we get with probability at least $(1-\delta/2)$: 
$\Big| \frac{Z}{d} - \frac{\thets}{\theta_b} \Big| < \frac{1}{2}\max\bigg( \frac{\thets}{\theta_b},1\bigg)$.

So we are only left to prove the required sample complexity of \algs\, to yield $d = 10\ln \frac{4}{\delta}$ wins of $\thetb$. For this, note that at any round item $b$ wins with probability $\frac{\thetb}{\thets+\thetb}$. So for any fixed $\tau$ rounds $E[d] =  \frac{\thetb \tau}{\thets+\thetb}$. Then applying multiplicative Chernoff bounds, we know that for any $\epsilon \in (0,1)$,
\[
Pr\Big( d \le (1-\epsilon)E[d] \Big) \le \exp\Big(-\frac{E[d]\epsilon^2}{2}\Big), 
\]
which implies whenever $\tau \ge \frac{20(\thets + \thetb)}{\thetb \epsilon^2}\ln \frac{4}{\delta}$, $d \ge (1-\epsilon)\frac{\tau \thetb}{\thetb + \thets}$ with probability at least $(1-\delta/4)$. Finally noting that we need $d \ge 10\ln \frac{4}{\delta}$, this implies we can easily set $\epsilon = \frac{1}{2}$ so that
\begin{align*}
d & \ge (1-\epsilon)\frac{20(\thets+\thetb)}{\thetb\epsilon^2}\frac{\thetb}{\thetb + \thets}\ln \frac{4}{\delta} \ge 10 \ln \frac{4}{\delta},
\end{align*}
and the claim follows.
\end{proof}

\begin{restatable}[]{lem}{lemestalgs}
\label{lem:algs}
Let us denote $\thetsh: =  \max($\algs$(S,b,\delta),1)$.
Consider the notations introduced in Thm. \ref{thm:algs}, 
Then with probability at least $(1-\delta)$, $ \max(1, \frac{\thets}{2\theta_b}) \le \frac{Z}{d} \le \max( \frac{\thets}{2\theta_b}, \frac{3}{2}) $, and the algorithm \algs\, terminates in at most ${40(\thets + \theta_b)}\ln \frac{4}{\delta}$ many iterations.
\end{restatable}

\begin{proof}
The proof directly follows from Thm. \ref{thm:algs}. 
\end{proof}

\begin{restatable}[]{cor}{corestalgs}
\label{cor:algs}
Let $S \subseteq [n],\, |S| = k$, and let $\thetk = \max_{S \subseteq [n]\mid |S|=k}\sum_{i \in S}\theta_i$. With the notation of Lem. \ref{lem:algs}, if $b$ is an $\frac{1}{2}$-optimal item such that $\theta_b > \theta_1 - \epsilon$ for any $\epsilon \in \Big( 0, \frac{1}{2}\Big]$, then with probability at least $(1-\delta)$, $ \max(1, \thets/2) \le \frac{Z}{d} \le 6 \thetk $, and the  \algs\, algorithm terminates in at most ${80\thetk }\ln \frac{4}{\delta}$  iterations.
\end{restatable}

\begin{proof}
The proof directly follows from Lem. \ref{lem:algs}, with noting that by definition for any $S \subseteq [n],\, |S| = k$,  $\thets \le \thetk$, $\theta_b > \frac{1}{2}$ and $\thetk \ge 1$, since we assume $\theta_1 = 1$ and of course $\theta_1 \in \thetk$.
\end{proof}

\noindent
\textbf{Rank-Breaking Subroutine} \cite{AzariRB+14,KhetanOh16}. 
This is a procedure of deriving pairwise comparisons from multiwise (subsetwise) preference information. Formally, given any set $S \subseteq [n]$, $m \le |S| < n$, if $\bsigma \in \bSigma_{S}^m$ denotes a possible \tf\, of $S$, \rb\, considers each item in $S$ to be beaten by its preceding items in $\bsigma$ in a pairwise sense and extracts out total $\sum_{i = 1}^{m}(k-i) = \frac{m(2k-m-1)}{2}$ such pairwise comparisons. For instance, given a full ranking of a set of $4$ elements $S = \{a,b,c,d\}$, say $b \succ a \succ c \succ d$, Rank-Breaking generates the set of $6$ pairwise comparisons: $\{(b\succ a), (b\succ c), (b\succ d), (a\succ c), (a\succ d), (c\succ d)\}$. %

\vspace*{-5pt}
\begin{center}
\begin{algorithm}[h]
   \caption{\rb\, subroutine}
   \label{alg:updt_win}
\begin{algorithmic}[1]
   \STATE {\bfseries Input:} Subset $S \subseteq [n]$, such that $|S| = k$ ($n\ge k$) 
   \STATE ~~~ A top-$m$ ranking $\bsigma \in \bSigma_{S}^m$, for some $m \in [k-1]$
   \STATE ~~~ Pairwise win-count $w_{ij}$ for each item pair $i,j \in S$
   \WHILE {$\ell = 1, 2, \ldots m$}
   	\STATE Update $w_{\sigma(\ell)i} \leftarrow w_{\sigma(\ell)i} + 1$, for all $i \in S \setminus\{\sigma(1),\ldots,\sigma(\ell)\}$
	\ENDWHILE
\end{algorithmic}
\end{algorithm}
\end{center}
\vspace*{-5pt}

\subsection{Pseudo-code for \algep} 
\label{app:wiub_div}

\textbf{Algorithm description:} The algorithm \algep\, first divides the set of $n$ items into batches of size $k$, and plays each group sufficiently long enough, until a single item of that group stands out as the empirical winner in terms of its empirical pairwise advantage over the rest (again estimated through \rb). It then just retains this empirical winner for every group and recurses on the set of surviving winners, until only a single item is left behind, which is declared as the $(\epsilon,\delta)$-PAC item. The 
Its important to note that the sample complexity of our algorithm (see Thm. \ref{lem:up_algep}) offers an improved instance dependent guarantee (compared to the $O\Big( \frac{n}{m\epsilon^2}\ln \frac{k}{\delta} \Big)$ sample complexity algorithm \algdnb\, proposed by \cite{SGwin18}), which would turn out to be crucial for the \emph{instance-dependent sample-complexity} analyses of our main algorithms, Alg. \ref{alg:fewf} or \ref{alg:fetf}, later. (See proof of Thm. \ref{thm:sc_fewf} and \ref{thm:sc_fetf} respectively for details.) Though our proposed algorithm proceed along a line similar to \algdnb\, of \cite{SGwin18}, the crux of our proposed algorithm lies in sampling each subset $\cG_g$ just sufficiently enough in an adaptive way for only $O\bigg(\frac{\thetg}{m\epsilon^2} \ln \frac{2n}{\delta}\bigg)$ times---thanks to our sum estimation routine \algs\, (Alg. \ref{alg:algs})---instead of sampling them blindly for $O\bigg(\frac{k}{m\epsilon^2} \ln \frac{2n}{\delta}\bigg)$ times as proposed in \algdnb. To find the $(1/2,\delta)$-optimal item: $b \in [n]$ required to estimate $\thetg$, we can use the existing algorithms like \algdnb.
The complete algorithm is described in Alg. \ref{alg:epsdel}.


\begin{center}
\begin{algorithm}[H]
   \caption{\textbf{\algep} (for {TR} feedback) }
   \label{alg:epsdel}
\begin{algorithmic}[1]
   \STATE {\bfseries Input:} 
   \STATE ~~~ Set of items: $[n]$, and subset size: $k > 2$ ($n \ge k \ge m$)
   \STATE ~~~ Error bias: $\epsilon >0$, and confidence parameter: $\delta >0$
    \STATE ~~~ A $(1/2,\delta)$-optimal item: $b \in [n]$, such that $\theta_b > \frac{1}{2}$
   \STATE {\bfseries Initialize:} 
   \STATE ~~~ $S \leftarrow [n]$, $\epsilon_0 \leftarrow \frac{\epsilon}{8}$, and $\delta_0 \leftarrow \frac{\delta}{2}$  
   \STATE ~~~ Divide $S$ into $G: = \lceil \frac{n}{k} \rceil$ sets $\cG_1, \cG_2, \cdots \cG_G$ such that $\cup_{j = 1}^{G}\cG_j = S$ and $\cG_{j} \cap \cG_{j'} = \emptyset, ~\forall j,j' \in [G], \, |G_j| = k,\, \forall j \in [G-1]$.
    \textbf{If} $|\cG_{G}| < k$, \textbf{then} set $\cR_1 \leftarrow \cG_G$  and $G = G-1$, \textbf{Else} $\cR_1 \leftarrow \emptyset$.
   \WHILE{$\ell = 1,2, \ldots$}
   \STATE Set $\delta_\ell \leftarrow \frac{\delta_{\ell-1}}{2}, \epsilon_\ell \leftarrow \frac{3}{4}\epsilon_{\ell-1}$
   \FOR {$g = 1,2, \cdots G$}
   \STATE $\thetg \leftarrow $ \algs$\big(\cG_g,b,\delta_\ell \big)$
    \STATE Initialize pairwise (empirical) win-count $w_{ij} \leftarrow 0$, for each item pair $i,j \in \cG_g$
	\FOR {$\tau = 1, 2, \ldots t := \left\lceil \frac{16\thetg}{m\epsilon_l^2} \ln \frac{2k}{\delta_\ell} \right\rceil$  }
   	\STATE Play the set $\cG_g$ (one round of battle)
   	\STATE Receive feedback: The top-$m$ ranking $\bsigma \in \bSigma_{\cG_{g}}^m$ 
   	\STATE Update win-count $w_{ij}$ of each item pair $i,j \in \cG_g$ using \rb$(\cG_g,\bsigma)$
   	\ENDFOR 
	\STATE Estimate pairwise win probabilities: $\forall i,j \in \cG_g \; \hat p_{i,j} = \frac{w_{ij}}{w_{ij}+w_{ji}}$ if $w_{ij}+w_{ji} > 0$, $\hat p_{i,j} = \frac{1}{2}$ otherwise 
   	\STATE If $\exists i \in \cG_g$ such that $\hp_{i j} + \frac{\epsilon_\ell}{2} \ge \frac{1}{2}, \, \forall j \in \cG_g$, then set $c_g \leftarrow i$, else select $c_g \leftarrow$ uniformly at random from $\cG_g$. Set $S \leftarrow S \setminus \left( \cG_g \setminus  \{c_g\}\right)$
   \ENDFOR
   \STATE $S \leftarrow S \cup \cR_\ell$
   \IF{$(|S| == 1)$}
   \STATE Break (out of the {\bf while} loop)
   \ELSIF{$|S|\le k$}
   \STATE $S' \leftarrow $ Randomly sample $k-|S|$ items from $[n] \setminus S$, and set $S \leftarrow S \cup S'$, $\epsilon_\ell \leftarrow \frac{2\epsilon}{3}$, $\delta_\ell \leftarrow {\delta}$  
 \ENDIF
   \STATE Divide $S$ into $G: = \big \lceil \frac{|S|}{k} \big \rceil$ sets $\cG_1, \cdots \cG_G$ such that $\cup_{j = 1}^{G}\cG_j = S$, $\cG_{j} \cap \cG_{j'} = \emptyset, ~\forall j,j' \in [G], \, |G_j| = k,\, \forall j \in [G-1]$. \textbf{If} $|\cG_{G}| < k$, \textbf{then} set $\cR_{\ell+1} \leftarrow \cG_G$  and $G = G-1$, \textbf{Else} $\cR_1 \leftarrow \emptyset$.
  
   \ENDWHILE
   \STATE {\bfseries Output:} $r_*$, the single item remaining in $S$ 
\end{algorithmic}
\end{algorithm}
\vspace{-10pt}
\end{center}

\ubalgep*


\begin{proof}
For notational convenience we will use $\tp_{ij} = p_{ij} - \frac{1}{2}, \, \forall i,j \in [n]$.

We start by recalling a lemma from \cite{SGwin18} which will be used crucially in the analysis:

\begin{lem}\cite{SGwin18}
\label{lem:pl_sst}
For any three items $a,b,c \in [n]$ such that $\theta_a > \theta_b > \theta_c$, if $\tp_{ba} > -\epsilon_1$ and $\tp_{cb} > -\epsilon_2$, where $\epsilon_1,\epsilon_2 > 0$ and $(\epsilon_1+\epsilon_2) < \frac{1}{2}$, then $\tp_{ca} > -(\epsilon_1+\epsilon_2)$.
\end{lem}

We first bound the sample complexity of Algorithm \ref{alg:epsdel}. 
For clarity of notation, we denote the set $S$ at the beginning of iteration $\ell$ (i.e., at line 9) by $S_\ell$.
Note that at an iteration $\ell$, any set $\cG_g$ is played $t= \lceil \frac{16\thetg}{m\epsilon_l^2} \ln \frac{2k}{\delta_\ell} \rceil \le \lceil \frac{96\thetk}{m\epsilon_l^2} \ln \frac{2k}{\delta_\ell}\rceil$ times, where the inequality follows from Corollary \ref{cor:algs}. Also,  since the algorithm discards exactly $k-1$ items from each set $\cG_g$, the maximum number of iterations possible is $\lceil  \ln_k n \rceil$. Now, at iteration $\ell$, since $G = \Big\lfloor \frac{|S_\ell|}{k} \Big\rfloor < \frac{|S_\ell|}{k}$, the total sample complexity for iteration $\ell$ is at most 
$$\frac{|S_\ell|}{k}t \le \frac{n}{k^{\ell}} \left\lceil  \frac{96 \thetk}{m\epsilon_\ell^2}\ln \frac{2k}{\delta_\ell} \right\rceil $$ 

 using the fact that $|S_\ell| \le \frac{n}{k^{\ell - 1}}$ for all $\ell \in [\lfloor  \ln_k n \rfloor]$. 
 For all iterations $\ell \in [\lfloor  \ln_k n \rfloor]$ except the final one, we have $\epsilon_\ell = \frac{\epsilon}{8}\bigg( \frac{3}{4} \bigg)^{\ell-1}$ and $\delta_\ell = \frac{\delta}{2^{\ell+1}}$.
Moreover, for the last iteration $\ell = \lceil  \ln_k n \rceil$, the sample complexity is at most $\left\lceil \frac{96\thetk}{m(\epsilon/2)^2}\ln \frac{4k}{\delta} \right\rceil$ since, in this case, $\epsilon_\ell = \frac{\epsilon}{2}$, and $\delta_\ell = \frac{\delta}{2}$, and $|S| = k$.

Let us ignore, for the moment, the additional sample complexity due to the score estimation subroutine, \algs, in the operation of Algorithm \ref{alg:epsdel}. Then, the argument above implies that the sample complexity of the algorithm is at most

\begin{align*}
(A) := \sum_{\ell = 1}^{\lceil  \ln_k n \rceil}\frac{|S_\ell|}{k}t&  
 \le \sum_{\ell = 1}^{\infty}  \frac{n}{k^\ell} \left\lceil \frac{96 \thetk}{m\bigg(\frac{\epsilon}{8}\big(\frac{3}{4}\big)^{\ell-1}\bigg)^2}\ln \frac{k 2^{\ell+1}}{\delta} \right\rceil + \left\lceil \frac{96\thetk}{m(\epsilon/2)^2}\ln \frac{4k}{\delta} \right\rceil \\ 
 &  
 \le \sum_{\ell = 1}^{\infty}  \frac{n}{k^\ell} \left( \frac{96 \thetk}{m\bigg(\frac{\epsilon}{8}\big(\frac{3}{4}\big)^{\ell-1}\bigg)^2}\ln \frac{k 2^{\ell+1}}{\delta}  + 1 \right) + \left( \frac{96\thetk}{m(\epsilon/2)^2}\ln \frac{4k}{\delta} + 1\right) \\
& \le \frac{4096n\thetk}{mk\epsilon^2}\sum_{\ell = 1}^{\infty}\frac{16^{\ell-1}}{(9k)^{\ell-1}}\Big( \ln \frac{k}{\delta} + {(\ell+1)} \Big) + \frac{n}{k-1} + \frac{384\thetk}{m \epsilon^2}\ln \frac{4k}{\delta} + 1\\
& \le \frac{4096n\thetk}{mk\epsilon^2}\ln \frac{k}{\delta}\sum_{\ell = 1}^{\infty}\frac{4^{\ell-1}}{(9k)^{\ell-1}}\Big( {3\ell} \Big) + \frac{384\thetk}{m \epsilon^2}\ln \frac{4k}{\delta} + \left( 1 + \frac{n}{k-1} \right) \\
&= O\bigg(\frac{n\thetk}{mk\epsilon^2}\ln \frac{k}{\delta} + \frac{n}{k} \bigg) \quad [\text{for any } k > 1] \\
&= O\bigg(\frac{n\thetk}{mk\epsilon^2}\ln \frac{k}{\delta} + \frac{n \thetk }{k} \ln \frac{k}{\delta} \ \bigg) \quad [\text{since } \thetk \geq 1, \ln \frac{k}{\delta} \geq 1 ].
\end{align*}

Turning to the extra effort expended by the score estimation subroutine \algs$(\cG_g,b,\delta_\ell)$, at each phase $\ell$, the sample complexity of \algs\, is known by Cor. \ref{cor:algs} to be at most ${80\thetk }\ln \frac{4}{\delta_\ell} = O\Big(\thetk \ln \frac{1}{\delta_\ell}\Big)$ for any subgroup $\cG_g$. And since there are at most $G = \Big\lfloor \frac{|S_\ell|}{k} \Big\rfloor < \frac{|S_\ell|}{k}$ subgroups at any phase $\ell$, this implies that the total sample complexity incurred at any phase owing to \algs\, is at most $\frac{80 |S_\ell| \thetk}{k}\ln \frac{4}{\delta_\ell} \le \frac{80 n \thetk}{k^{\ell+1}}\ln \frac{4}{\delta_\ell}$. Following the same calculations as before, the total sample complexity incurred by the \algs\, subroutine within the algorithm, over all iterations, is at most

\begin{align*}
(B) := \sum_{\ell = 1}^{\lceil  \ln_k n \rceil}\frac{80 n \thetk}{k^{\ell+1}}\ln \frac{4}{\delta_\ell}&  
 \le \sum_{\ell = 1}^{\infty}\frac{n \thetk}{k}\frac{80 }{k^{\ell}}\ln \frac{8 2^{\ell}}{\delta} = O\bigg(\frac{n\thetk}{k}\ln \frac{1}{\delta}\bigg) = O\bigg(\frac{n\thetk}{k}\ln \frac{k}{\delta}\bigg).\end{align*}

Observe now that the term (B) is dominated by (A) in general unless $\frac{1}{m\epsilon^2}= O(1)$, or in other words $m$ is so large  that $m = \Omega \left(\frac{1}{\epsilon^2} \right)$. Thus taking care of the above tradeoff between term (A) and (B), the final sample complexity can be expressed as $O(\frac{n \thetk}{k}\max\big(1,\frac{1}{m\epsilon^2}\big) \log \frac{k}{\delta})$.
This proves the sample complexity bound for Algorithm \ref{alg:epsdel}.

We now proceed to prove the $(\epsilon,\delta)$-{PAC} correctness of Algorithm \ref{alg:epsdel}. We start by making the following observation.

\begin{lem}
\label{lem:divbat_n1} 
Consider any particular set $\cG_g$ at any iteration $\ell \in \lfloor \frac{n}{k} \rfloor$, and let $q_{i}: = \sum_{\tau = 1}^{t}\1(i \in \cG^m_{g})$ be the number of times any item $i \in \cG_g$ appears in the top-$m$ rankings when $\cG_g$ is played for $t$ rounds. If $i_g := \arg \max_{i \in \cG_g}\theta_i$ and $\theta_{i_g} > \frac{1}{2}$, then for any $\eta \in \big(\frac{3}{32\sqrt 2},1 \big]$, with probability at least $\Big(1-\frac{\delta_\ell}{2k}\Big)$, $q_{i_g} > (1-\eta)\frac{mt}{k}$.
\end{lem} 

\begin{proof}
Define $i^\tau: = \1(i \in \cG_{g}^m)$ as the indicator of the event that the $i^{th}$ element appears in the top-$m$ ranking at iteration $\tau \in [t]$.  Using the  definition of the top-$m$ ranking feedback model, we get $\E[i_g^\tau] = Pr(\{i_g \in \cG_{g}^m\}) = Pr\big( \exists j \in [m] ~|~ \sigma(j) = i_g \big) = \sum_{j = 1}^{m}Pr\big( \sigma(j) = i_g \Big) > \sum_{j = 0}^{m-1}\frac{\theta_{i_g}}{\thetg} \ge \frac{m\theta_{i_g}}{\thetg}$, as $Pr(\{i_g | S\}) = \frac{\theta_{i_g}}{\sum_{j \in S}\theta_j} \ge \frac{\theta_{i_g}}{\thetg}$ for any $S \subseteq [\cG_g]$, $i \in \cG_g$, as $i_g := \arg \max_{i \in \cG_g}\theta_i$ is the best item of set $\cG_g$. Hence $\E[q_{i_g}] = \sum_{\tau = 1}^{t}\E[i_g^\tau] \ge \frac{m t \theta_{i_g}}{\Theta_{\cG_g}} > \frac{m t }{2\Theta_{\cG_g}}$. 


Applying the Chernoff-Hoeffding concentration inequality for $w_{i_g}$, we get that for any $\eta \in (\frac{3}{32},1]$, 
\begin{align*}
Pr\Big( q_{i_g} \le (1-\eta)\E[q_{i_g}] \Big) & \le \exp\left(- \frac{\E[q_{i_g}]\eta^2}{2}\right) \le \exp\left(- \frac{mt\eta^2}{4 \Theta_{\cG_g}}\right) \\
& = \exp\bigg(- \frac{2\eta^2}{\epsilon_\ell^2} \ln \bigg( \frac{2k}{\delta_\ell} \bigg) \bigg) = \exp\bigg(- \frac{(\sqrt 2\eta)^2}{\epsilon_\ell^2} \ln \bigg( \frac{2k}{\delta_\ell} \bigg) \bigg) \\
& \le \exp\bigg(- \ln \bigg( \frac{2k}{\delta_\ell} \bigg) \bigg) \le \frac{\delta_\ell}{2k} ,
\end{align*}

where the second last inequality holds as $\eta \ge \frac{3}{32\sqrt 2}$ and $\epsilon_\ell \le \frac{3}{32}$, for any iteration $\ell \in \lceil \ln n \rceil$; in other words for any $\eta \ge \frac{3}{32\sqrt 2}$, we have $\frac{\sqrt 2\eta}{\epsilon_\ell} \ge 1$ which leads to the second last inequality. Thus, we get that 
with probability at least $\Big(1-\frac{\delta_\ell}{2k}\Big)$, it holds that  $q_{i_g} > (1-\eta)\E[q_{i_g}] \ge (1-\eta)\frac{tm}{2\Theta_{\cG_g}}$.
\end{proof}

In particular, fixing $\eta = \frac{1}{2}$ in Lemma \ref{lem:divbat_n1}, we get that with probability at least $\big(1-\frac{\delta_\ell}{2}\big)$,  $q_{i_g} > (1-\frac{1}{2})\E[w_{i_g}] > \frac{mt}{4\Theta_{\cG_g}}$. 
Note that for any round $\tau \in [t]$, whenever an item $i \in \cG_g$ appears in the top-$m$ set $\cG_{gm}^\tau$, then the rank breaking update ensures that every element in the top-$m$ set gets compared with rest of the $k-1$ elements of $\cG_g$. Based on this observation, we now prove that for any set $\cG_g$, its best item $i_g$ is retained as the winner $c_g$ with probability at least $\big( 1 - \frac{\delta_\ell}{2}\big)$. More formally, we make the following observation.

\begin{lem}
\label{lem:divbat_n2} 
Consider any particular set $\cG_g$ at any iteration $\ell \in \lfloor \frac{n}{k} \rfloor$. If $i_g = \arg \max_{i \in \cG_g}\theta_i$ and $\theta_{i_g} > \frac{1}{2}$, then the following events occur with probability at least $\Big(1-{\delta_\ell}\Big)$: (1) $\hp_{i_g j} + \frac{\epsilon_\ell}{2} \ge \frac{1}{2}$ for all $\epsilon_\ell$-optimal items in $\cG_g$, i.e., $\forall j \in \cG_g$ such that $p_{i_g j} \in \big( \frac{1}{2}, \frac{1}{2} + \epsilon_\ell \big]$, and (2) $\hp_{i_g j} - \frac{\epsilon_\ell}{2} \ge \frac{1}{2}$ for all non $\epsilon_\ell$-optimal items in $\cG_g$, i.e., $ j \in \cG_g$ such that $p_{i_g j} > \frac{1}{2} + \epsilon_\ell$.
\end{lem}

\begin{proof}
With top-$m$ ranking feedback, the crucial observation lies in the fact that at any round $\tau \in [t]$, whenever an item $i \in \cG_g$ appears in the top-$m$ set $\cG_{g}^m$, then the rank breaking update ensures that every element in the top-$m$ set gets compared with each of the rest of the $k-1$ elements of $\cG_g$: it gets defeated by every element preceding it in $\sigma \in \Sigma_{\cG_{g}^m}$, and defeats all other items in the top-$m$ set $\cG_{g}^m$.  
Therefore, defining $n_{ij} = w_{ij} + w_{ji}$ to be the number of times item $i$ and $j$ are compared after rank-breaking, $i,j \in \cG_g$. Clearly $n_{ij} = n_{ji}$, and $0 \le n_{ij} \le tk$. Moreover, from Lemma \ref{lem:divbat_n1} with $\eta = \frac{1}{2}$, we have that $n_{i_g j} \ge \frac{mt}{4\Theta_{\cG_g}}$.
Given the above arguments in place let us analyze the probability of a `bad event', i.e.:

\textbf{Case 1.} $j$ is $\epsilon_\ell$-optimal with respect to $i_g$, i.e. $p_{i_g j} \in \big(\frac{1}{2}, \frac{1}{2} + \epsilon_\ell \big]$. Then we have 

\begin{align*}
Pr\Bigg( \bigg\{ \hp_{i_g j} + \frac{\epsilon_\ell}{2} < \frac{1}{2}  \bigg\} & \cap \bigg\{ n_{i_g j} \ge \frac{mt}{4\Theta_{\cG_g}} \bigg\} \Bigg)
 = Pr\Bigg( \bigg\{ \hp_{i_g j} < \frac{1}{2} - \frac{\epsilon_\ell}{2} \bigg\} \cap \bigg\{ n_{i_g j} \ge \frac{mt}{4\Theta_{\cG_g}} \bigg\} \Bigg)\\
& \le Pr\bigg( \bigg\{ \hp_{i_g j} - p_{i_g j} <  - \frac{\epsilon_\ell}{2}\bigg\} \cap \bigg\{ n_{i_g j} \ge \frac{mt}{4\Theta_{\cG_g}} \bigg\} \bigg)\\
&  \le \exp\Big( -2\frac{mt}{4\Theta_{\cG_g}}{(\epsilon_\ell/2)}^2 \Big) \bigg) = \frac{\delta_\ell}{2k},
\end{align*}

where the first inequality follows as $p_{i_g j} > \frac{1}{2}$, and the second inequality follows from Lemma \ref{lem:pl_simulator} with $\eta = \frac{\epsilon_\ell}{2}$ and $v = \frac{mt}{4\Theta_{\cG_g}}$.

\textbf{Case 2.} $j$ is non $\epsilon_\ell$-optimal with respect to $i_g$, i.e. $p_{i_g j} > \frac{1}{2} + \epsilon_\ell$. Similar to before, we have

\begin{align*}
Pr\Bigg( \bigg\{ \hp_{i_g j} - \frac{\epsilon_\ell}{2} < \frac{1}{2}  \bigg\} & \cap \bigg\{ n_{i_g j} \ge \frac{mt}{4\Theta_{\cG_g}} \bigg\} \Bigg)
 = Pr\Bigg( \bigg\{ \hp_{i_g j} < \frac{1}{2} + \frac{\epsilon_\ell}{2} \bigg\} \cap \bigg\{ n_{i_g j} \ge \frac{mt}{4\Theta_{\cG_g}} \bigg\} \Bigg)\\
& \le Pr\bigg( \bigg\{ \hp_{i_g j} - p_{i_g j} <  - \frac{\epsilon_\ell}{2}\bigg\} \cap \bigg\{ n_{i_g j} \ge \frac{mt}{4\Theta_{\cG_g}} \bigg\} \bigg)\\
&  \le \exp\Big( -2 \frac{mt}{4\Theta_{\cG_g}} {(\epsilon_\ell/2)}^2 \Big) \bigg) = \frac{\delta_\ell}{2k},
\end{align*}

where the third last inequality follows since in this case $p_{i_g j} > \frac{1}{2} + \epsilon_\ell$, and the last inequality follows from Lemma \ref{lem:pl_simulator} with $\eta = \frac{\epsilon_\ell}{2}$ and $v = \frac{mt}{2k}$.

Let us define the event $\cE: = \bigg\{ \exists j \in \cG_g  \text{ such that } \hp_{i_g j} + \frac{\epsilon_\ell}{2} < \frac{1}{2}, p_{i_g j} \in \big( \frac{1}{2}, \frac{1}{2} + \epsilon_\ell \big] \text{ or } \hp_{i_g j} -  \frac{\epsilon_\ell}{2} < \frac{1}{2}, p_{i_g j} > \frac{1}{2} + \epsilon_\ell  \bigg\}$. Then by combining Case $1$ and $2$, we get
\begin{align*}
 Pr\Big( \cE \Big) & = Pr\Bigg( \cE \cap \bigg\{ n_{i_g j} \ge \frac{mt}{4\Theta_{\cG_g}} \bigg\} \Bigg) + Pr\Bigg( \cE \cap \bigg\{ n_{i_g j} < \frac{mt}{4\Theta_{\cG_g}} \bigg\} \Bigg)\\
& \le \sum_{j \in \cG_g \text{ s.t. } p_{i_g j} \in \big(\frac{1}{2}, \frac{1}{2} + \epsilon_\ell \big]}Pr\Bigg( \bigg\{ \hp_{i_g j} + \frac{\epsilon_\ell}{2} < \frac{1}{2} \bigg\} \cap \bigg\{ n_{i_g j} \ge \frac{mt}{4\Theta_{\cG_g}} \bigg\} \Bigg) \\
& + \sum_{j \in \cG_g \text{ s.t. } p_{i_g j} > \frac{1}{2} + \epsilon_\ell}Pr\Bigg( \bigg\{ \hp_{i_g j} - \frac{\epsilon_\ell}{2} < \frac{1}{2} \bigg\} \cap \bigg\{ n_{i_g j} \ge \frac{mt}{4\Theta_{\cG_g}} \bigg\} \Bigg) + Pr\Bigg( \bigg\{ n_{i_g j} < \frac{mt}{4\Theta_{\cG_g}} \bigg\} \Bigg)\\
& \le \frac{(k-1)\delta_\ell}{2k} + \frac{\delta_\ell}{2k} \le {\delta_\ell}
\end{align*}

where the last inequality follows from the above two case analyses and Lemma \ref{lem:divbat_n1}.
\end{proof}



Given Lemma \ref{lem:divbat_n2} in place, let us now analyze with what probability the algorithm can select a non $\epsilon_\ell$-optimal item $j \in \cG_g$ as $c_g$ at any iteration $\ell \in \lceil \frac{n}{k} \rceil$. For any set $\cG_g$ (or set $S$ for the last iteration $\ell = \lceil \frac{n}{k} \rceil$), we define the set of non-$\epsilon_\ell$-optimal elements $\cO_g = \{ j \in \cG_g \mid p_{i_g j} > \frac{1}{2} + \epsilon_\ell \}$, and recall the event $\cE: = \bigg\{ \exists j \in \cG_g  \text{ such that } \hp_{i_g j} + \frac{\epsilon_\ell}{2} < \frac{1}{2}, p_{i_g j} \in \big( \frac{1}{2}, \frac{1}{2} + \epsilon_\ell \big] \text{ or } \hp_{i_g j} -  \frac{\epsilon_\ell}{2} < \frac{1}{2}, p_{i_g j} > \frac{1}{2} + \epsilon_\ell  \bigg\}$. We then have 

\begin{align}
\label{eq:divbat_1}
\nonumber Pr(c_g \in \cO_g) & \le Pr\Bigg( \bigg\{ \exists j \in \cG_g, \hp_{i_g j} + \frac{\epsilon_\ell}{2} < \frac{1}{2} \bigg\} \cup \bigg\{ \exists j \in \cO_g, \hp_{j i_g} + \frac{\epsilon_\ell}{2} \ge \frac{1}{2} \bigg\} \Bigg) \\
\nonumber & \le Pr\Bigg( \cE \cup \bigg\{ \exists j \in \cO_g, \hp_{j i_g} + \frac{\epsilon_\ell}{2} \ge \frac{1}{2} \bigg\} \Bigg) \\
\nonumber & = Pr \Big( \cE \Big) + Pr\Bigg( \bigg\{ \exists j \in \cO_g, \hp_{j i_g} + \frac{\epsilon_\ell}{2} \ge \frac{1}{2} \bigg\} \cap \cE^{c} \Bigg)\\ 
& = Pr \Big( \cE \Big) + Pr\Bigg( \bigg\{ \exists j \in \cO_g, \hp_{j i_g} + \frac{\epsilon_\ell}{2} \ge \frac{1}{2} \bigg\} \cap \cE^{c} \Bigg) \le \delta_\ell + 0 = \delta_\ell,
\end{align}

where the last inequality follows from Lemma \ref{lem:divbat_n2} and the fact that $\hp_{i_g j} - \frac{\epsilon_\ell}{2} \ge \frac{1}{2} \implies \hp_{j i_g} + \frac{\epsilon_\ell}{2} < \frac{1}{2}$. The proof now follows by combining all the above parts together.

At each iteration $\ell$, let us define $g_\ell \in [G]$ to be the index of the set that contains the \textit{best item} of the currently surviving set $S$, i.e., the index $g_\ell$ such that $\arg\max_{i \in S}\theta_i \in \cG_{g_\ell}$. Then from \eqref{eq:divbat_1}, with probability at least $(1-\delta_\ell)$,\, $\tp_{c_{g_\ell}i_{g_\ell}} > -\epsilon_\ell$. Now for each iteration $\ell$, recursively applying \eqref{eq:divbat_1} and Lemma \ref{lem:pl_sst} to $\cG_{g_\ell}$, we get that $\tp_{r_*1} > -\Big( \frac{\epsilon}{8} + \frac{\epsilon}{8}\Big(\frac{3}{4}\Big) + \cdots + \frac{\epsilon}{8}\big(\frac{3}{4}\big)^{\lfloor \frac{n}{k} \rfloor} \Big) + \frac{\epsilon}{2} \ge -\frac{\epsilon}{8}\Big( \sum_{i = 0}^{\infty}\big( \frac{3}{4} \big)^{i} \Big) + \frac{\epsilon}{2} = \epsilon$.
(Note that for this analysis to go through, it is in fact sufficient to consider only the set of iterations $\{\ell \ge \ell_0 \mid \ell_0 = \min\{l \mid 1 \notin \cR_l, \, l \ge 1\} \}$, because prior to considering item $1$, it does not matter even if the algorithm makes a mistake in any of the iterations $\ell < \ell_0$). Thus assuming that the algorithm does not fail in any of the iterations $\ell$, we have that $p_{r_*1} > \frac{1}{2} - \epsilon$.

Finally, since at each iteration $\ell$, the algorithm fails with probability at most $\delta_\ell$, the total failure probability of the algorithm is at most $\Big( \frac{\delta}{4} + \frac{\delta}{8} + \cdots + \frac{\delta}{2^{\lceil \frac{n}{k}  \rceil}} \Big) + \frac{\delta}{2} \le \delta$.
This concludes the correctness of the algorithm showing that it indeed returns an $\epsilon$-best element $r_*$ such that $p_{r_*1} \ge \frac{1}{2} - \epsilon$ with probability at least $1-\delta$. 
\end{proof}

\subsection{Proof of Thm. \ref{thm:sc_fewf}}
\label{app:prf_alg_fewf}
\scalgfewf*

\begin{proof}

The proof is based on the following four main observations:

\begin{enumerate}
\item The \bi\, $a^*$ (i.e. item $1$ in our case) is likely to beat the ($\epsilon_s,\delta_s$)-PAC item by sufficiently high margin, for any sub-phase $s= 1,2\ldots$, and hence is never discarded (see Lem. \ref{lem:bistays}).
\item With high probability the set of suboptimal items get discarded at a fixed rate once played for sufficiently long duration (see Lem. \ref{lem:nbigoes}).
\item The number of occurrences of any sub-optimal item $i$ before it gets discarded is proportional to $O\Big(\frac{1}{\Delta_i^2}\Big)$ which yields the desired sample complexity of the algorithm (see Lem.\ref{lem:sc_item}).
\item Lastly we show (using Thm. \ref{lem:up_algep} and Lem. \ref{lem:algs})) that the additional sample complexity incurred due to invoking the subroutine \algep\, and \algs\, at every sub-phase $s$ is orderwise same as the sample complexity incurred by \algfewf\, in the rest of the sub-phase, due to which \algep\, so they do  not actually contribute to the overall sample complexity of the algorithm modulo some constant factors.
\end{enumerate}
While analysing any particular batch $\cB_b$ of a given phase $s$, we will denote by $S = \cB_b \sm \{b_s\}$ and by $\thets = \sum_{i \in S}\theta_i$.
We will first prove the correctness of the algorithm, i.e. with high probability $(1-\delta)$, \algfewf\, indeed returns the \bi\,, i.e. item $1$ in our case. We prove this using the following two lemmas: Lem. \ref{lem:bistays} and Lem. \ref{lem:nbigoes} respectively.

\begin{restatable}[]{lem}{bistays}
\label{lem:bistays}
With high probability of at least $(1-\frac{\delta}{20})$, item $1$ is never eliminated, i.e. $1 \in \cA_s$ for all sub-phase $s$. More formally, at the end of any sub-phase $s= 1,2,\ldots$, $\hp_{1b_s} > \frac{1}{2} - \epsilon_s$.
\end{restatable}

\begin{proof}
Firstly note that at any sub-phase $s$, each batch $b \in B$ within that phase is played for $t_s = \frac{2\thetk}{\epsilon_s^2}\ln \frac{k}{\delta_s}$ rounds. Now consider the batch $\cB \owns 1$ at any phase $s$. Clearly $b_s \in \cB$ too. Again since $b_s$ is returned by Alg. \ref{alg:epsdel}, by Thm. \ref{lem:up_algep} we know that with probability at least $(1-\delta_s)$, $p_{b_s1} > \frac{1}{2} - \epsilon_s \, \implies \theta_{b_s} > \theta_1 - 4\epsilon$. This further implies $\theta_{b_s} \ge \theta_1 - \frac{1}{2} = \frac{1}{2}$ (since we assume $\theta_1 = 1$, and at any $s,\, \epsilon_s < \frac{1}{8}$). Moreover by Lem. \ref{lem:algs}, we have $\thetsh \ge \frac{\theta_{b_s} +\thets }{\theta_{b_s}} > \frac{\thets + 1}{2}$ (recall we denote $S = \cB_b \sm \{b_s\}$)

Now let us define $w_i$ as number of times item $i \in \cB$ was returned as the winner in $t_s$ rounds and $i_\tau$ be the winner retuned by the environment upon playing $\cB$ for the $\tau^{th}$ round, where $\tau \in [t_s]$.
Then clearly $Pr(\{i_\tau = 1\}) = \frac{\theta_{1}}{\sum_{j \in \cB}\theta_j} = \frac{1}{\theta_{b_s} + \thets} \ge \frac{1}{1+ \thets}, \, \forall \tau \in [t_s]$, as $1 := \arg \max_{i \in \cB}\theta_i$. Hence $\E[w_{1}] = \sum_{\tau = 1}^{t_s}\E[\1(i_\tau = 1)] = \frac{t_s}{\theta_{b_s} + \thets} \ge \frac{t_s}{(1+\thets)}$. 
Now assuming $b_s$ to be indeed an $(\epsilon_s,\delta_s)$-PAC \bi\, and the bound of Lem. \ref{lem:algs} to hold good as well, applying the multiplicative form of the  Chernoff-Hoeffding bound on the random variable $w_{1}$, we get that for any $\eta \in (\sqrt{2}\epsilon_s,1]$, 

\begin{align*}
Pr\Big( w_{1} \le (1-\eta)\E[w_{1}] \Big) & \le \exp\bigg(- \frac{\E[w_{1}]\eta^2}{2}\bigg) \le \exp\bigg(- \frac{t_s\eta^2}{2(1+\thets)}\bigg) \\
& = \exp\bigg(- \frac{2\thetsh\eta^2}{2\epsilon_s^2(1+\thets)}\ln \frac{k}{\delta_s}\bigg)\\
&\stackrel{(a)}{\le} \exp\bigg(- \frac{2(\thets+1)\eta^2}{4\epsilon_s^2(1+\thets)}\ln \frac{k}{\delta_s}\bigg)\\
& \le \exp\bigg(- \frac{\eta^2}{2\epsilon_s^2} \ln \bigg( \frac{k}{\delta_s} \bigg) \bigg) \le \exp\bigg(- \ln \bigg( \frac{k}{\delta_s} \bigg) \bigg) = \frac{\delta_s}{k},
\end{align*}
where $(a)$ holds since we proved $\thetsh \ge \frac{\theta_{b_s} +\thets }{\theta_{b_s}} > \frac{\thets + 1}{2}$, and the last inequality holds as $\eta > \epsilon_s\sqrt{2}$.

In particular, note that $\epsilon_s \le \frac{1}{8}$ for any sub-phase $s$, due to which we can safely choose $\eta = \frac{1}{2}$ for any $s$, which gives that with probability at least $\big(1-\frac{\delta_s}{k}\big)$,  $w_{1} > (1-\frac{1}{2})\E[w_{1}] > \frac{t_s}{2(\theta_{b_s}+\thets)}$, for any subphase $s$.


Thus above implies that with probability atleast $(1-\frac{\delta_s}{k})$, after $t_s$ rounds we have $w_{1b_s} \ge \frac{t_s}{2(\theta_{b_s}+\thets)} \implies w_{1b_s} + w_{b_s1} \ge \frac{t_s}{2(\theta_{b_s}+\thets)}$. Let us denote  $n_{1b_s} = w_{1b_s} + w_{b_s1}$. Then the probability of the event:

\begin{align*}
Pr\Bigg(  \hp_{1 b_s} & < \frac{1}{2} - \epsilon_s  ,  n_{1 b_s} \ge \frac{t_s}{2(\theta_{b_s}+\thets)}  \Bigg)
 = Pr\Bigg( \hp_{1 b_s} - \frac{1}{2} < - \epsilon_s  ,  n_{1 b_s} \ge \frac{t_s}{2(\theta_{b_s}+\thets)} \Bigg)\\
& \le Pr\Bigg( \hp_{1 b_s} - p_{1 b_s} < - \epsilon_s  ,  n_{1 b_s} \ge \frac{t_s}{2(\theta_{b_s}+\thets)} \Bigg) ~~~(\text{as } \p_{1 b_s} > \frac{1}{2})\\
& \stackrel{(a)} \le \exp\Big( -2\frac{t_s}{2(\theta_{b_s}+\thets)}\big({\epsilon_s}\big)^2 \Big) \\
& \le \exp\Big( -\frac{2(\theta_{b_s} + \thets)}{\epsilon_s^2\theta_{b_s}(\theta_{b_s}+\thets)}\big({\epsilon_s}\big)^2 \Big) 
\le \frac{\delta_s}{k},
\end{align*}

where the last inequality $(a)$ follows from Lem. \ref{lem:pl_simulator} for $\eta = \epsilon_s$ and $v = \frac{t_s}{2k}$.

Thus under the two assumptions that (1). $b_s$ is indeed an $(\epsilon_s,\delta_s)$-PAC \bi\, and (2). the bound of Lem. \ref{lem:algs} holds good, combining the above two claims, at any sub-phase $s$, we have

\begin{align*}
& Pr\Bigg(  \hp_{1 b_s} < \frac{1}{2} - \epsilon_s\Bigg)\\
 & = 
Pr\Bigg(  \hp_{1 b_s} < \frac{1}{2} - \epsilon_s  ,  n_{1 b_s} \ge \frac{t_s}{2(\theta_{b_s}+\thets)}  \Bigg) + Pr\Bigg(  \hp_{1 b_s} < \frac{1}{2} - \epsilon_s  ,  n_{1 b_s} < \frac{t_s}{2(\theta_{b_s}+\thets)}  \Bigg)\\
& \le \frac{\delta_s}{k} + Pr\Bigg( n_{1 b_s} < \frac{t_s}{2(\theta_{b_s}+\thets)}  \Bigg) \le \frac{2\delta_s}{k} \le \delta_s ~~~(\text{since } k \ge 2)\\
\end{align*}

Moreover from Thm. \ref{lem:up_algep} and Lem. \ref{lem:algs} we know that the above two assumptions hold with probability at least $(1-2\delta_s)$. Then taking union bound over all sub-phases $s = 1,2, \ldots$, the probability that item $1$ gets eliminated at any round:

\begin{align*}
Pr\Bigg( \exists s = 1,2, \ldots \text{ s.t. } \hp_{1 b_s} < \frac{1}{2} - \epsilon_s \Bigg) & = \sum_{s = 1}^{\infty}{3\delta_s} = \sum_{s = 1}^{\infty}\frac{3\delta}{120s^3} \le \frac{\delta}{40}\sum_{s = 1}^{\infty}\frac{1}{s^2} \le \frac{\delta}{40}\frac{\pi^2}{6} \le \frac{\delta}{20},
\end{align*}
where the first inequality holds since $k \ge 2$.
\end{proof}

We next introduce few notations before proceeding to the next claim of Lem. \ref{lem:nbigoes}.

\textbf{Notations.}
Recall that we defined $\Delta_i = \theta_1-\theta_i$ (Sec. \ref{sec:prb_setup}). We further denote $\Delta_{\min} = \min_{i \in [n]\sm\{1\}}\Delta_i$.
We define the set of arms $[n]_r := \{ i \in [n] : \frac{1}{2^r} \le \Delta_i < \frac{1}{2^{r-1}} \}$, and denote the set of surviving arms in $[n]_r$ at $s^{th}$ sub-phase by $\cA_{r,s}$, i.e. $\cA_{r,s} = [n]_r \cap \cA_s$, for all $s = 1,2, \ldots$.

\begin{restatable}[]{lem}{nbigoes}
\label{lem:nbigoes}
Assuming that the best arm $1$ is not eliminated at any sub-phase $s = 1,2,\ldots$, then with probability at least $(1-\frac{19\delta}{20})$, for any sub-phase $s \ge r$, $|\cA_{r,s}| \le \frac{2}{19}|\cA_{r,s-1}|$, for any $r = 0,1,2,\ldots \log_2 (\Delta_{\min})$. 
\end{restatable}

\begin{proof}
Consider any sub-phase $s$, and let us start by noting some properties of $b_s$. Note that by Lem. \ref{lem:up_algep}, with high probability $(1-\delta_s)$, $p_{b_s 1} > \frac{1}{2} - \epsilon_s$. Then this further implies

\begin{align*}
p_{b_s 1} > \frac{1}{2} - \epsilon_s & \implies \frac{(\theta_{b_s} - \theta_1)}{2(\theta_{b_s} + \theta_1)} > -\epsilon_s\\
& \implies (\theta_{b_s} - \theta_1) > -2\epsilon_s(\theta_{b_s} + \theta_1) > -4\epsilon_s ~~(\text{as} ~\theta_i \in (0,1)\, \forall i \in [n]\sm\{1\}) 
\end{align*}  

So we have with probability atleast $(1-\delta_s)$,  $\theta_{b_s} > \theta_1 - 4\epsilon_s = \theta_1 - \frac{1}{2^s}$. 

Now consider any fixed $r = 0,1,2,\ldots \log_2(\Delta_{\min})$. Clearly by definition, for any item $i \in [n]_r$, $\Delta_i = \theta_1 - \theta_i > \frac{1}{2^r}$.

Then combining the above two claims, we have for any sub-phase $s \ge r$, $\theta_{b_s} > \theta_1 - \frac{1}{2^s} \ge \theta_i + \frac{1}{2^r} - \frac{1}{2^s} > 0 \implies p_{b_si} > \frac{1}{2}$, at any $s \ge r$.

Moreover note that for any $s \ge 1$, $\epsilon_s > \frac{1}{8}$, so that implies $\theta_{b_s} > \theta_1 - 4\epsilon_s > \frac{1}{2}$.


Recall that at any sub-phase $s$, each batch within that phase is played for $t_s = \frac{2\thetk}{\epsilon_s^2}\ln \frac{k}{\delta_s}$ many rounds. Now consider any batch such that $\cB \owns i$ for any $i \in [n]_r$. Of course $b_s \in \cB$ as well, and note that we have shown $p_{b_si} > \frac{1}{2}$ with high probability $(1-\delta_s)$. 

Same as Lem. \ref{lem:bistays}, let us again define $w_i$ as number of times item $i \in \cB$ was returned as the winner in $t_s$ rounds, and $i_\tau$ be the winner retuned by the environment upon playing $\cB$ for the $\tau^{th}$ rounds, where $\tau \in [t_s]$.
Then given $\theta_{b_s} > \frac{1}{2}$ (as derived earlier), clearly $Pr(\{i_\tau = b_s\}) = \frac{\theta_{b_s}}{\theta_{b_s} + \thets}, \, \forall \tau \in [t_s]$. Hence $\E[w_{b_s}] = \sum_{\tau = 1}^{t_s}\E[\1(i_\tau = b_s)] = \frac{\theta_{b_s}t_s}{(\theta_{b_s} + \thets)}$. 
Now applying multiplicative Chernoff-Hoeffdings bound on the random variable $w_{b_s}$, we get that for any $\eta \in (\sqrt 2\epsilon_s,1]$, 

\begin{align*}
Pr\Big( w_{b_s} & \le (1-\eta)\E[w_{b_s}] \mid \theta_{b_s} > \frac{1}{2}, \thetsh > \frac{(\theta_{b_s} + \thets)}{\theta_{b_s}} \Big)\\
 & \le \exp\bigg(- \frac{\E[w_{b_s}]\eta^2}{2}\bigg) \le \exp\bigg(- \frac{\theta_{b_s}t_s\eta^2}{2(\theta_{b_s} + \thets)}\bigg) \\
& \le \exp\bigg(- \frac{\eta^2}{2\epsilon_s^2} \ln \bigg( \frac{k}{\delta_s} \bigg) \bigg) \le \exp\bigg(- \ln \bigg( \frac{k}{\delta_s} \bigg) \bigg) = \frac{\delta_s}{k},
\end{align*}
where the last inequality holds as $\eta > \sqrt 2\epsilon_s$.
So as a whole, for any $\eta \in (\sqrt 2\epsilon_s,1]$,

\begin{align*}
& Pr\Big( w_{b_s} \le (1-\eta)\E[w_{b_s}] \Big)\\
& \le Pr\Big( w_{b_s} \le (1-\eta)\E[w_{b_s}] \mid \theta_{b_s} > \frac{1}{2}, \thetsh > \frac{(\theta_{b_s} + \thets)}{\theta_{b_s}}  \Big)Pr\Big(\theta_{b_s} > \frac{1}{2}\Big) + Pr\Big(\theta_{b_s} < \frac{1}{2},\thetsh > \frac{(\theta_{b_s} + \thets)}{\theta_{b_s}} \Big)\\ 
& \le \frac{\delta_s}{k} + 2\delta_s
\end{align*}

In particular, note that $\epsilon_s < \frac{1}{8}$ for any  sub-phase $s$, due to which we can safely choose $\eta = \frac{1}{2}$ for any $s$, which gives that with probability at least $\big(1-\frac{\delta_s}{k}\big)$,  $w_{b_s} > (1-\frac{1}{2})\E[w_{b_s}] > \frac{t_s \theta_{b_s}}{2 (\thets + \theta_{b_s})}$, for any subphase $s$.

Thus above implies that with probability at least $(1-\frac{\delta_s}{k}-2\delta_s)$, after $t_s$ rounds we have $w_{b_si} \ge \frac{t_s \theta_{b_s}}{2 (\thets + \theta_{b_s})} \implies w_{ib_s} + w_{b_si} \ge \frac{t_s \theta_{b_s}}{2 (\thets + \theta_{b_s})}$. Let us denote  $n_{ib_s} = w_{ib_s} + w_{b_si}$. Then the probability that item $i$ is not eliminated at any sub-phase $s \ge r$ is:

\begin{align*}
Pr\Bigg(  \hp_{i b_s} > \frac{1}{2} & - \epsilon_s  ,  n_{i b_s} \ge \frac{t_s \theta_{b_s}}{2 (\thets + \theta_{b_s})}  \Bigg)
 = Pr\Bigg( \hp_{i b_s} - \frac{1}{2} > - \epsilon_s  ,  n_{i b_s} \ge \frac{t_s \theta_{b_s}}{2 (\thets + \theta_{b_s})} \Bigg)\\
& \le Pr\Bigg( \hp_{i b_s} - p_{i b_s} > - \epsilon_s  ,  n_{1 b_s} \ge \frac{t_s \theta_{b_s}}{2 (\thets + \theta_{b_s})}\Bigg) ~~~(\text{as } \p_{i b_s } < \frac{1}{2})\\
& \stackrel{(a)}\le \exp\Big( -2\frac{t_s}{2(\thets + \theta_{b_s})}\big({\epsilon_s}\big)^2 \Big)\\
& \le \exp\Big( -2\frac{(\thets + \theta_{b_s})}{2\theta_{b_s}(\thets + \theta_{b_s})}\big({\epsilon_s}\big)^2 \ln \frac{\delta_s}{k}\Big) \le \frac{\delta_s}{k},
\end{align*}

where $(a)$ follows from Lem. \ref{lem:pl_simulator} for $\eta = \epsilon_s$ and $v = \frac{t_s}{2k}$.

Now combining the above two claims, at any sub-phase $s$, we have:

\begin{align*}
Pr\Bigg(  \hp_{i b_s} > \frac{1}{2} - \epsilon_s\Bigg) & = 
Pr\Bigg(  \hp_{i b_s} > \frac{1}{2} - \epsilon_s  ,  n_{i b_s} \ge \frac{t_s}{4k}  \Bigg) + Pr\Bigg(  \hp_{i b_s} > \frac{1}{2} - \epsilon_s  ,  n_{i b_s} < \frac{t_s}{4k}  \Bigg)\\
& \le \frac{\delta_s}{k} + 2\delta_s + \frac{\delta_s}{k}  \le 3\delta_s ~~~~ (\text{since } k \ge 2)
\end{align*}

This consequently implies that for any sup-phase $s \ge r$, $\E[|\cA_{r,s}|] \le {3\delta_s}\E[|\cA_{r,s-1}|]$.
Then applying Markov's Inequality we get:

\begin{align*}
Pr\Bigg( |\cA_{r,s}| \le \frac{2}{19}|\cA_{r,s-1}| \Bigg)  \le \frac{3\delta_s|\cA_{r,s-1}|}{\frac{2}{19}|\cA_{r,s-1}|} = \frac{57\delta_s}{2}
\end{align*}

Finally applying union bound over all sub-phases $s= 1,2,\ldots$, and all $r = 1,2, \ldots s$, we get: 

\begin{align*}
\sum_{s = 1}^{\infty}\sum_{r = 1}^{s}\frac{57\delta_s}{2} = \sum_{s = 1}^{\infty}s\frac{57\delta}{240s^3} = \frac{57\delta}{240}\sum_{s = 1}^{\infty}\frac{1}{s^2} \le \frac{57\delta}{240}\frac{\pi^2}{6} \le \frac{57\delta}{120} \le \frac{19\delta}{20}.
\end{align*}

\end{proof}


Thus combining Lem. \ref{lem:bistays} and \ref{lem:nbigoes}, we get that the total failure probability of \algfewf\, is at most $\frac{\delta}{20} + \frac{19\delta}{20} = \delta $.

The remaining thing is to prove the sample complexity bound which crucially relies on the following claim. At any sub-phase $s$, we call the item $b_s$ as the pivot element of phase $s$.

\begin{restatable}[]{lem}{scitem}
\label{lem:sc_item}
Assume both Lem. \ref{lem:bistays} and Lem. \ref{lem:nbigoes} holds good and the algorithm does not do a mistake. Consider any item $i \in [n]_r$, for any $r = 1,2, \ldots \log_2(\Delta_{\min})$. Then the total number of times item $i$ gets played as a non-pivot item (i.e. appear in at most one of the $k$-subsets per sub-phase $s$) during the entire run of Alg. \ref{alg:fewf} is $O\Big( (2^r)^2\thetk\ln \frac{rk}{\delta}  \Big)$. 
\end{restatable}

\begin{proof}
Let us denote the sample complexity of item $i$ (as a non-pivot element) from phase $x$ to $y$ as $\cN^{(i)}_{x,y}$, for any $1 \le x < y < \infty$.
Additionally, recalling from Lem. \ref{lem:algs} that $\thetsh \le 7\thetk$, we now prove the claim with the following two case analyses:

\textbf{(Case 1)} Sample complexity till sub-phase $s = r-1$: Note that in the worst case item $i$ can get picked at every sub-phase $s = 1,2,\ldots$ $r-1$, and at every $s$ it is played for $t_s$ round. Additionally, recalling from Lem. \ref{lem:algs} that $\thetsh \le 7\thetk$, the total number of plays of item $i \in [n]_r$ (as a non-pivot item), till sub-phase $r-1$ becomes:

\[
\cN^{(i)}_{1,r-1} \le \sum_{s = 1}^{r-1}t_s = \sum_{s = 1}^{r-1}\frac{14\thetk}{\epsilon_s^2}\ln \frac{k}{\delta_s} = \frac{14\thetk}{4^{-2}}\sum_{s = 1}^{r-1}(2^s)^2\ln \frac{120k^3}{\delta} = O\Big( (2^r)^2\thetk\ln \frac{rk}{\delta}  \Big)
\]

\textbf{(Case 2)} Sample complexity from sub-phase $s \ge r$ onwards: Assuming Lem. \ref{lem:nbigoes} holds good, note that if we define a random variable $I_s$ for any sub-phase $s \ge r$ such that $I_s = \1(i \in \cA_s)$, then clearly $\E[I_{s}] \le \frac{2}{19}\E[I_{s-1}]$ (as follows from the analysis of Lem. \ref{lem:bistays}). Then the total expected sample complexity of item $i \in [n]_r$ for round $r, r+1, \ldots \infty$ becomes:
\begin{align*}
& \cN^{(i)}_{r,\infty}\le 224\thetk\sum_{s = r}^{\infty}\bigg(\frac{2}{19}\bigg)^{s-r+1}4^s\ln \frac{k}{\delta_s} = 224\thetk(2^r)^2\sum_{s = 0}^{\infty}\bigg(\frac{2}{19}\bigg)^{s+1}(2^s)^2\ln \frac{120k(s+r)^3}{\delta}\\
& = \frac{448}{19}\thetk(2^r)^2\sum_{s = 0}^{\infty}\bigg(\frac{8}{19}\bigg)^{s}\ln \frac{120k(s+r)^3}{\delta} \\
& \le \frac{448}{19}\thetk(2^r)^2\Bigg[\ln \frac{120kr}{\delta}\sum_{s = 0}^{\infty}\bigg(\frac{8}{19}\bigg)^{s} + \sum_{s = 0}^{\infty}\bigg(\frac{8}{19}\bigg)^{s}\ln(120ks)\Bigg] =  O\Big( (2^r)^2\thetk\ln \frac{rk}{\delta}  \Big)
\end{align*}

Combining the two cases above we get $\cN^{(i)}_{1,\infty} = O\Big( (2^r)^2\thetk\ln \frac{rk}{\delta}\Big)$ as well, which concludes the proof.
\end{proof}

Following similar notations as $\cN^{(i)}_{x,y}$, we now denote the number of times any $k$-subset $S \subseteq [n]$ played by the algorithm in sub-phase $x$ to $y$ as $\cN^{(S)}_{x,y}$. Then using Lem. \ref{lem:sc_item}, the total sample complexity of the algorithm \algfewf\, (lets call it algorithm $\cA$) can be written as:

\begin{align}
\label{eq:sc_fewf}
\nonumber \cN_\cA(0,\delta) & = \sum_{S \subset [n]\mid |S|=k} \sum_{s=1}^{\infty}\1(S \in \{\cB_1,\ldots, \cB_{B_s}\})t_s = \sum_{i \in [n]} \sum_{s=1}^{\infty}\frac{\1(i \in \cA_s\sm\{b_s\})}{k-1}t_s\\
\nonumber & = \sum_{s=1}^{\infty} \sum_{i \in [n]}\frac{\1(i \in \cA_s\sm\{b_s\})}{k-1}t_s = \sum_{s=1}^{\infty} \sum_{r=1}^{\log_2(\Delta_{\min})}\sum_{i \in [n]_r}\frac{\1(i \in \cA_s\sm\{b_s\})}{k-1}t_s \\
\nonumber & = \frac{1}{k-1}\sum_{r=1}^{\log_2(\Delta_{\min})}\sum_{i \in [n]_r}\cN^{(i)}_{1,\infty} ~~(\text{ Lem. \ref{lem:sc_item}})\\
\nonumber & = \frac{1}{k-1}\sum_{r=1}^{\log_2(\Delta_{\min})}\sum_{i \in [n]_r}O\Big( (2^r)^2\thetk\ln \frac{rk}{\delta}\Big) \\
& = \frac{\thetk}{k-1}\sum_{r=1}^{\log_2(\Delta_{\min})}|[n]_r|O\Big( (2^r)^2\ln \frac{rk}{\delta}\Big) = O\Big(\frac{\thetk}{k} \sum_{i=2}^{n}\frac{1}{\Delta_i^2}\ln \big(\frac{k}{\delta} \ln \frac{1}{\Delta_i} \big) \Big),
\end{align}
where the last inequality follows since $2^r < \frac{2}{\Delta_i}$ by definition for all $i \in [n]_r$.
Finally the last thing to account for is the additional sample complexity incurred due to calling the subroutine \algep\ and \algs\, at every sub-phase $s$, which is combinedly known to be of $O\bigg(\frac{|\cA_s|\Theta_[k]}{k} \Big( 1,\frac{1}{(2^s)^2}\Big)\ln \frac{k}{\delta} \bigg)$ at any sub-phase $s$ (from Thm. \ref{lem:up_algep} and Cor. \ref{cor:algs}). And using a similar summation as shown above over all $s = 1,2,\ldots \infty$, combined with Lem. \ref{lem:nbigoes} and using the fact that $2^r < \frac{2}{\Delta_i}$, one can show that the total sample complexity incurred due to the above subroutines is at most $\frac{\Theta_{[k]}}{k}\sum_{i = 2}^{n}\max\big(1,\frac{1}{\Delta_i^2}\big) \log \frac{k}{\delta})$. Considering the above sample complexity added with that derived in Eqn. \ref{eq:sc_fewf} finally gives the desired $O\bigg(\frac{\thetk}{k}\sum_{i = 2}^n\max\Big(1,\frac{1}{\Delta_i^2}\Big) \ln\frac{k}{\delta}\Big(\ln \frac{1}{\Delta_i}\Big)\bigg)$ sample complexity bound of Alg. \ref{alg:fewf}.


\end{proof}




\begin{restatable}[Deviations of pairwise win-probability estimates for the PL model \cite{SGwin18}]
{lem}{plsimulator}
\label{lem:pl_simulator}
Consider a Plackett-Luce choice model with parameters $\btheta = (\theta_1,\theta_2, \ldots, \theta_n)$, and fix two distinct items $i,j \in [n]$. Let $S_1, \ldots, S_T$ be a sequence of (possibly random) subsets of $[n]$ of size at least $2$, where $T$ is a positive integer, and $i_1, \ldots, i_T$ a sequence of random items with each $i_t \in S_t$, $1 \leq t \leq T$, such that for each $1 \leq t \leq T$, (a) $S_t$ depends only on $S_1, \ldots, S_{t-1}$, and (b) $i_t$ is distributed as the Plackett-Luce winner of the subset $S_t$, given $S_1, i_1, \ldots, S_{t-1}, i_{t-1}$ and $S_t$, and (c) $\forall t: \{i,j\} \subseteq S_t$ with probability $1$. Let $n_i(T) = \sum_{t=1}^T \1(i_t = i)$ and $n_{ij}(T) = \sum_{t=1}^T \1(\{i_t \in \{i,j\}\})$. Then, for any positive integer $v$, and $\eta \in (0,1)$,
\begin{align*}
\hspace{-20pt} Pr \left( \frac{n_i(T)}{n_{ij}(T)}  - \frac{\theta_i}{\theta_i + \theta_j} \ge \eta, \; n_{ij}(T) \geq v \right) \vee Pr\left( \frac{n_i(T)}{n_{ij}(T)} - \frac{\theta_i}{\theta_i + \theta_j} \le -\eta, \; n_{ij}(T) \geq v \right) \leq e^{-2v\eta^2}. 
\end{align*}
\end{restatable}


\subsection{Modified version of \algfewf\, (Alg. \ref{alg:fewf}) for general $(\epsilon,\delta)$-PAC guarantee (for any $\epsilon \in [0,1]$)}
\label{app:alg_fewfep}

\begin{center}
\begin{algorithm}[H]
   \caption{\textbf{Modified \algfewf ~(for a general $(\epsilon,\delta)$-PAC guarantee) }}
   \label{alg:fewfep}
\begin{algorithmic}[1]
   \STATE {\bfseries input:} Set of items: $[n]$, Subset size: $n \geq k > 1$, Confidence term $\delta > 0$
   \STATE {\bfseries init:}  $\cA_0 \leftarrow [n]$, $s \leftarrow 1$ 
   \WHILE {$|\cA_{s-1}| > 1$} 
   	\STATE Set $\epsilon_s = \frac{1}{2^{s+2}}$, $\delta_s = \frac{\delta}{120s^3}$, and $\cR_s \leftarrow \emptyset$.
   \STATE $b_s \leftarrow $ \algep$(\cA_{s-1},k,1,\epsilon_s/4,\delta_s)$ 
   \STATE \textbf{If} $(\epsilon_s \le \epsilon)$ \textbf{then} $\cA' \leftarrow \{b_s\}$, and exit the \textbf{while} loop (go to Line $35$)
   \STATE $\cB_1,\ldots \cB_{B_s} \leftarrow $ \algdiv$(\cA_{s-1}\sm \{b_s\},k-1)$ 
   \STATE \textbf{if} $|\cB_{B_s}| < k-1$, \textbf{then} $\cR_s \leftarrow \cB_{B_s}$ and $B_s = B_s-1$   
   \FOR{$b = 1,2 \ldots B_s$}	   
   \STATE $\thetsh \leftarrow \algs(b_s,\cB_b,\delta_s)$. Set $\thetsh \leftarrow \max(2\thetsh + 1, 2)$.
   \STATE Set $\cB_b \leftarrow \cB_b \cup \{b_s\}$ 
   \STATE Play $\cB_b$ for $t_s := \frac{2\thetsh}{\epsilon_s^2}\ln \frac{k}{\delta_s}$ rounds
   \STATE Receive the winner feedback: $\sigma_1, \sigma_2,\ldots \sigma_{t_s} \in \bSigma_{\cB_b}^1$ after each respective $t_s$ rounds.
   \STATE Update pairwise empirical win-count $w_{ij}$ using \rb \, on $\sigma_1\ldots \sigma_{t_s}, ~\forall i,j \in \cB_b$  
   \STATE $\hp_{ij} := \frac{w_{ij}}{w_{ij} + w_{ji}}$ for all $i,j \in \cB_b$
   \STATE \textbf{If} $\exists i \in \cB_b$ s.t. $\hp_{ib_s} > \frac{1}{2} - \epsilon_s$, \textbf{then} $\cA_s \leftarrow \cA_{s} \cup \{i\}$
   \ENDFOR
   \STATE $\cA_s \leftarrow \cA_{s} \cup \cR_s$, $s \leftarrow s + 1$
   \IF{$1 < |\cA_{s-1}| \le k$}
   \STATE Append $\cA_{s-1}$ with any $(k-|\cA_{s-1}|)$ elements from $[n]\sm \cA_{s-1}$
	\STATE Pairwise empirical win-count $w_{ij} \leftarrow 0$, $~\forall i,j \in \cA_{s-1}$; $\cA \leftarrow \cA_{s-1}$; $\cA' \leftarrow \cA_{s-1}$
   
	\WHILE {$|\cA'| > 1$}
   	\STATE Set $\epsilon_s = \frac{1}{2^{s+2}}$, and $\delta_s = \frac{\delta}{1200s^3}$
   \STATE $b_s \leftarrow $ \algep$(\cA,k,1,\epsilon_s/4,\delta_s)$
   \STATE \textbf{If} $(\epsilon_s \le \epsilon)$ \textbf{then} $\cA' \leftarrow \{b_s\}$, and exit the \textbf{while} loops (go to Line $35$)
   \STATE $\thetsh \leftarrow \algs(b_s,\cA\sm\{b_s\},\delta_s)$. Set $\thetsh \leftarrow \max(2\thetsh + 1, 2)$.
   \STATE Play $\cB$ for $t_s := \frac{2\thetsh}{\epsilon_s^2}\ln \frac{k}{\delta_s}$ rounds, and receive the corresponding winner feedback: $\sigma_1,\ldots \sigma_{t_s} \in \bSigma_{\cA}^1$ per round.
   \STATE Update pairwise empirical win-count $w_{ij}$ using \rb \, on $\sigma_1\ldots \sigma_{t_s}, ~\forall i,j \in \cA'$  
   \STATE Update $\hp_{ij} := \frac{w_{ij}}{w_{ij} + w_{ji}}$ for all $i,j \in \cA'$
   \STATE \textbf{If} $\exists i \in \cA'$ with $\hp_{ib_s} < \frac{1}{2} - \epsilon_s$, \textbf{then} $\cA' \leftarrow \cA' \sm \{i\}$
   \STATE $s \leftarrow s + 1$
  \ENDWHILE   
   
   \ENDIF
  \ENDWHILE
   \STATE {\bfseries output:} The item remaining in $\cA'$ 
\end{algorithmic}
\end{algorithm}
\end{center}

\subsection{Proof of Thm. \ref{thm:sc_fewfep}}  

\scalgfewfep*  
  
\begin{proof}
Let us denote by $s_0$ to be the sub-phase at which $\epsilon_s$ falls below $\epsilon$ for the first time, i.e. $s_0 := \arg \min_{s=1,2,\ldots}\1(\epsilon_s \le \epsilon)$.
We first proof the $(\epsilon,\delta)$-PAC correctness of the algorithm:

\emph{(Proof of Correctness)}: Note from Lem. \ref{lem:bistays} that the probability the \bi\, $1$ gets eliminated till sub-phase $s = 1,2,\ldots (s_0-1)$ is upper bounded by $\sum_{s = 1}^{s_0-1}\frac{2\delta_s}{k} \le \sum_{s = 1}^{s_0-1}{\delta_s}$, since $k \ge 2$.

So with probability at least $(1- \sum_{s = 1}^{s_0-1}{\delta_s})$, item $1$ survives till the beginning of sub-phase $s_0$. And by Thm. \ref{thm:sc_fewfep}, we know that with probability at least $(1-\delta_s)$, $p_{b_s1} > \frac{1}{2} - \epsilon_s/4 \implies  \theta_{b_s} > \theta_1 - \epsilon_s$, which ensures $\epsilon_s$ optimality of the item $b_s$ (see Defn. \ref{def:pac}).
So at $s = s_0$, we have $Pr(\theta_{b_{s_0}} > \theta_1 - \epsilon_{s_0}) > (1-\delta_{s_0}) $ which ensures the $(\epsilon,\delta)$ correctness of the algorithm as at $s_0$, $\epsilon_{s_0} \le \epsilon$.

Moreover the over all probability of the algorithm failing to return an $\epsilon$-optimal item is $\sum_{s = 1}^{s_0-1}{\delta_s} + \delta_{s_0} \le \sum_{s = 1}^{\infty}{\delta_s} = \frac{\delta}{120}\sum_{s = 1}^{\infty}\frac{1}{s^2} \le \frac{\delta}{120}\frac{\pi^2}{6} \le \frac{\delta}{20}$.

For the rest of the analysis we will assume that the claim of Lem. \ref{lem:nbigoes} holds good for all $s = 1,2\ldots s_0$, which we know to satisfy with probability at least $(1-\frac{19\delta}{20})$.

\emph{(Proof for Sample-complexity)}: We now proceed to prove the sample complexity of the algorithm.
Let us call $b_s$ to be the pivot item of any phase $s$, and denote the sample complexity of item $i$ (as a non-pivot element) from phase $x$ to $y$ as $\cN^{(i)}_{x,y}$, for any $1 \le x < y < \infty$.
Additionally, recalling from Lem. \ref{lem:algs} that $\thetsh \le 7\thetk$, we now prove the claim with the following two case analyses:

\textbf{(Case 1)} For suboptimal item $i \in [n]\sm\{1\}$ such that $\Delta_i > \epsilon$: Recall from Lem. \ref{lem:sc_item} that the sample complexity of item $i$ (as a non-pivot) is $\cN^{(i)}_{1,\infty} = O\Big( {(2^r)^2\thetk}\ln \frac{rk}{\delta}  \Big)$, where $i \in [n]_r$. Hence we further get $\cN^{(i)}_{1,\infty} = O\Big( \frac{\thetk}{\Delta_i^2}\ln \Big(\frac{k}{\delta}\ln \frac{1}{\Delta_i} \Big) \Big)$ as since $i \in [n]_r$, so by definition $\Delta_i < \frac{2}{2^r}$.

\textbf{(Case 2)} For items $i$ such that $\Delta_i \le \epsilon$: Recall due to Thm. \ref{lem:up_algep} the orderwise sample complexity of playing the sets $\cB_1,\ldots, \cB_{B_s}$ is same as that incurred due to calling the subroutine \algep\, at sub-phase $s$, for all $s=1,2,\ldots$. Now in the worst case, all items $i$ with $\Delta_i < \epsilon$ might survive till phase $s_0$. Thus the maximum sample complexity of any such item $i$ (as a non-pivot) till sub-phase $s_0$ can be upper bounded as:


\begin{align*}
\cN^{(i)}_{1,\infty} & = \cN^{(i)}_{1,s_0} \le \sum_{s = 1}^{s_0}t_s = \sum_{s = 1}^{r-1}\frac{14\thetk}{\epsilon_s^2}\ln \frac{k}{\delta_s} = \frac{14\thetk}{4^{-2}}\sum_{s = 1}^{s_0}(2^s)^2\ln \frac{120ks^3}{\delta}\\
& = O\Big( (2^{s_0+1})^2\thetk\ln \frac{(s_0+1)k}{\delta}  \Big) = O\bigg( \frac{\thetk}{\epsilon_{s_0}^2}\ln \Big(\frac{k}{\delta}\ln \frac{1}{\epsilon_{s_0}}  \Big)\bigg) = O\bigg( \frac{\thetk}{\epsilon^2}\ln \Big(\frac{k}{\delta}\ln \frac{1}{\epsilon}  \Big)\bigg),
\end{align*}
where the last equality follows as ${\epsilon} < 2\epsilon_{s_0} = \epsilon_{s_0-1}$, by definition of $s_0$.

Now denoting the number of times any $k$-subset $S \subseteq [n]$ played by the algorithm in sub-phase $x$ to $y$ as $\cN^{(S)}_{x,y}$, and using the claims from above two cases, the total sample complexity of the algorithm (lets call it algorithm $\cA$) becomes:

\begin{align*}
\cN_\cA(0,\delta) & = \sum_{S \subset [n]\mid |S|=k} \sum_{s=1}^{\infty}\1(S \in \{\cB_1,\ldots, \cB_{B_s}\})t_s = \sum_{i \in [n]} \sum_{s=1}^{\infty}\frac{\1(i \in \cA_s\sm\{b_s\})}{k-1}t_s\\
& = \sum_{s=1}^{\infty} \sum_{i \in [n]}\frac{\1(i \in \cA_s\sm\{b_s\})}{k-1}t_s \\
& = \sum_{s=1}^{\infty} \sum_{r=1}^{\log_2(\Delta_{\min})}\sum_{i \in [n]_r}\frac{\1(i \in \cA_s\sm\{b_s\})}{k-1}t_s = \frac{1}{k-1}\sum_{r=1}^{\log_2(\Delta_{\min})}\sum_{i \in [n]_r}\cN^{(i)}_{1,\infty}\\
& = \frac{1}{k-1}\sum_{r=1}^{\log_2(\Delta_{\min})}\bigg( \sum_{\{i \in [n]_r \mid \Delta_i > \epsilon\}}\cN^{(i)}_{1,\infty} + \sum_{\{i \in [n]_r \mid \Delta_i \le \epsilon\}}\cN^{(i)}_{1,\infty}\bigg)\\
& = \frac{1}{k-1}\sum_{r=1}^{\log_2(\Delta_{\min})}\bigg( \sum_{\{i \in [n]_r \mid \Delta_i > \epsilon\}}O\Big( \frac{\thetk}{\Delta_i^2}\ln \Big(\frac{k}{\delta}\frac{1}{\Delta_i} \Big) \Big) + \sum_{\{i \in [n]_r \mid \Delta_i \le \epsilon\}}O\bigg( \frac{\thetk}{\epsilon^2}\ln \Big(\frac{k}{\delta}\ln \frac{1}{\epsilon}  \Big)\bigg)\bigg)\\
& = O\bigg(\frac{\thetk}{k}\sum_{i = 2}^n\frac{1}{\max(\Delta_i,\epsilon)^2} \ln\frac{k}{\delta}\Big(\ln \frac{1}{\max(\Delta_i,\epsilon)}\Big)\bigg),
\end{align*}
where note that the second last inequality is follows from Case 1 and 2 derived above. Finally, as shown in the proof of Thm. \ref{thm:sc_fewf}, further taking into consideration the additional sample complexity incurred at each sub-phase $s$ due to invoking the \algep\, and \algs \, subroutine can shown to be at most $\frac{\Theta_{[k]}}{k}\sum_{i = 2}^{n}\max\big(1,\frac{1}{\max(\epsilon^2,\Delta_i^2)}\big) \log \frac{k}{\delta})$, combining which with the above sample complexity gives the desired sample complexity bound of Alg. \ref{alg:fewfep}.

\end{proof}  

\subsection{Modified version of \algfewf\, (Alg. \ref{alg:fewf}) for \tf}
\label{app:alg_fetf}

The pseudo code is provided in Alg. \ref{alg:fetf}.

\begin{center}
\begin{algorithm}[H]
   \caption{\textbf{\algfetf }}
   \label{alg:fetf}
\begin{algorithmic}[1]
   \STATE {\bfseries input:} Set of items: $[n]$, Subset size: $n \geq k > 1$, Ranking feedback size: $m \in [k-1]$, Confidence term $\delta > 0$
   \STATE {\bfseries init:} $\cA_0 \leftarrow [n]$, $s \leftarrow 1$ 
   \WHILE {$|\cA_{s-1}| \ge k$}
   	\STATE Set $\epsilon_s = \frac{1}{2^{s+2}}$, $\delta_s = \frac{\delta}{120s^3}$, and $\cR_s \leftarrow \emptyset$.
   \STATE $b_s \leftarrow $ \algep$(\cA_{s-1},k,m,\epsilon_s,\delta_s)$
   \STATE $\cB_1,\ldots \cB_{B_s} \leftarrow $ \algdiv$(\cA_{s-1}\sm \{b_s\},k-1)$ 
   \STATE \textbf{if} $|\cB_{B_s}| < k-1$, \textbf{then} $\cR_s \leftarrow \cB_{B_s}$ and $B_s = B_s-1$   
   \FOR{$b = 1,2 \ldots B_s$}	   
      \STATE $\thetsh \leftarrow \algs(b_s,\cB_b,\delta_s)$. Set $\thetsh \leftarrow \max(2\thetsh + 1, 2)$.
   \STATE Set $\cB_b \leftarrow \cB_b \cup \{b_s\}$ 
   \STATE Play $\cB_b$ for $t_s := \frac{2\thetsh}{m\epsilon_s^2}\ln \frac{k}{\delta_s}$ rounds
   \STATE Receive the winner feedback: $\sigma_1, \sigma_2,\ldots \sigma_{t_s} \in \bSigma_{\cB_b}^m$ after each respective $t_s$ rounds.
   \STATE Update pairwise empirical win-count $w_{ij}$ using \rb \, on $\sigma_1\ldots \sigma_{t_s}, ~\forall i,j \in \cB_b$  
   \STATE $\hp_{ij} := \frac{w_{ij}}{w_{ij} + w_{ji}}$ for all $i,j \in \cB_b$
   \STATE \textbf{If} $\exists i \in \cB_b$ with $\hp_{ib_s} > \frac{1}{2} - \epsilon_s$, \textbf{then} $\cA_s \leftarrow \cA_{s} \cup \{i\}$
   \ENDFOR
   \STATE $\cA_s \leftarrow \cA_{s} \cup \cR_s$, $s \leftarrow s + 1$
  \ENDWHILE
  \STATE $\cA \leftarrow \cA_{s-1}$; $\cB \leftarrow$ $\cA_{s-1} \cup \{(k-|\cA_{s-1}|)$ \text{ elements from } $[n]\sm \cA_{s-1}\}$
  \STATE Pairwise empirical win-count $w_{ij} \leftarrow 0$, $~\forall i,j \in \cA$ 
  \WHILE{$|\cA| > 1$}
  \STATE Set $\epsilon_s = \frac{1}{2^{s+2}}$, and $\delta_s = \frac{\delta}{120s^3}$
   \STATE $b_s \leftarrow $ \algep$(\cB,k,m,\epsilon_s,\delta_s)$
\STATE $\thetsh \leftarrow \algs(b_s,\cA\sm\{b_s\},\delta_s)$. Set $\thetsh \leftarrow \max(2\thetsh + 1, 2)$.
   \STATE Play $\cB$ for $t_s := \frac{2\thetsh}{m\epsilon_s^2}\ln \frac{k}{\delta_s}$ rounds, and receive the corresponding winner feedback: $\sigma_1, \sigma_2,\ldots \sigma_{t_s} \in \bSigma_{\cB}^m$ per round.
   \STATE Update pairwise empirical win-count $w_{ij}$ using \rb \, on $\sigma_1\ldots \sigma_{t_s}, ~\forall i,j \in \cA$  
   \STATE Update $\hp_{ij} := \frac{w_{ij}}{w_{ij} + w_{ji}}$ for all $i,j \in \cA$
   \STATE \textbf{If} $\exists i \in \cA$ with $\hp_{ib_s} < \frac{1}{2} - \epsilon_s$, \textbf{then} $\cA \leftarrow \cA \sm \{i\}$
   \STATE $s \leftarrow s + 1$
  \ENDWHILE
   \STATE {\bfseries output:} The item remaining in $\cA$
   \end{algorithmic}
\end{algorithm}
\end{center}

\subsection{Proof of Thm. \ref{thm:sc_fetf}}

\scalgfetf*


\begin{proof}
As argued, the main idea behind the $\frac{1}{m}$ factor improvement in the sample complexity w.r.t \wf\, (as proved in Thm. \ref{thm:sc_fewf}), lies behind using \rb \, updates (see Alg. \ref{alg:updt_win}) to the general \tf. This actually gives rise to $O(m)$ times additional number of pairwise preferences in comparison to \wf\, which is why in this case it turns out to be sufficient to sample any batch $\cB_b, \forall b \in [B_s]$ for only  $O\big( \frac{1}{m} \big)$ times compared to the earlier case---precisely the reason behind $\frac{1}{m}$-factor improved sample complexity of \algfewf\, for \tf. The rest of the proof argument is mostly similar to that of Thm. \ref{thm:sc_fewf}. We provide the detailed analysis below for the sake of completeness.


We start by proving the correctness of the algorithm, i.e. with high probability $(1-\delta)$, \algfewf\, indeed returns the \bi\,, i.e. item $1$ in our case. Towards this we first prove the following two lemmas: Lem. \ref{lem:bistaystf} and Lem. \ref{lem:nbigoestf}, same as what was derived for Thm. \ref{thm:sc_fewf} as well---However it is important to note that its is due to the \tf\, feedback the exact same guarantees holds in this case as well, even with a $m$-times lesser observed samples.

\begin{restatable}[]{lem}{bistaystf}
\label{lem:bistaystf}
With high probability of at least $(1-\frac{\delta}{20})$, item $1$ is never eliminated, i.e. $1 \in \cA_s$ for all sub-phase $s$. More formally, at the end of any sub-phase $s= 1,2,\ldots$, $\hp_{1b_s} > \frac{1}{2} - \epsilon_s$.
\end{restatable}

\begin{proof}
Firstly note that at any sub-phase $s$, each batch $b \in B$ within that phase is played for $t_s = \frac{2\thetk}{m\epsilon_s^2}\ln \frac{k}{\delta_s}$ rounds. Now consider the batch $\cB \owns 1$ at any phase $s$. Clearly $b_s \in \cB$ too. Again since $b_s$ is returned by Alg. \ref{alg:epsdel}, by Thm. \ref{lem:up_algep} we know that with probability at least $(1-\delta_s)$, $p_{b_s1} > \frac{1}{2} - \epsilon_s \, \implies \theta_{b_s} > \theta_1 - 4\epsilon$. This further implies $\theta_{b_s} \ge \theta_1 - \frac{1}{2} = \frac{1}{2}$ (since we assume $\theta_1 = 1$, and at any $s,\, \epsilon_s < \frac{1}{8}$). Moreover by Lem. \ref{lem:algs}, we have $\thetsh \ge \frac{\theta_{b_s} +\thets }{\theta_{b_s}} > \frac{\thets + 1}{2}$ (recall we denote $S = \cB_b \sm \{b_s\}$)

Now let us define $w_i$ as number of times item $i \in \cB$ was returned as the winner in $t_s$ rounds and $i_\tau$ be the winner retuned by the environment upon playing $\cB$ for the $\tau^{th}$ round, where $\tau \in [t_s]$.
Then clearly $Pr(\{1 \in \sigma_\tau\}) = \sum_{j = 1}^{m}Pr\big( \sigma_\tau(j) = 1 \Big) \ge \frac{m\theta_{1}}{\sum_{j \in \cB}\theta_j} = \frac{m}{\theta_{b_s} + \thets} \ge \frac{m}{1+ \thets}, \, \forall \tau \in [t_s]$, as $1 := \arg \max_{i \in \cB}\theta_i$. Hence $\E[w_{1}] = \sum_{\tau = 1}^{t_s}\E[\1(i_\tau = 1)] = \frac{mt_s}{\theta_{b_s} + \thets} \ge \frac{mt_s}{(1+\thets)}$. 
Now assuming $b_s$ to be indeed an $(\epsilon_s,\delta_s)$-PAC \bi\, and the bound of Lem. \ref{lem:algs} to hold good as well, applying multiplicative Chernoff-Hoeffdings bound on the random variable $w_{1}$, we get that for any $\eta \in (\sqrt{2}\epsilon_s,1]$, 

\begin{align*}
Pr\Big( w_{1} \le (1-\eta)\E[w_{1}] \Big) & \le \exp\bigg(- \frac{\E[w_{1}]\eta^2}{2}\bigg) \le \exp\bigg(- \frac{mt_s\eta^2}{2(1+\thets)}\bigg) \\
& = \exp\bigg(- \frac{2\thetsh\eta^2}{2m\epsilon_s^2(1+\thets)}\ln \frac{k}{\delta_s}\bigg)\\
&\stackrel{(a)}{\le} \exp\bigg(- \frac{2(\thets+1)\eta^2}{4\epsilon_s^2(1+\thets)}\ln \frac{k}{\delta_s}\bigg)\\
& \le \exp\bigg(- \frac{\eta^2}{2\epsilon_s^2} \ln \bigg( \frac{k}{\delta_s} \bigg) \bigg) \le \exp\bigg(- \ln \bigg( \frac{k}{\delta_s} \bigg) \bigg) = \frac{\delta_s}{k},
\end{align*}
where $(a)$ holds since we proved $\thetsh \ge \frac{\theta_{b_s} +\thets }{\theta_{b_s}} > \frac{\thets + 1}{2}$, and the last inequality holds as $\eta > \epsilon_s\sqrt{2}$.

In particular, note that $\epsilon_s \le \frac{1}{8}$ for any sub-phase $s$, due to which we can safely choose $\eta = \frac{1}{2}$ for any $s$, which gives that with probability at least $\big(1-\frac{\delta_s}{k}\big)$,  $w_{1} > (1-\frac{1}{2})\E[w_{1}] > \frac{mt_s}{2(\theta_{b_s}+\thets)}$, for any subphase $s$.


Thus above implies that with probability atleast $(1-\frac{\delta_s}{k})$, after $t_s$ rounds we have $w_{1b_s} \ge \frac{mt_s}{2(\theta_{b_s}+\thets)} \implies w_{1b_s} + w_{b_s1} \ge \frac{mt_s}{2(\theta_{b_s}+\thets)}$. Let us denote  $n_{1b_s} = w_{1b_s} + w_{b_s1}$. Then the probability of the event:

\begin{align*}
Pr\Bigg(  \hp_{1 b_s} & < \frac{1}{2} - \epsilon_s  ,  n_{1 b_s} \ge \frac{mt_s}{2(\theta_{b_s}+\thets)}  \Bigg)
 = Pr\Bigg( \hp_{1 b_s} - \frac{1}{2} < - \epsilon_s  ,  n_{1 b_s} \ge \frac{mt_s}{2(\theta_{b_s}+\thets)} \Bigg)\\
& \le Pr\Bigg( \hp_{1 b_s} - p_{1 b_s} < - \epsilon_s  ,  n_{1 b_s} \ge \frac{mt_s}{2(\theta_{b_s}+\thets)} \Bigg) ~~~(\text{as } \p_{1 b_s} > \frac{1}{2})\\
& \stackrel{(a)} \le \exp\Big( -2\frac{mt_s}{2(\theta_{b_s}+\thets)}\big({\epsilon_s}\big)^2 \Big) \\
& \le \exp\Big( -\frac{2m(\theta_{b_s} + \thets)}{m\epsilon_s^2\theta_{b_s}(\theta_{b_s}+\thets)}\big({\epsilon_s}\big)^2 \Big) 
\le \frac{\delta_s}{k},
\end{align*}

where the last inequality $(a)$ follows from Lem. \ref{lem:pl_simulator} for $\eta = \epsilon_s$ and $v = \frac{t_s}{2k}$.

Thus under the two assumptions that (1). $b_s$ is indeed an $(\epsilon_s,\delta_s)$-PAC \bi\, and (2). the bound of Lem. \ref{lem:algs} holds good, combining the above two claims, at any sub-phase $s$, we have

\begin{align*}
& Pr\Bigg(  \hp_{1 b_s} < \frac{1}{2} - \epsilon_s\Bigg)\\
 & = 
Pr\Bigg(  \hp_{1 b_s} < \frac{1}{2} - \epsilon_s  ,  n_{1 b_s} \ge \frac{mt_s}{2(\theta_{b_s}+\thets)}  \Bigg) + Pr\Bigg(  \hp_{1 b_s} < \frac{1}{2} - \epsilon_s  ,  n_{1 b_s} < \frac{mt_s}{2(\theta_{b_s}+\thets)}  \Bigg)\\
& \le \frac{\delta_s}{k} + Pr\Bigg( n_{1 b_s} < \frac{mt_s}{2(\theta_{b_s}+\thets)}  \Bigg) \le \frac{2\delta_s}{k} \le \delta_s ~~~(\text{since } k \ge 2)\\
\end{align*}

Moreover from Thm. \ref{lem:up_algep} and Lem. \ref{lem:algs} we know that the above two assumptions hold with probability at least $(1-2\delta_s)$. Then taking union bound over all sub-phases $s = 1,2, \ldots$, the probability that item $1$ gets eliminated at any round:

\begin{align*}
Pr\Bigg( \exists s = 1,2, \ldots \text{ s.t. } \hp_{1 b_s} < \frac{1}{2} - \epsilon_s \Bigg) & = \sum_{s = 1}^{\infty}{3\delta_s} = \sum_{s = 1}^{\infty}\frac{3\delta}{120s^3} \le \frac{\delta}{40}\sum_{s = 1}^{\infty}\frac{1}{s^2} \le \frac{\delta}{40}\frac{\pi^2}{6} \le \frac{\delta}{20},
\end{align*}
where the first inequality holds since $k \ge 2$.
\end{proof}
Recall the notations introduced in the proof of Thm. \ref{thm:sc_fewf}: $\Delta_i = \theta_1-\theta_i$ (Sec. \ref{sec:prb_setup}), $\Delta_{\min} = \min_{i \in [n]\sm\{1\}}\Delta_i$.
Further $[n]_r := \{ i \in [n] : \frac{1}{2^r} \le \Delta_i < \frac{1}{2^{r-1}} \}$, and $\cA_{r,s}$, i.e. $\cA_{r,s} = [n]_r \cap \cA_s$, for all $s = 1,2, \ldots$.
Then in this case again we claim:

\begin{restatable}[]{lem}{nbigoestf}
\label{lem:nbigoestf}
Assuming that the best arm $1$ is not eliminated at any sub-phase $s = 1,2,\ldots$, then with probability at least $(1-\frac{19\delta}{20})$, for any sub-phase $s \ge r$, $|\cA_{r,s}| \le \frac{2}{19}|\cA_{r,s-1}|$, for any $r = 0,1,2,\ldots \log_2 (\Delta_{\min})$. 
\end{restatable}

\begin{proof}
Consider any sub-phase $s$, and let us start by noting some properties of $b_s$. Note that by Lem. \ref{lem:up_algep}, with high probability $(1-\delta_s)$, $p_{b_s 1} > \frac{1}{2} - \epsilon_s$. Then this further implies

\begin{align*}
p_{b_s 1} > \frac{1}{2} - \epsilon_s & \implies \frac{(\theta_{b_s} - \theta_1)}{2(\theta_{b_s} + \theta_1)} > -\epsilon_s\\
& \implies (\theta_{b_s} - \theta_1) > -2\epsilon_s(\theta_{b_s} + \theta_1) > -4\epsilon_s ~~(\text{as} ~\theta_i \in (0,1)\, \forall i \in [n]\sm\{1\}) 
\end{align*}  

So we have with probability atleast $(1-\delta_s)$,  $\theta_{b_s} > \theta_1 - 4\epsilon_s = \theta_1 - \frac{1}{2^s}$. 

Now consider any fixed $r = 0,1,2,\ldots \log_2(\Delta_{\min})$. Clearly by definition, for any item $i \in [n]_r$, $\Delta_i = \theta_1 - \theta_i > \frac{1}{2^r}$.
Then combining the above two claims, we have for any sub-phase $s \ge r$, $\theta_{b_s} > \theta_1 - \frac{1}{2^s} \ge \theta_i + \frac{1}{2^r} - \frac{1}{2^s} > 0 \implies p_{b_si} > \frac{1}{2}$, at any $s \ge r$.
Moreover note that for any $s \ge 1$, $\epsilon_s > \frac{1}{8}$, so that implies $\theta_{b_s} > \theta_1 - 4\epsilon_s > \frac{1}{2}$.


Recall that at any sub-phase $s$, each batch within that phase is played for $t_s = \frac{2\thetk}{m\epsilon_s^2}\ln \frac{k}{\delta_s}$ many rounds. Now consider any batch such that $\cB \owns i$ for any $i \in [n]_r$. Of course $b_s \in \cB$ as well, and note that we have shown $p_{b_si} > \frac{1}{2}$ with high probability $(1-\delta_s)$. 

Same as Lem. \ref{lem:bistays}, let us again define $w_i$ as number of times item $i \in \cB$ was returned as the winner in $t_s$ rounds, and $i_\tau$ be the winner retuned by the environment upon playing $\cB$ for the $\tau^{th}$ rounds, where $\tau \in [t_s]$.
Then given $\theta_{b_s} > \frac{1}{2}$ (as derived earlier), clearly $Pr(\{b_s \in \sigma_\tau\}) = \sum_{j = 1}^{m}Pr\big( \sigma_\tau(j) = b_s \Big) \ge \sum_{j = 0}^{m-1}\frac{\theta_{b_s}}{(\theta_{b_s} + \thets)} = \frac{m\theta_{b_s}}{\theta_{b_s} + \thets}, \, \forall \tau \in [t_s]$. Hence $\E[w_{b_s}] = \sum_{\tau = 1}^{t_s}\E[\1(i_\tau = b_s)] = \frac{m\theta_{b_s}t_s}{(\theta_{b_s} + \thets)}$. 
Now applying multiplicative Chernoff-Hoeffdings bound on the random variable $w_{b_s}$, we get that for any $\eta \in (\sqrt 2\epsilon_s,1]$, 

\begin{align*}
Pr\Big( w_{b_s} & \le (1-\eta)\E[w_{b_s}] \mid \theta_{b_s} > \frac{1}{2}, \thetsh > \frac{(\theta_{b_s} + \thets)}{\theta_{b_s}} \Big)\\
 & \le \exp\bigg(- \frac{\E[w_{b_s}]\eta^2}{2}\bigg) \le \exp\bigg(- \frac{m\theta_{b_s}t_s\eta^2}{2(\theta_{b_s} + \thets)}\bigg) \\
& \le \exp\bigg(- \frac{\eta^2}{2\epsilon_s^2} \ln \bigg( \frac{k}{\delta_s} \bigg) \bigg) \le \exp\bigg(- \ln \bigg( \frac{k}{\delta_s} \bigg) \bigg) = \frac{\delta_s}{k},
\end{align*}
where the last inequality holds as $\eta > \sqrt 2\epsilon_s$.
So as a whole, for any $\eta \in (\sqrt 2\epsilon_s,1]$,

\begin{align*}
& Pr\Big( w_{b_s} \le (1-\eta)\E[w_{b_s}] \Big)\\
& \le Pr\Big( w_{b_s} \le (1-\eta)\E[w_{b_s}] \mid \theta_{b_s} > \frac{1}{2}, \thetsh > \frac{(\theta_{b_s} + \thets)}{\theta_{b_s}}  \Big)Pr\Big(\theta_{b_s} > \frac{1}{2}\Big) + Pr\Big(\theta_{b_s} < \frac{1}{2},\thetsh > \frac{(\theta_{b_s} + \thets)}{\theta_{b_s}} \Big)\\ 
& \le \frac{\delta_s}{k} + 2\delta_s
\end{align*}

In particular, note that $\epsilon_s < \frac{1}{8}$ for any  sub-phase $s$, due to which we can safely choose $\eta = \frac{1}{2}$ for any $s$, which gives that with probability at least $\big(1-\frac{\delta_s}{k}\big)$,  $w_{b_s} > (1-\frac{1}{2})\E[w_{b_s}] > \frac{mt_s \theta_{b_s}}{2 (\thets + \theta_{b_s})}$, for any subphase $s$.

Thus above implies that with probability at least $(1-\frac{\delta_s}{k}-2\delta_s)$, after $t_s$ rounds we have $w_{b_si} \ge \frac{mt_s \theta_{b_s}}{2 (\thets + \theta_{b_s})} \implies w_{ib_s} + w_{b_si} \ge \frac{mt_s \theta_{b_s}}{2 (\thets + \theta_{b_s})}$. Let us denote  $n_{ib_s} = w_{ib_s} + w_{b_si}$. Then the probability that item $i$ is not eliminated at any sub-phase $s \ge r$ is:

\begin{align*}
Pr\Bigg(  \hp_{i b_s} > \frac{1}{2} & - \epsilon_s  ,  n_{i b_s} \ge \frac{mt_s \theta_{b_s}}{2 (\thets + \theta_{b_s})}  \Bigg)
 = Pr\Bigg( \hp_{i b_s} - \frac{1}{2} > - \epsilon_s  ,  n_{i b_s} \ge \frac{mt_s \theta_{b_s}}{2 (\thets + \theta_{b_s})} \Bigg)\\
& \le Pr\Bigg( \hp_{i b_s} - p_{i b_s} > - \epsilon_s  ,  n_{1 b_s} \ge \frac{mt_s \theta_{b_s}}{2 (\thets + \theta_{b_s})}\Bigg) ~~~(\text{as } \p_{i b_s } < \frac{1}{2})\\
& \stackrel{(a)}\le \exp\Big( -2\frac{mt_s}{2(\thets + \theta_{b_s})}\big({\epsilon_s}\big)^2 \Big)\\
& \le \exp\Big( -2\frac{m(\thets + \theta_{b_s})}{2m\theta_{b_s}(\thets + \theta_{b_s})}\big({\epsilon_s}\big)^2 \ln \frac{\delta_s}{k}\Big) \le \frac{\delta_s}{k},
\end{align*}

where $(a)$ follows from Lem. \ref{lem:pl_simulator} for $\eta = \epsilon_s$ and $v = \frac{t_s}{2k}$.

Now combining the above two claims, at any sub-phase $s$, we have:

\begin{align*}
Pr\Bigg(  \hp_{i b_s} > \frac{1}{2} - \epsilon_s\Bigg) & = 
Pr\Bigg(  \hp_{i b_s} > \frac{1}{2} - \epsilon_s  ,  n_{i b_s} \ge \frac{t_s}{4k}  \Bigg) + Pr\Bigg(  \hp_{i b_s} > \frac{1}{2} - \epsilon_s  ,  n_{i b_s} < \frac{t_s}{4k}  \Bigg)\\
& \le \frac{\delta_s}{k} + 2\delta_s + \frac{\delta_s}{k}  \le 3\delta_s ~~~~ (\text{since } k \ge 2)
\end{align*}

This consequently implies that for any sup-phase $s \ge r$, $\E[|\cA_{r,s}|] \le {3\delta_s}\E[|\cA_{r,s-1}|]$.
Then applying Markov's Inequality we get:

\begin{align*}
Pr\Bigg( |\cA_{r,s}| \le \frac{2}{19}|\cA_{r,s-1}| \Bigg)  \le \frac{3\delta_s|\cA_{r,s-1}|}{\frac{2}{19}|\cA_{r,s-1}|} = \frac{57\delta_s}{2}
\end{align*}

Finally applying union bound over all sub-phases $s= 1,2,\ldots$, and all $r = 1,2, \ldots s$, we get: 

\begin{align*}
\sum_{s = 1}^{\infty}\sum_{r = 1}^{s}\frac{57\delta_s}{2} = \sum_{s = 1}^{\infty}s\frac{57\delta}{240s^3} = \frac{57\delta}{240}\sum_{s = 1}^{\infty}\frac{1}{s^2} \le \frac{57\delta}{240}\frac{\pi^2}{6} \le \frac{57\delta}{120} \le \frac{19\delta}{20}.
\end{align*}

\end{proof}


Thus combining Lem. \ref{lem:bistaystf} and \ref{lem:nbigoestf}, we get that the total failure probability of \algfewf\, is at most $\frac{\delta}{20} + \frac{19\delta}{20} = \delta $.

The remaining thing is to prove the sample complexity bound which crucially follows from a similar claim as proved in Lem. \ref{lem:sc_item}. As before, at any sub-phase $s$, we call the item $b_s$ as the pivot element of phase $s$, then

\begin{restatable}[]{lem}{scitemtf}
\label{lem:sc_itemtf}
Assume both Lem. \ref{lem:bistaystf} and Lem. \ref{lem:nbigoestf} holds good and the algorithm does not do a mistake. Consider any item $i \in [n]_r$, for any $r = 1,2, \ldots \log_2(\Delta_{\min})$. Then the total number of times item $i$ gets played as a non-pivot item (i.e. appear in at most one of the $k$-subsets per sub-phase $s$) during the entire run of Alg. \ref{alg:fetf} is $O\Big( \frac{(2^r)^2\thetk}{m}\ln \frac{rk}{\delta}  \Big)$. 
\end{restatable}

\begin{proof}
Let us denote the sample complexity of item $i$ (as a non-pivot element) from phase $x$ to $y$ as $\cN^{(i)}_{x,y}$, for any $1 \le x < y < \infty$. Additionally, recalling from Lem. \ref{lem:algs} that $\thetsh \le 7\thetk$, we now prove the claim with the following two case analyses:

\textbf{(Case 1)} Sample complexity till sub-phase $s = r-1$: Note that in the worst case item $i$ can get picked at every sub-phase $s = 1,2,\ldots$ $r-1$, and at every $s$ it is played for $t_s$ rounds. Thus the total number of plays of item $i \in [n]_r$ (as a non-pivot item), till sub-phase $r-1$ becomes:

\[
\cN^{(i)}_{1,r-1} \le \sum_{s = 1}^{r-1}t_s = \sum_{s = 1}^{r-1}\frac{14\thetk}{m\epsilon_s^2}\ln \frac{k}{\delta_s} = \frac{14\thetk}{m4^{-2}}\sum_{s = 1}^{r-1}(2^s)^2\ln \frac{120ks^3}{\delta} = O\Big( \frac{(2^r)^2\thetk}{m}\ln \frac{rk}{\delta}  \Big)
\]

\textbf{(Case 2)} Sample complexity from sub-phase $s \ge r$ onwards: Assuming Lem. \ref{lem:nbigoestf} holds good, note that if we define a random variable $I_s$ for any sub-phase $s \ge r$ such that $I_s = \1(i \in \cA_s)$, then clearly $\E[I_{s}] \le \frac{2}{19}\E[I_{s-1}]$ (as follows from the analysis of Lem. \ref{lem:bistaystf}). Then the total expected sample complexity of item $i \in [n]_r$ for round $r, r+1, \ldots \infty$ becomes:
\begin{align*}
& \cN^{(i)}_{r,\infty}\le \frac{224\thetk}{m}\sum_{s = r}^{\infty}\bigg(\frac{2}{19}\bigg)^{s-r+1}4^s\ln \frac{k}{\delta_s} = \frac{224\thetk(2^r)^2}{m}\sum_{s = 0}^{\infty}\bigg(\frac{2}{19}\bigg)^{s+1}(2^s)^2\ln \frac{120k(s+r)^3}{\delta}\\
& = \frac{448}{19m}\thetk(2^r)^2\sum_{s = 0}^{\infty}\bigg(\frac{8}{19}\bigg)^{s}\ln \frac{120k(s+r)^3}{\delta} \\
& \le \frac{448}{19m}\thetk(2^r)^2\Bigg[\ln \frac{120kr}{\delta}\sum_{s = 0}^{\infty}\bigg(\frac{8}{19}\bigg)^{s} + \sum_{s = 0}^{\infty}\bigg(\frac{8}{19}\bigg)^{s}\ln(120ks)\Bigg] =  O\Big( \frac{(2^r)^2\thetk}{m}\ln \frac{rk}{\delta}  \Big)
\end{align*}

Combining the two cases above we get $\cN^{(i)}_{1,\infty} = O\Big( \frac{(2^r)^2\thetk}{m}\ln \frac{rk}{\delta}\Big)$ as well, which concludes the proof.
\end{proof}

Following similar notations as $\cN^{(i)}_{x,y}$, we now denote the number of times any $k$-subset $S \subseteq [n]$ played by the algorithm in sub-phase $x$ to $y$ as $\cN^{(S)}_{x,y}$. Then using Lem. \ref{lem:sc_itemtf}, the total sample complexity of the algorithm \algfewf\, (lets call it algorithm $\cA$) can be written as:

\begin{align}
\label{eq:sc_fetf}
\nonumber \cN_\cA(0,\delta) & = \sum_{S \subset [n]\mid |S|=k} \sum_{s=1}^{\infty}\1(S \in \{\cB_1,\ldots, \cB_{B_s}\})t_s = \sum_{i \in [n]} \sum_{s=1}^{\infty}\frac{\1(i \in \cA_s\sm\{b_s\})}{k-1}t_s\\
\nonumber & = \sum_{s=1}^{\infty} \sum_{i \in [n]}\frac{\1(i \in \cA_s\sm\{b_s\})}{k-1}t_s = \sum_{s=1}^{\infty} \sum_{r=1}^{\log_2(\Delta_{\min})}\sum_{i \in [n]_r}\frac{\1(i \in \cA_s\sm\{b_s\})}{k-1}t_s \\
\nonumber & = \frac{1}{k-1}\sum_{r=1}^{\log_2(\Delta_{\min})}\sum_{i \in [n]_r}\cN^{(i)}_{1,\infty} ~~(\text{ Lem. \ref{lem:sc_item}})\\
\nonumber & = \frac{1}{k-1}\sum_{r=1}^{\log_2(\Delta_{\min})}\sum_{i \in [n]_r}O\Big( \frac{(2^r)^2\thetk}{m}\ln \frac{rk}{\delta}\Big) = \frac{\thetk}{k-1}\sum_{r=1}^{\log_2(\Delta_{\min})}|[n]_r|O\Big( \frac{(2^r)^2}{m}\ln \frac{rk}{\delta}\Big) \\
& = O\Big( \frac{\thetk}{k}\sum_{i=2}^{n}\frac{1}{m\Delta_i^2}\ln \big(\frac{k}{\delta} \ln \frac{1}{\Delta_i} \big) \Big),
\end{align}
where the last inequality follows since $2^r < \frac{2}{\Delta_i}$ by definition for all $i \in [n]_r$.
Finally, same as derived in the proof of Thm. \ref{thm:sc_fewf}, the last thing to account for is the additional sample complexity incurred due to calling the subroutine \algep\ and \algs\, at every sub-phase $s$, which is combinedly known to be of $O\bigg(\frac{|\cA_s|\Theta_[k]}{k} \Big( 1,\frac{1}{(m 2^s)^2}\Big)\ln \frac{k}{\delta} \bigg)$ at any sub-phase $s$ (from Thm. \ref{lem:up_algep} and Cor. \ref{cor:algs}). And using a similar summation as shown above over all $s = 1,2,\ldots \infty$, combined with Lem. \ref{lem:nbigoestf} and using the fact that $2^r < \frac{2}{\Delta_i}$, one can show that the total sample complexity incurred due to the above subroutines is at most $\frac{\Theta_{[k]}}{k}\sum_{i = 2}^{n}\max\big(1,\frac{1}{m \Delta_i^2}\big) \log \frac{k}{\delta})$. Considering the above sample complexity added with the one derived in Eqn. \ref{eq:sc_fetf} finally gives the desired $O\bigg(\frac{\thetk}{k}\sum_{i = 2}^n\max\Big(1,\frac{1}{m \Delta_i^2}\Big) \ln\frac{k}{\delta}\Big(\ln \frac{1}{\Delta_i}\Big)\bigg)$ sample complexity bound of Alg. \ref{alg:fetf}.


\end{proof}


\section{Appendix for Sec. \ref{sec:lb}  }
  
\subsection{Proof of Thm. \ref{thm:lb_fewf}}
\label{app:lb_fewf}

\lbfewf*

\begin{proof}
The argument is based on a change-of-measure argument (Lemma  $1$) of \cite{Kaufmann+16_OnComplexity}, restated below for convenience: 

Consider a multi-armed bandit (MAB) problem with $n$ arms or actions $\cA = [n]$. At round $t$, let $A_t$ and $Z_t$ denote the arm played and the observation (reward) received, respectively. Let $\cF_t = \sigma(A_1,Z_1,\ldots,A_t,Z_t)$ be the sigma algebra generated by the trajectory of a sequential bandit algorithm upto round $t$.
\begin{restatable}[Lemma $1$, \cite{Kaufmann+16_OnComplexity}]{lem}{gar16}
\label{lem:gar16}
Let $\nu$ and $\nu'$ be two bandit models (assignments of reward distributions to arms), such that $\nu_i ~(\text{resp.} \,\nu'_i)$ is the reward distribution of any arm $i \in \cA$ under bandit model $\nu ~(\text{resp.} \,\nu')$, and such that for all such arms $i$, $\nu_i$ and $\nu'_i$ are mutually absolutely continuous. Then for any almost-surely finite stopping time $\tau$ with respect to $(\cF_t)_t$,
\vspace*{-15pt}
\begin{align*}
\sum_{i = 1}^{n}\E_{\nu}[N_i(\tau)]KL(\nu_i,\nu_i') \ge \sup_{\cE \in \cF_\tau} kl(Pr_{\nu}(\cE),Pr_{\nu'}(\cE)),
\end{align*}
where $kl(x, y) := x \log(\frac{x}{y}) + (1-x) \log(\frac{1-x}{1-y})$ is the binary relative entropy, $N_i(\tau)$ denotes the number of times arm $i$ is played in $\tau$ rounds, and $Pr_{\nu}(\cE)$ and $Pr_{\nu'}(\cE)$ denote the probability of any event $\cE \in \cF_{\tau}$ under bandit models $\nu$ and $\nu'$, respectively.
\end{restatable}

The heart of the lower bound analysis stands on the ground on constructing \pll\, instances, and slightly modified versions of it such that no $(0,\delta)$-PAC algorithm can correctly identify the \bi\, of both the instances without examining enough (precisely $\Omega\Big(\sum_{i=2}^{n}\frac{\theta_i\theta_1}{\Delta_i^2}\ln \big( \frac{1}{\delta} \big) \Big)$) many subsetwise samples per instance. We describe the our constructed problem instances below:

Consider an \pll\, instance with the arm (item) set $A$ containing all subsets of size $k$ of $[n]$ defined as $A = \{S \subseteq [n] \mid |S| = [k]\}$. Let PL$(n,\btheta^1)$ be the true distribution associated to the bandit arms $[n]$, given by the score parameters $\btheta = (\theta_1,\ldots,\theta_n)$, such that $\theta_1 > \theta_i, \, \forall i \in [n]\setminus\{1\}$. Thus we have
\begin{align*}
\textbf{True Instance: } \text{PL}(n,\btheta^1): \theta_1^1 > \theta_2^1 \ge \ldots \ge \theta_n^1.
\end{align*}

Clearly, the \bi\, of PL$(n,\btheta^1)$ is $a^* = 1$. Now for every suboptimal item $a \in [n]\setminus \{1\}$, consider the altered problem instance PL$(n,\btheta^a)$ such that:
\begin{align*}
\textbf{Instance a: } \text{PL}(n,\btheta^a): \theta_a^a = \theta_1^1 + \epsilon; ~\theta_i^a = \theta_i^1, ~~\forall i \in [n]\sm \{a\}
\end{align*}

for some $\epsilon > 0$. Clearly, the \bi\, of PL$(n,\btheta^a)$ is $a^* = a$. Note that, for problem instance PL$(n,\btheta^a) \, a \in [n]$, the probability distribution associated to arm $S \in A$ is given by:
\[
p^a_S \sim Categorical(p_1, p_2, \ldots, p_k), \text{ where } p_i = Pr(i|S) = \frac{\theta_i^a}{\sum_{j \in S}\theta_j^a}, ~~\forall i \in [k], \, \forall S \in A, \, \forall a \in [n],
\]
recall the definition of $Pr(i|S)$ is as defined in Sec. \ref{sec:prb_setup}. Now applying Lem. \ref{lem:gar16} we get: 

\begin{align}
\label{eq:FI_a}
\sum_{\{S \in A \mid a \in S\}}\E_{\btheta^1}[N_S(\tau_\cA)]KL(p^1_S, p^a_S) \ge {kl(Pr_{\btheta^1}(\cE), Pr_{\btheta^a}(\cE))},
\end{align}
where $\tau_\cA := N_\cA(0,\delta)$ denotes the sample complexity (number of rounds of subsetwise game played before stopping) of Alg. $\cA$ and for any subset $S \in A$, $N_S(\tau)$ denotes the number of times $S$ was played by $\cA$ in $N_\cA(0,\delta)$ rounds.
The above result holds from the straightforward observation that for any arm $S \in \cA$ with $a \notin S$, $p^1_S$ is same as $p^a_S$, hence $KL(p^1_S, p^a_S)=0$, $\forall S \in A, \,a \notin S$. 

For the notational convenience we will henceforth denote $S^a = \{S \in \cA \mid a \in S\}$. 
Now let us analyse the right hand side of \eqref{eq:FI_a}, for any set $S \in S^a$. 
We further denote $\Delta'_a = \Delta_a + \epsilon = (\theta_1 - \theta_a) + \epsilon$, 
and $\theta_S^a = \sum_{i \in S}\theta_i^a$ for any $a \in [n]$.
Now using the following upper bound on $KL(\p,\q) \le \sum_{x \in \X}\frac{p^2(x)}{q(x)} -1$, $\p$ and $\q$ be two probability mass functions on the discrete random variable $\X$ \cite{klub16}, we get:

\begin{align}
\label{eq:lb_wiwf_kl}
\nonumber KL(p^1_S, p^a_S) & \le \sum_{i \in S\sm\{a\}}\bigg(\frac{\theta_i^1}{\theta_S^1}\bigg)^2\bigg( \frac{\theta_S^a}{\theta_i^a}\bigg) + \bigg(\frac{\theta_a^1}{\theta_S^1}\bigg)^2\bigg( \frac{\theta_S^a}{\theta_a^a}\bigg) - 1\\
\nonumber & = \sum_{i \in S\sm\{a\}}\bigg(\frac{\theta_i^1}{\theta_S^1}\bigg)^2\bigg( \frac{\theta_S^1 + \Delta_a'}{\theta_i^1}\bigg) + \bigg(\frac{\theta_a^1}{\theta_S^1}\bigg)^2\bigg( \frac{\theta_S^1 + \Delta_a'}{\theta_a^1 + \Delta_a'}\bigg) - 1\\
\nonumber & = \bigg(\frac{\theta_S^1 + \Delta_a'}{(\theta_S^1)^2} \bigg) \bigg(\sum_{i \in S\sm\{a\}}\theta^1_i + \frac{(\theta_a^1)^2}{\theta_a^1 + \Delta_a'}\bigg) - 1\\
\nonumber & = \bigg(\frac{\theta_S^1 + \Delta_a'}{(\theta_S^1)^2} \bigg) \bigg( \frac{\theta_a^1\theta_S^1 + \Delta_a'(\theta_S^1 - \theta_a^1)}{\theta_a^1 + \Delta_a'}\bigg) - 1 ~~~\bigg[\text{replacing } \sum_{i \in S\sm\{a\}}\theta^1_i = (\theta_S^1 - \theta_a^1)\bigg]\\
& = \frac{ \Delta_a'^2(\theta_S^1 - \theta_a^1)}{(\theta_S^1)^2(\theta_a^1 + \Delta_a')} \le \frac{ \Delta_a'^2}{\theta_S^1(\theta_a^1 + \Delta_a')} = \frac{ \Delta_a'^2}{\theta_S^1(\theta_1^1 + \epsilon)}
\end{align}

Now, consider $\cE_0 \in \cF_\tau$ be an event such that the algorithm $\cA$ returns the element $i = 1$, and let us analyse the left hand side of \eqref{eq:FI_a} for $\cE = \cE_0$. Clearly, $A$ being an $(0,\delta)$-PAC algorithm, we have $Pr_{\btheta^1}(\cE_0) > 1-\delta$, and $Pr_{\btheta^a}(\cE_0) < \delta$, for any suboptimal arm $a \in [n]\setminus\{1\}$. Then we have: 

\begin{align}
\label{eq:win_lb2}
kl(Pr_{\btheta^1}(\cE_0),Pr_{\btheta^a}(\cE_0)) \ge kl(1-\delta,\delta) \ge \ln \frac{1}{2.4\delta}
\end{align}

where the last inequality follows from \cite{Kaufmann+16_OnComplexity}(see Eqn. $(3)$).
Now combining \eqref{eq:FI_a} and \eqref{eq:win_lb2}, for each problem instance PL$(n,\btheta^a)$, $a\in [n]\sm\{1\}$, we get,

\begin{align*}
\nonumber \sum_{S \in S^a}\E_{\btheta^1}[N_S(\tau_A)]KL(p^1_S,p^a_S) & \ge \ln \frac{1}{2.4\delta}\\
\end{align*}

Moreover, using \eqref{eq:lb_wiwf_kl}, we further get:

\begin{align}
\label{eq:win_lb3}
\ln \frac{1}{2.4\delta} \le \sum_{S \in S^a} \E_{\btheta^1}[N_S(\tau_A)]KL(p^1_S,p^a_S) \le \sum_{S \in S^a}\E_{\btheta^1}[N_S(\tau_A)] \frac{ \Delta_a'^2}{\theta_S^1(\theta_1^1 + \epsilon)}
\end{align}

Clearly, the total sample complexity of $\cA$: $\tau_A = \sum_{S \in A}N_S(\tau_\cA)$, then note that the problem of finding the sample complexity lower bound problem actually reduces down to 

\begin{align*}
\textbf{Primal LP (P):} \hspace{20pt}\min_{S \in A} \sum_{S \in A}& \E_{\btheta^1}[N_S(\tau_\cA)]\\
\text{such that, } ~~\ln \frac{1}{2.4\delta} \le & \sum_{S \in S^a}\E_{\btheta^1}[N_S(\tau_A)] \frac{ \Delta_a'^2}{\theta_S^1(\theta_1^1 + \epsilon)}, ~\forall a \in [n]\sm\{1\},
\end{align*}

which can equivalently be written as a linear programming (LP) of the following form:

\begin{align*}
\textbf{Dual LP (D):} \hspace{20pt}\min_{y} \b^\top\y\\
\text{such that, } \bK^\top \y & \ge \z, \text{ and } \y \geq 0,
\end{align*}

where $\y \in \R^{M}$, $M = |A| = \binom{n}{k}$, with $y(S) = \E_{\btheta^1}[N_S(\tau_A)], ~\forall S \in A$, $\z \in \R^{n-1}$ with $z(i) = \ln \frac{1}{2.4\delta} ~\forall i \in [n-1]$, $\bK \in \R^{M \times (n-1)}$ such that $K(S,a) = \begin{cases} \frac{ \Delta_a'^2}{\theta_S^1(\theta_1^1 + \epsilon)}, \text{ if } S \in S^a\\
0, \text{ otherwise }
\end{cases}$, and $\b \in \R^{M \times 1}$ such that $b(i) = 1 ~\forall i \in [M]$.

Then taking the dual of the above LP (see Chapter 5, \cite{boyd_book}) we get:

\begin{align*}
\max_{\x} \z^\top\x, ~~
\text{ such that, } \bK \x & \le \b, \text{ and } \x \geq 0,
\end{align*}
where clearly $\x \in \R^{n-1}$ is the dual optimization variable.

Now we know that by \emph{strong duality} if $\y^*$ and $\x^*$ respectively denotes the optimal solution of \textbf{(P)} and \textbf{(D)}, then $\b^\top \y^* = \z^{\top}\x^*$.
Thus at any feasible solution $\x'$ of \textbf{(D)}, $\z^{\top}\x' \le \z^{\top}\x^* = \b^\top \y^*$.

\textbf{Claim.} $x'_i = \frac{\theta_{i+1}^1(\theta_1^1+\epsilon)}{{\Delta_a'}^2}$ for all $i \in [n-1]$ is a feasible solution of \textbf{(D)}.

\begin{proof}
Clearly, $x'_i \ge 0 ~\forall i \in [n-1]$ which ensures that the second set of constraints of \textbf{(D)} hold good. Expanding the first set of constraints $\bK \x' \le \b$ we get $M$ constraints, one for each $S \in A$ such that 
\begin{align*}
\sum_{i = 1}^{n-1}K(S,i)x'_i & = \sum_{i = 1}^{n-1}\1(S \in S^{i+1})K(S,i)\frac{\theta_{i+1}^1(\theta_1^1+\epsilon)}{{\Delta_a'}^2} \\
& = \sum_{i = 2}^{n}\1(i \in S)\frac{ \Delta_a'^2}{\theta_S^1(\theta_1^1 + \epsilon)}\frac{\theta_i^1(\theta_1^1+\epsilon)}{{\Delta_a'}^2} \begin{cases} = 1 \text{ if } 1 \notin S\\ 
\le 1 \text{ otherwise } \end{cases}.
\end{align*}
 
The claim now follows recalling that $b(i) = 1 ~\forall i \in [M]$.  
\end{proof}

Thus we get $\ln \big( \frac{1}{\delta} \big)\sum_{i=2}^{n}\frac{\theta_i\theta_1}{{\Delta'_i}^2} =  \z^\top \x' \le  \z^\top \x^* = \b^\top\y^* = \sum_{S \in A}\E_{\btheta^1}[N_S(\tau_\cA)]$.
Moreover since $\epsilon>0$ is a construction dependent parameter, taking $\epsilon \to 0$ the expected sample complexity of $\cA$ under PL($n,\btheta^1$) becomes:
\[
\E_{\btheta^1}\Big[N_\cA(0,\delta)\Big] = \sum_{S \in A}\E_{\btheta^1}[N_S(\tau_\cA)] \ge \sum_{i=2}^{n}\frac{\theta_i\theta_1}{\Delta_i^2}\ln  \frac{1}{\delta} 
\]
Now taking $\epsilon \to 0$, the above construction shows that for any general problem instance, precisely PL($n,\btheta^1$), it requires a sample complexity of $\Omega\Bigg( \sum_{a=2}^n\frac{\theta_1\theta_a}{\Delta_a^2}\ln  \frac{1}{\delta}  \Bigg)$ on expectation, to find the \bi \,\,(i.e. to achieve $(0,\delta)$-PAC objective).
Finally to get the additional $\frac{n}{k}\log \frac{1}{\delta}$ term we appeal to the lower bound argument provided in  \cite{ChenSoda+18} (see their Thm. $B.9$) for the $\left(0, \frac{1}{8}\right)$-{PAC} best-arm identification problem. For such `low confidence' regimes, i.e., when $\delta = \Omega(1) \nrightarrow 0$, these explicitly shows a simple $\frac{n}{k} \log \frac{1}{\delta}$ term (independent of the instance) lower bound, which slightly improves the bound of Thm. \ref{thm:lb_fewf} for instances when $\theta_i \to 0$ (or $\Delta_i \to 1$) for all suboptimal item $i \in [n]\sm \{1\}$---note that a term like $\frac{n}{k} \log \frac{1}{\delta}$ is also intuitive, as for any \pl\, instance, the learner needs to query at the least $\Omega\Big(\frac{n}{k}\ln \frac{1}{\delta}\Big)$ many samples to make sure it covers the entire set of $n$ items.
\end{proof}
  
\subsection{Proof of Thm. \ref{thm:lb_fetf}}

\lbfetf*

\begin{proof}
The proof proceeds almost same as the proof of Thm. \ref{thm:lb_fewf}, the only difference lies in the analysis of the KL-divergence terms with \tf. 

Consider the exact same set of PL instances, PL$(n,\btheta^a)$ we constructed for Thm. \ref{thm:lb_fewf}. It is now interesting to note that how \tf\, affects the KL-divergence analysis, precisely the KL-divergence shoots up by a factor of $m$ which in fact triggers an $\frac{1}{m}$ reduction in regret learning rate.
Note that for \tf\, for any problem instance PL$(n,\btheta^a), \, a \in [n]$, each $k$-set $S \subseteq [n]$ is associated to ${k \choose m} (m!)$ number of possible outcomes, each representing one possible ranking of set of $m$ items of $S$, say $S_m$. Also the probability of any permutation $\bsigma \in \bSigma_{S}^{m}$ is given by
$
p^a_S(\bsigma) = Pr_{\btheta^a}(\bsigma|S),
$
where $Pr_{\btheta^a}(\bsigma|S)$ is as defined for \tf\, (in Sec. \ref{sec:prb_setup}). More formally,
for problem \textbf{Instance-a}, we have that: 

\begin{align*}
p^a_S(\bsigma) 
 = \prod_{i = 1}^{m}\frac{{\theta_{\sigma(i)}^a}}{\sum_{j = i}^{m}\theta_{\sigma(j)}^a + \sum_{j \in S \setminus \sigma(1:m)}\theta_{\sigma(j)}^a}, ~~~ \forall a \in [n],
\end{align*}

The important thing now to note is that for any such top-$m$ ranking of $\bsigma \in \bSigma_S^m$, $KL(p^1_S(\bsigma), p^a_S(\bsigma)) = 0$ for any set $S \not\owns a$. Hence while comparing the KL-divergence of instances $\btheta^1$ vs $\btheta^a$, we need to focus only on sets containing $a$. Applying \emph{Chain-Rule} of KL-divergence, we now get

\begin{align}
\label{eq:lb_witf_kl}
\nonumber KL(p^1_S, p^a_S) = KL(p^1_S(\sigma_1),& p^a_S(\sigma_1)) + KL(p^1_S(\sigma_2 \mid \sigma_1), p^a_S(\sigma_2 \mid \sigma_1)) + \cdots \\ 
& + KL(p^1_S(\sigma_m \mid \sigma(1:m-1)), p^a_S(\sigma_m \mid \sigma(1:m-1))),
\end{align}
where we abbreviate $\sigma(i)$ as $\sigma_i$ and $KL( P(Y \mid X),Q(Y \mid X)): = \sum_{x}Pr\Big( X = x\Big)\big[ KL( P(Y \mid X = x),Q(Y \mid X = x))\big]$ denotes the conditional KL-divergence. 
Moreover it is easy to note that for any $\sigma \in \Sigma_{S}^m$ such that $\sigma(i) = a$, we have $KL(p^1_S(\sigma_{i+1} \mid \sigma(1:i)), p^a_S(\sigma_{i+1} \mid \sigma(1:i))) := 0$, for all $i \in [m]$.

Now as derived in \eqref{eq:lb_wiwf_kl} in the proof of Thm. \ref{thm:lb_fewf}, we have 
\[
KL(p^1_S(\sigma_1), p^a_S(\sigma_1)) \le \frac{ \Delta_a'^2}{\theta_S^1(\theta_1^1 + \epsilon)}.
\]

To bound the remaining terms of \eqref{eq:lb_witf_kl},  note that for all $i \in [m-1]$
\begin{align*}
KL&(p^1_S(\sigma_{i+1} \mid \sigma(1:i)), p^a_S(\sigma_{i+1} \mid \sigma(1:i))) \\
& = \sum_{\sigma' \in \Sigma_S^i}Pr(\sigma')KL(p^1_S(\sigma_{i+1} \mid \sigma(1:i))=\sigma', p^a_S(\sigma_{i+1} \mid \sigma(1:i))=\sigma')\\
& = 	\sum_{\sigma' \in \Sigma_S^i \mid a \notin \sigma'}\Bigg[\prod_{j = 1}^{i}\Bigg(\dfrac{\theta^1_{\sigma'_j}}{\theta_S^1 - \sum_{j'=1}^{j-1}\theta_{\sigma'_{j'}}}\Bigg)\Bigg]\dfrac{ \Delta_a'^2}{(\theta_S^1-\sum_{l=1}^i \theta_{\sigma_l}^1)(\theta_1^1 + \epsilon)}  = 	 \dfrac{ \Delta_a'^2}{\theta_S^1(\theta_1^1 + \epsilon)}
\end{align*}

Thus applying above in \eqref{eq:lb_witf_kl} we get:

\begin{align}
\label{eq:lb_witf_kl2}
\nonumber KL(p^1_S, p^a_S) & = KL(p^1_S(\sigma_1) + \cdots + KL(p^1_S(\sigma_m \mid \sigma(1:m-1)), p^a_S(\sigma_m \mid \sigma(1:m-1)))\\
& \le \dfrac{ m\Delta_a'^2}{\theta_S^1(\theta_1^1 + \epsilon)}.
\end{align}

Eqn. \eqref{eq:lb_witf_kl2} gives the main result to derive Thm. \ref{thm:lb_fetf} as it shows an $m$-factor blow up in the KL-divergence terms owning to \tf. The rest of the proof follows exactly the same argument used in \ref{thm:lb_fewf}. We add the steps below for convenience.

Same as before, consider $\cE_0 \in \cF_\tau$ be an event such that the algorithm $\cA$ returns the element $i = 1$, and combining \eqref{eq:FI_a} and \eqref{eq:win_lb2}, for each problem instance PL$(n,\btheta^a)$, $a\in [n]\sm\{1\}$, we get,

\begin{align*}
\nonumber \sum_{S \in S^a}\E_{\btheta^1}[N_S(\tau_A)]KL(p^1_S,p^a_S) & \ge \ln \frac{1}{2.4\delta}\\
\end{align*}

Now using \eqref{eq:lb_witf_kl2}, we further get:

\begin{align}
\label{eq:win_lb3t}
\ln \frac{1}{2.4\delta} \le \sum_{S \in S^a} \E_{\btheta^1}[N_S(\tau_A)]KL(p^1_S,p^a_S) \le \sum_{S \in S^a}\E_{\btheta^1}[N_S(\tau_A)] \frac{ m\Delta_a'^2}{\theta_S^1(\theta_1^1 + \epsilon)}
\end{align}

Again consider the primal problem towards finding the sample complexity lower bound:

\begin{align*}
\textbf{Primal LP (P):} \hspace{20pt}\min_{S \in A} \sum_{S \in A}& \E_{\btheta^1}[N_S(\tau_\cA)]\\
\text{such that, } ~~\ln \frac{1}{2.4\delta} \le & \sum_{S \in S^a}\E_{\btheta^1}[N_S(\tau_A)] \frac{ m\Delta_a'^2}{\theta_S^1(\theta_1^1 + \epsilon)}, ~\forall a \in [n]\sm\{1\},
\end{align*}

which can equivalently be written as a linear programming (LP) of the following form:

\begin{align*}
\textbf{Dual LP (D):} \hspace{20pt}\min_{y} \b^\top\y\\
\text{such that, } \bK^\top \y & \ge \z, \text{ and } \y \geq 0,
\end{align*}

where $\y \in \R^{M}$, $M = |A| = \binom{n}{k}$, with $y(S) = \E_{\btheta^1}[N_S(\tau_A)], ~\forall S \in A$, $\z \in \R^{n-1}$ with $z(i) = \ln \frac{1}{2.4\delta} ~\forall i \in [n-1]$, $\bK \in \R^{M \times (n-1)}$ such that $K(S,a) = \begin{cases} \frac{ m\Delta_a'^2}{\theta_S^1(\theta_1^1 + \epsilon)}, \text{ if } S \in S^a\\
0, \text{ otherwise }
\end{cases}$, and $\b \in \R^{M \times 1}$ such that $b(i) = 1 ~\forall i \in [M]$.

The dual of the above LP boils down to:

\begin{align*}
\max_{\x} \z^\top\x\\
\text{such that, } \bK \x & \le \b, \text{ and } \x \geq 0,
\end{align*}
where clearly $\x \in \R^{n-1}$ is the dual optimization variable.

\textbf{Claim.} $x'_i = \frac{\theta_{i+1}^1(\theta_1^1+\epsilon)}{m{\Delta_a'}^2}$ for all $i \in [n-1]$ is a feasible solution of \textbf{(D)}.

\begin{proof}
Clearly, $x'_i \ge 0 ~\forall i \in [n-1]$ which ensures that the second set of constraints of \textbf{(D)} hold good. Expanding the first set of constraints $\bK \x' \le \b$ we get $M$ constraints, one for each $S \in A$ such that 
\begin{align*}
\sum_{i = 1}^{n-1}K(S,i)x'_i & = \sum_{i = 1}^{n-1}\1(S \in S^{i+1})K(S,i)\frac{\theta_{i+1}^1(\theta_1^1+\epsilon)}{m{\Delta_a'}^2}\\
& = \sum_{i = 2}^{n}\1(i \in S)\frac{ m\Delta_a'^2}{\theta_S^1(\theta_1^1 + \epsilon)}\frac{\theta_i^1(\theta_1^1+\epsilon)}{m{\Delta_a'}^2} \begin{cases} = 1 \text{ if } 1 \notin S\\ 
\le 1 \text{ otherwise } \end{cases}.
\end{align*}
 
The claim now follows recalling that $b(i) = 1 ~\forall i \in [M]$.  
\end{proof}

Thus we get $\ln \big( \frac{1}{\delta} \big)\sum_{i=2}^{n}\frac{\theta_i\theta_1}{m{\Delta'_i}^2} =  \z^\top \x' \le  \z^\top \x^* = \b^\top\y^* = \sum_{S \in A}\E_{\btheta^1}[N_S(\tau_\cA)]$.
Moreover since $\epsilon>0$ is a construction dependent parameter, taking $\epsilon \to 0$ the expected sample complexity of $\cA$ under PL($n,\btheta^1$) becomes:
\[
\E_{\btheta^1}\Big[N_\cA(0,\delta)\Big] = \sum_{S \in A}\E_{\btheta^1}[N_S(\tau_\cA)] \ge \sum_{i=2}^{n}\frac{\theta_i\theta_1}{m\Delta_i^2}\ln  \frac{1}{\delta} 
\]
Now taking $\epsilon \to 0$, the above construction shows that for any general problem instance, precisely PL($n,\btheta^1$), it requires a sample complexity of $\Omega\Bigg( \sum_{a=2}^n\frac{\theta_1\theta_a}{m\Delta_a^2}\ln  \frac{1}{\delta}  \Bigg)$ on expectation, to find the \bi \,\,(i.e. to achieve $(0,\delta)$-PAC objective) with \tf. Finally, to prove the additional instance independent $\Omega\Big( \frac{n}{k}\log \frac{1}{\delta}  \Big)$ term, we can use a similar argument provided in the Thm. \ref{thm:lb_fewf}, which ensures that no matter what the underlying \pl\, instance is, the learner needs to query at the least $\Omega\Big(\frac{n}{k}\ln \frac{1}{\delta}\Big)$ queries to cover the entire set of $n$ items--note that this term is independent of $m$.%
\end{proof}

\section{Appendix for Sec. \ref{sec:ft}}


\subsection{Proof of Thm. \ref{thm:lb_fttf}}

\lbfttf*


\begin{proof}

Similar to our lower bounds proofs for \fe\, setting (see Thm. \ref{thm:lb_fewf}, \ref{thm:lb_fetf}), we again use a \emph{change-of-measure argument} to prove the instance-dependent lower bounds for the \ft\, setting.

We start by constructing the problem instances as follows:
Consider a general the true underlying \pll\, problem instance $\text{PL}(n,\btheta^1): \theta_1^1 > \theta_2^1 \ge \ldots \ge \theta_n^1$, and corresponding to each suboptimal item $a \in [n]\setminus \{1\}$, let us define an alternative problem instance $\text{PL}(n,\btheta^a): \theta_a^a = \theta_1^1; \,\theta_1^a = \theta_a^1; ~\theta_i^a = \theta_i^1, ~~\forall i \in [n]\sm \{a,1\}$, for some $\epsilon>0$.

Then using a similar derivation shown for Eqn. \eqref{eq:lb_witf_kl2}, for above construction of problem instances in this case we can can show that:

\begin{align}
\label{eq:ft_kl}
KL(p^1_S, p^a_S) 
\le \dfrac{ m\Delta_a^2}{\theta_S^1}\bigg(\dfrac{\theta^1_1 \1(1 \in S) + \theta^1_a \1(a \in S)}{\theta^1_1\theta^1_a}\bigg)
\end{align}
where recall that we denote $\Delta_a = \theta_1^1 - \theta_a^1$, for any sub-optimal arm $a \in [n]\sm\{a\}$. Clearly for any subset $S \subset [n]$ such that $\{1,a\}\cap S = \emptyset$ must lead to $KL(p^1_S, p^a_S) = 0$ which is also follows from \eqref{eq:ft_kl}.

Same as the proof of Thm. \ref{thm:lb_fetf}, now applying Lem. \ref{lem:gar16} for any event $\cE \in \cF_\tau$ we get: 

\begin{align*}
\sum_{\{S \subseteq [n], |S|=k \mid a \in S\}}\E_{\btheta^1}[N_S(Q)]KL(p^1_S, p^a_S) \ge {kl(Pr_{\btheta^1}(\cE), Pr_{\btheta^a}(\cE))},
\end{align*}
where for any $k$-subset $S$, $N_S(Q)$ denotes the total number of times $S$ was played (i.e. queried upon for the \tf) by $\cA$ in $Q$ samples.
Now, consider $\cE_0 \in \cF_\tau$ be an event such that the algorithm $\cA$ indeed outputs the \bi\, $1$ upon termination, and let us analyse the left hand side of \eqref{eq:FI_a} for $\cE = \cE_0$. Now $\cA$ being  \bcon\, algorithm (see Defn. \ref{def:con}), we have $Pr_{\btheta^1}(\cE_0) > 1 - \exp(-f(\btheta)\sc)$. Moreover, since $\cA$ is \ordo\, as well, we also have $Pr_{\btheta^a}(\cE_0) < \exp(-f(\btheta)\sc)$, for any suboptimal arm $a \in [n]\setminus\{1\}$. Combining above two claims and denoting $\delta = \exp(-f(\btheta)\sc)$, we get: 

\begin{align*}
kl(Pr_{\btheta^1}(\cE_0),Pr_{\btheta^a}(\cE_0)) \ge kl(1-\delta,\delta) \ge \ln \frac{1}{2.4\delta}
\end{align*}

where the last inequality follows from \cite{Kaufmann+16_OnComplexity} (see Eqn. $(3)$). 
Then combining the above two claims with \eqref{eq:ft_kl}, for any problem instance PL$(n,\btheta^a)$, $a\in [n]\sm\{1\}$, we get,

\begin{align}
\label{eq:ft_con}
\ln \frac{1}{2.4\delta} \le \sum_{S \in \cS} \E_{\btheta^1}[N_S(Q)]KL(p^1_S,p^a_S) \le \sum_{S \in \cS}\E_{\btheta^1}[N_S(Q)] \dfrac{ m\Delta_a^2}{\theta_S^1}\bigg(\dfrac{\theta^1_1 \1(1 \in S) + \theta^1_a \1(a \in S)}{\theta^1_1\theta^1_a}\bigg),
\end{align}
where we denote the set of all possible k-subsets of $[n]$ by $\cS = \{S \subseteq [n] \mid |S|=k\}$.

Now coming back to our actual problem objective, recall that our goal is to understand the best possible lower bound on the quantity $\delta$---since the left hand side above is a decreasing function of $\delta$, at best any algorithm can aim to minimize $\delta$ as much as possible without violating the right hand side constraints for any $a \in [n]\sm\{1\}$. In other words any algorithm can at best aim to achieve a error confidence $\delta$ such that $\ln\Big(\frac{1}{2.4\delta}\Big)$ is upper bounded by:

\begin{align*}
& \textbf{Max-Min Optimization (P):}  \\
& \hspace{1 in} \max_{\{N_S(Q)\}_{S \in \cS}}\min_{a=2}^{n}\sum_{S \in \cS}\E_{\btheta^1}[N_S(Q)] \dfrac{ m\Delta_a^2}{\theta_S^1}\bigg(\dfrac{\theta^1_1 \1(1 \in S) + \theta^1_a \1(a \in S)}{\theta^1_1\theta^1_a}\bigg) \\
& \hspace*{1.5in} \text{such that, } ~\sum_{S \in \cS}\E_{\btheta^1}[N_S(Q)] = Q, \text{  and } \E_{\btheta^1}[N_S(Q)]\ge 0, \,\forall S \in \cS
\end{align*}

Clearly the optimization variables in \textbf{(P)} are $\{\E_{\btheta^1}[N_S(Q)]\}_{S \in \cS}$. We denote the simplex on $\cS$ by $\Pi_\cS = \{\bpi \in [0,1]^{n \choose k} \mid \sum_{i}\pi(i) = 1\}$. In general, we denote any $d$-dimensional simplex by $\Pi_d$, for any $d \in \N$. Then it is easy to follow that the above optimization problem \textbf{(P)} can be equivalently written in terms of optimization variables $x_S: = \frac{\E_{\btheta^1}[N_S(Q)]}{Q}$ as:

\begin{align*}
&\textbf{Equivalent Max-Min Optimization (P'):}\\
& \hspace{1in}  Q\Bigg[\max_{\{x_S\}_{S \in \cS}}\min_{\blam \in \Pi_{n-1}}\sum_{a = 2}^{n}\lambda(a)\Bigg(\sum_{S \in \cS}x_S \dfrac{ m\Delta_a^2}{\theta_S^1}\bigg(\dfrac{\theta^1_1 \1(1 \in S) + \theta^1_a \1(a \in S)}{\theta^1_1\theta^1_a}\bigg)\Bigg)\Bigg] \\
& \hspace{1.5in} \text{such that, } ~\sum_{S \in \cS}x_S = 1, \text{  and } x_s\ge 0, \,\forall S \in \cS
\end{align*}

We denote by opt\textbf{(P)} and opt\textbf{(P')} the optimal values of problem \textbf{P} and \textbf{P'} respectively. Note that opt\textbf{(P)} = opt\textbf{(P')}. Also note that $(x_S)_{S \in \cS} \in \Pi_\cS$. Then, opt(\textbf{P'}) can be further rewritten as:

\begin{align*}
\frac{\text{opt{\bf(P')}}}{Q} & = \max_{\{x_S\}_{S \in \cS}}\min_{\blam \in \Pi_{n-1}}\sum_{a = 2}^{n}\lambda(a)\Bigg(\sum_{S \in \cS}x_S\dfrac{ m\Delta_a^2}{\theta_S^1} \bigg(\dfrac{\theta^1_1 \1(1 \in S) + \theta^1_a \1(a \in S)}{\theta^1_1\theta^1_a}\bigg)\Bigg)\\
& = \min_{\blam \in \Pi_{n-1}}\max_{\{x_S\}_{S \in \cS}}\sum_{S \in \cS}\sum_{a = 2}^{n}\lambda(a)x_S\dfrac{ m\Delta_a^2}{\theta_S^1} \bigg(\dfrac{\theta^1_1 \1(1 \in S) + \theta^1_a \1(a \in S)}{\theta^1_1\theta^1_a}\bigg)\\
& = \min_{\blam \in \Pi_{n-1}}\max_{S \in \cS}\sum_{a = 2}^{n}\lambda(a) \dfrac{ m\Delta_a^2}{\theta_S^1} \bigg(\dfrac{\theta^1_1 \1(1 \in S) + \theta^1_a \1(a \in S)}{\theta^1_1\theta^1_a}\bigg)
\end{align*}
where the second equality follows from Von Neumann’s well-known Minmax Theorem \cite{Freund96}. Now further setting $\lambda'(a) = \dfrac{{(\theta_a^1)^2}/{\Delta_a^2}}{\sum_{i = 2}^n{(\theta_i^1)^2}/{\Delta_i^2}}$, for all $a \in [n]\sm\{1\}$ (note that $\blam' \in \Pi_{n-1}$), using $\blam = \blam'$ in opt\textbf{(P')}, it can further be upper bounded as:

\begin{align*}
& \frac{\text{opt{\bf(P')}}}{Q} \le \max_{S \in \cS}\sum_{a = 2}^{n}\lambda'(a) \dfrac{ m\Delta_a^2}{\theta_S^1} \bigg(\dfrac{\theta^1_1 \1(1 \in S) + \theta^1_a \1(a \in S)}{\theta^1_1\theta^1_a}\bigg) \\
& = \max_{S \in \cS}\sum_{a = 2}^{n}\dfrac{(\theta_a^1)^2}{\sum_{i = 2}^n{(\theta_i^1)^2}/{\Delta_i^2}} \dfrac{ m}{\theta_S^1} \bigg(\dfrac{\theta^1_1 \1(1 \in S) + \theta^1_a \1(a \in S)}{\theta^1_1\theta^1_a}\bigg)\\
& = 
\begin{cases} = \underset{S \in \cS}{\max}\sum_{a = 2}^{n}\dfrac{(\theta_a^1)^2}{\sum_{i = 2}^n{(\theta_i^1)^2}/{\Delta_i^2}} \dfrac{ m}{\theta_S^1} \bigg(\dfrac{1}{\theta^1_1}\bigg) \le m\bigg(\sum_{a = 2}^{n}\frac{(\theta_a^1)^2}{\Delta_a^2}\bigg)^{-1},
 \text{ if } 1 \notin S, \, a\in S\\ 
\le \underset{S \in \cS}{\max}\sum_{a = 2}^{n}\dfrac{(\theta_a^1)^2}{\sum_{i = 2}^n{(\theta_i^1)^2}/{\Delta_i^2}} \dfrac{ m}{\theta_S^1} \bigg(\dfrac{1}{\theta^1_a}\bigg) = m\bigg(\sum_{a = 2}^{n}\frac{(\theta_a^1)^2}{\Delta_a^2}\bigg)^{-1}, \text{ if } a \notin S,\, 1 \in S\\
\le \underset{S \in \cS}{\max}\sum_{a = 2}^{n}\dfrac{(\theta_a^1)^2}{\sum_{i = 2}^n{(\theta_i^1)^2}/{\Delta_i^2}} \dfrac{ m}{\theta_S^1} \bigg(\dfrac{\theta^1_a + \theta^1_1}{\theta^1_1\theta^1_a}\bigg) \le 2m\bigg(\sum_{a = 2}^{n}\frac{(\theta_a^1)^2}{\Delta_a^2}\bigg)^{-1}, \text{ if both } 1,a \in S\\
=  0 \text{ otherwise } 
\end{cases} \\
& \le 2m\Big(\sum_{a = 2}^{n}\frac{(\theta_a^1)^2}{\Delta_a^2}\Big)^{-1}
\end{align*}

Then combining above upper bound to Eqn. \ref{eq:ft_con}, we finally get:

\begin{align*}
\dfrac{\ln \frac{1}{2.4\delta}}{Q} \le 2m\Big(\sum_{a = 2}^{n}\frac{(\theta_a^1)^2}{\Delta_a^2}\Big)^{-1} & \implies 
\frac{1}{2.4\delta} \le \exp\Bigg( 2mQ\Big(\sum_{a = 2}^{n}\frac{(\theta_a^1)^2}{\Delta_a^2}\Big)^{-1} \Bigg)\\
& \implies \delta \ge \dfrac{\exp\Big(-2mQ\Big(\sum_{a = 2}^{n}\frac{(\theta_a^1)^2}{\Delta_a^2}\Big)^{-1} \Big)}{2.4},
\end{align*}
which proves the claim.
Thus we show for any general problem instance, precisely PL($n,\btheta^1$), such that any $(0,\delta)$-PAC algorithm incurs an error on at least $\Omega\Bigg({\exp\Big( -2mQ\Big(\sum_{a = 2}^{n}\frac{(\theta_a^1)^2}{\Delta_a^2}\Big)^{-1} \Big)}\Bigg)$ towards identifying the \bi\, with \tf.
\end{proof}

\subsection{Pseudo-code for \algfttf} 
\label{app:alg_fttf}

\vspace{-10pt}
\vspace{-0pt}
\begin{center}
\begin{algorithm}[H]
   \caption{\textbf{\algfttf}}
   \label{alg:fttf}
\begin{algorithmic}[1]
   \STATE {\bfseries input:} Set of items: $[n]$, Subset size: $k \le n$, Ranking \\~~~~~feedback size: $m \in [k-1]$, Sample complexity \sc
   \STATE {\bfseries init:} $\cA \leftarrow [n]$, $s \leftarrow 1$
   \WHILE {$|\cA| \geq k$} 
    \STATE $\cB_1, \cB_2, \ldots \cB_B \leftarrow$ \algdiv$(\cA,k)$
    \STATE \textbf{if} $|\cB_B| < k$, \textbf{then} $B \leftarrow B - 1$, $\cR \leftarrow \cB_B$
	\FOR {$b \in [B]$}   
   	\STATE Play the set $\cB_b$ for $Q':= \frac{kQ}{2n+k\log_2 k}$ times 
   	\STATE For all $i,j \in \cB_b$, update $\hp_{ij}$ with \rb
   	\STATE Compute $w_i: = \sum_{j \in \cB_b}\1(\hp_{ij}>\frac{1}{2})$
   	\STATE Define $\bar w \leftarrow $ Median$(\{w_i\}_{i \in \cB_b}), \, \forall i \in \cB_b$
   	\STATE $\cA \leftarrow \{i \in \cB_b \mid w_i \ge \bar w\}$ 
   	\ENDFOR 
    \STATE $\cA \leftarrow \cA \cup \cR$; 
 
	\ENDWHILE
	\STATE $\cB \leftarrow \cA \cup$ $\{ k-|\cA|$ random elements from $[n]\sm\cA\}$
	\WHILE{$|\cA| > 1$} 
	\STATE Play the set $\cB$ for $Q':= \frac{kQ}{2n+k\log_2 k}$ times 
   	\STATE For all $i,j \in \cA$, update $\hp_{ij}$ with \rb
   	\STATE Compute $z_i: = \sum_{j \in \cA}\1(\hp_{ij}>\frac{1}{2}),\, \forall i \in \cA$
   	\STATE Define $\bar z \leftarrow $ Median$(\{z_i\}_{i \in \cA})$
   	\STATE $\cA \leftarrow \{i \in \cA \mid z_i \ge \bar z\}$
   	\ENDWHILE
   \STATE {\bfseries output:} The remaining item in $\cA$
\end{algorithmic}
\end{algorithm}
\end{center}
\vspace{-15pt}

\subsection{Proof of Thm. \ref{thm:pr_fttf}}

\pralgfttf*

\begin{proof}
Firstly, we establish that the sample complexity of \algfttf\, is always within the stipulated constraint $Q$.

\textbf{Correctness of stipulated sample complexity (Q).} To show this note that inside any round, for any batch $\cB_b, \, b \in [B]$ $\frac{k}{2}$ items of $\cB_b$, by definition of $\bar w$. Thus at each consecutive round, the number of surviving elements gets halved, which implies that the total number of rounds can be at most $\log_2 n$. Hence size of the set of surviving items $|\cA|$ at round $i$ is approximately $\frac{n}{2^{\ell-1}}$, for any round $\ell = 1,2,\ldots \log_2 n$. Also number of sets formed at round $i$ is $\flr{\frac{|\cA|}{k}} < \frac{|\cA|}{k}$. Then total number of sets formed by the algorithm during its entire run can be at most:
\[
\frac{n}{k}\big( 1 + \frac{1}{2} + \frac{1}{2^2} + \cdots + \frac{1}{2^{\ceil{\log_2 \frac{n}{k}}}} \big) + \log_2 k  < \frac{n}{k}\big( \sum_{i = 0}^\infty \frac{1}{2^i} \big) + \log_2 k = \frac{(2n + k\log_2 k)}{k}
\]
where the extra $\log_2k$ term is due to the final $\log_2 k$ rounds for which $|\cA| < k$. Now since our strategy is to allocate uniform budget across all sets, the assumign sample complexity per set is $Q' = \frac{Q}{\dfrac{(2n + k\log_2 k)}{k}} = \frac{Qk}{2n + k\log_2 k}$. Hence our algorithm is always within the budget constraint $Q$.
The only part left is to now prove the confidence bound of Thm. \ref{thm:pr_fttf}, as analysed below:

\textbf{Bounding the \bi\, identification confidence.}
We first analyse the any particular batch $\cB \in \{\cB_b\}_{b \in [B]}$, for any particular round $\ell = 1,2,\ldots \log_2 n$, such that $1 \in \cB$. Let us analyze the probability of item $1$ getting eliminated from batch $\cB$ at the end of round $\ell$.

First recall the number of times $\cB$ is sampled is $Q' = \frac{kQ}{2n+ k \log_2k}$. 
Now let us define $w_i$ as the number of times item $i \in \cB$ was returned in the top-$m$ winner (i.e. $i$ appeared in the \tf\, $\sigma \in \Sigma_S^m$) in $Q'$ plays, and $\sigma_\tau$ be the top-$m$ ranking retuned by the environment upon playing the batch $\cB$ for the $\tau^{th}$ round, $\forall \tau \in [Q']$.
Then given $\theta_{1} = \arg\max_{i \in [n]}\theta_i$, clearly $Pr(\{1 \in \sigma_\tau\}) = \sum_{j = 1}^{m}Pr\big( \sigma_\tau(j) = 1 \Big) = \sum_{j = 0}^{m-1}\frac{1}{2(k-j)} \ge \frac{m}{k}$, since $Pr(\{1 | S\}) = \frac{\theta_{1}}{\sum_{j \in S}\theta_j} \ge \frac{1}{|S|}$ for any $S \subseteq \cB$.
Thus $\E[w_{1}] = \sum_{\tau = 1}^{Q'}\E[\1(1 \in \sigma_\tau)] \ge \frac{mQ'}{k}$. 
Now applying multiplicative Chernoff-Hoeffdings bound on the random variable $w_{1}$, we get that for any $\eta \in (0,1]$, 

\begin{align*}
Pr\Big( w_{1} \le (1-\eta)\E[w_{1}] \Big) & \le \exp\bigg(- \frac{\E[w_{1}]\eta^2}{2}\bigg) \le \exp\bigg(- \frac{mQ'\eta^2}{2k}\bigg) 
\end{align*}
In particular, setting $\eta = \frac{1}{2}$ we get with probability at least $\Bigg(1-\exp\bigg(- \frac{mQ'}{8k}\bigg)\Bigg)$,  $w_{1} > (1-\frac{1}{2})\E[w_{1}] > \frac{mQ'}{2k}$, for any such batch $\cB$, at any round $\ell$. This further implies that with probability at least $1-\exp\bigg(- \frac{mQ'}{8k}\bigg)$, after $Q'$ plays, we have $w_{1i} + w_{i1} \ge \frac{mQ'}{2k}$, for any item $i \in \cB\sm\{1\}$, as due to \rb\, update whenever an item appears in \tf\, $\sigma_\tau$, it ends up getting pairwise compared with the rest of the $k-1$ items in $\cB$ after $\tau^{th}$ play. Let us denote  $n_{1i} = w_{1i} + w_{i1}$. Then the probability that any suboptimal item $i \in \cB\sm\{1\}$ beats $1$ after $Q'$ plays is:

\begin{align*}
Pr\Bigg( & \hp_{1i} > \frac{1}{2},  n_{1 i} \ge \frac{mQ'}{2k}  \Bigg) = Pr\Bigg( \hp_{1i} - p_{1i} > \frac{1}{2} - \p_{1i},  n_{1i} \ge \frac{mQ'}{2k} \Bigg)\\
& = Pr\Bigg( \hp_{1i} - p_{1i} > \p_{1i} -\frac{1}{2},  n_{1i} \ge \frac{mQ'}{2k} \Bigg)\\
& \le Pr\Bigg( \hp_{1i} - p_{1i} > \frac{\Delta_i}{4},  n_{1i} \ge \frac{mQ'}{2k} \Bigg) ~~\bigg[\text{ as, } \p_{1i} -\frac{1}{2} = \frac{(\theta_1- \theta_i)}{2(\theta_1+ \theta_i)} > \frac{(\theta_1- \theta_i)}{4}\bigg]\\
& \le \exp\Big( -2\frac{mQ'}{2k}\bigg(\frac{\Delta_i}{4}\bigg)^2 \Big) = \exp\Big( -\frac{mQ'\Delta_i^2}{16k} \Big),
\end{align*}

where the last inequality follows from Lem. \ref{lem:pl_simulator} for $\eta = \frac{\Delta_i}{4}$, and $v = \frac{mQ'}{2k}$. 
So combining the above two claims, we get that the total probability of

\begin{align}
\label{eq:ft_1}
\nonumber Pr\Bigg(  \hp_{1i} > \frac{1}{2} \Bigg) & =
Pr\Bigg(  \hp_{1i} > \frac{1}{2},  n_{1 i} \ge \frac{mQ'}{2k}  \Bigg) + Pr\Bigg(  \hp_{1i} > \frac{1}{2},  n_{1 i} < \frac{mQ'}{2k}  \Bigg)\\
\nonumber  & \le \exp\Big( -\frac{mQ'\Delta_i^2}{16k} \Big) + Pr\Bigg(  w_{1 i} < \frac{mQ'}{2k}  \Bigg) \le  \exp\Big( -\frac{mQ'\Delta_i^2}{16k} \Big) + \exp\bigg(- \frac{mQ'}{8k}\bigg)\\
& \le 2\exp\Big( -\frac{mQ'\Delta_i^2}{16k} \Big).
\end{align}

Now let us try to analyze that for a fixed round $\ell$, how many such suboptimal item $i \in \cB\sm\{1\}$ can beat the \bi\, $1$. Towards this we define a random variable $V: = \sum_{i \in \cB\sm\{1\}}\1(\hp_{i1} > \frac{1}{2})$. Now from \eqref{eq:ft_1} we get that: 
\[
\E[V] = \sum_{i \in \cB\sm\{1\}}Pr(\hp_{i1} > \frac{1}{2}) \le 2(k-1)\exp\Big( -\frac{mQ'\Delta_i^2}{16k} \Big)
\]

Then applying Markov's inequality we have:
\[
Pr\Big(V \ge \frac{k}{2}\Big) \le \frac{\E[V]}{\frac{k}{2}} \le \frac{4(k-1)}{k}\exp\Big( -\frac{mQ'\Delta_i^2}{16k} \Big) \le \frac{4(k-1)}{k}\exp\Big( -\frac{mQ'\Delta_{\min}^2}{16k} \Big)
\]
It is important to note that in case if $V < \frac{k}{2}$, $\implies z_1 > \frac{k}{2}$ and hence $z_1 > \bar z$, as $\bar z \le \frac{k}{2}$. 

Therefore with probability at least $\Big( 1 - \frac{4(k-1)}{k}\exp\Big( -\frac{mQ'\Delta_{\min}^2}{16k} \Big) \Big)$, item $1$ is not eliminated in round $\ell$. Then the total probability of item $1$ getting eliminated in the entire run of \algfttf\, can be upper bounded as:
\begin{align*}
Pr\Big(\exists \ell = 1,2,\ldots &
log_2 n \text{ s.t. item 1 is eliminated at round } \ell \Big) \\
& \le \sum_{\ell = 1}^{\log_2 n}Pr\Big(\text{ Item 1 is eliminated at round } \ell \Big) \\
& \le 4\log_2 n \frac{(k-1)}{k}\exp\Big( -\frac{mQ'\Delta_{\min}^2}{16k} \Big)\\
& = 4\log_2 n \frac{(k-1)}{k}\exp\Big( -\frac{mQ\Delta_{\min}^2}{16(2n + k \log_2k)} \Big),
\end{align*}
where the last equality follows recalling that we set $Q' = \frac{kQ}{2n+k\log_2 k}$, which concludes the first claim. The second claim simply follows from the first as the total error probability is upper bounded by $\delta$, this further implies
\begin{align*}
& \delta \le 4\log_2 n \frac{(k-1)}{k}\exp\Big( -\frac{mQ\Delta_{\min}^2}{16(2n + k \log_2k)} \Big)\\
& \implies Q \ge \frac{16(2n + k \log_2k)}{m\Delta_{\min}^2}\ln\bigg(\frac{4(k-1)\log_2 n}{k\delta} \bigg)\Bigg) = O\Bigg(\frac{16(2n + k \log_2k)}{m\Delta_{\min}^2}\ln\bigg(\frac{\log_2 n}{\delta} \bigg)\Bigg)\Bigg)
\end{align*}
which proves the second claim.
\end{proof}

\section{Appendix for Sec. \ref{sec:experiments}}
\label{app:expts}

\textbf{Environments.}
1. {\it g1}, 2. {\it g4}, 3. {\it arith}, 4. {\it geo}, 5. {\it b1} all with $n = 16$, and three larger models 5. {\it g4-big}, 6. {\it arith-big}, and 7. {\it geo-big} each with $n=50$ items. Their individual score parameters are as follows: 
\textbf{1. g1:} $\theta_1 = 0.8$, $\theta_i = 0.2, \, \forall i \in [16]\sm\{1\}$
\textbf{2. g4:} $\theta_1 = 1$, $\theta_i = 0.7, \, \forall i \in \{2,\ldots 6\}$, $\theta_i = 0.5, \, \forall i \in \{7,\ldots 11\}$, and $\theta_i = 0.01$ otherwise.
\textbf{3. arith:} $\theta_1 = 1$ and $ \theta_{i} - \theta_{i+1} = 0.06, \, \forall i \in [15]$.
\textbf{4. geo:} $\theta_1=1$, and $\frac{\theta_{i+1}}{\theta_{i}} = 0.8, ~\forall i \in [15]$. 
\textbf{5. b1:} $\theta_1 = 0.8$, $\theta_i = 0.6, \, \forall i \in [16]\sm\{1\}$
\textbf{6. g4b:} $\theta_1 = 1$, $\theta_i = 0.7, \, \forall i \in \{2,\ldots 18\}$, $\theta_i = 0.5, \, \forall i \in \{19,\ldots 45\}$, and $\theta_i = 0.01$ otherwise.
\textbf{7. arithb:} $\theta_1 = 1$ and $ \theta_{i} - \theta_{i+1} = 0.2, \, \forall i \in [49]$.
\textbf{8. geob:} $\theta_1=1$, and $\frac{\theta_{i+1}}{\theta_{i}} = 0.9, ~\forall i \in [49]$.
